\documentclass[10pt,journal,compsoc]{IEEEtran}

\ifCLASSOPTIONcompsoc
  \usepackage[nocompress]{cite}
\else
  \usepackage{cite}
\fi

\ifCLASSINFOpdf
\else
\fi

\usepackage{graphicx}%
\usepackage{multirow}%
\usepackage{amsmath,amssymb,amsfonts}%
\usepackage{amsthm}%
\usepackage{mathrsfs}%
\usepackage{xcolor}%
\usepackage{textcomp}%
\usepackage{manyfoot}%
\usepackage{booktabs}%
\usepackage{listings}%

\DeclareMathOperator*{\argmin}{arg\,min}
\usepackage{dsfont}

\usepackage{xurl}

\usepackage{tikz}
\usetikzlibrary{quantikz}
\usepackage{hyperref}
\usepackage{cleveref}
\hypersetup{
  colorlinks=true,
  linkcolor=red,
  citecolor=teal,
  urlcolor=red,
  pdfborder={0 0 0}
}
\usepackage{bbold}
\usepackage{soul}
\usepackage{subfig}
\usepackage{enumitem}

\usepackage{braket}

\usepackage{tikz}
\usepackage[most]{tcolorbox}

\def\checkmark{\tikz\fill[scale=0.4](0,.35) -- (.25,0) -- (1,.7) -- (.25,.15) -- cycle;}

\newcommand{\R}{\mathbb{R}}

\newcommand{\etal}{\textit{et al}. }
\newcommand{\ie}{\textit{i}.\textit{e}.}
\newcommand{\eg}{\textit{e}.\textit{g}.}

\newcommand{\coo}{\ensuremath{\mathrm{CO_2}}}

\usepackage{tcolorbox}
\tcbuselibrary{theorems}
\tcbuselibrary{theorems}
\newtcbtheorem
  []
  {assumption}
  {Assumption}
  {%
    colback=green!2,
    colframe=green!25!gray,
    fonttitle=\bfseries,
  }
  {assum}

  \newtcbtheorem
  []
  {theorem}
  {Theorem}
  {%
    colback=red!0,
    colframe=red!0!gray,
    fonttitle=\bfseries,
  }
  {thm}
  
 \newtcbtheorem
  []
  {lemma}
  {Lemma}
  {%
    colback=blue!2,
    colframe=blue!10!gray,
    fonttitle=\bfseries,
  }
  {lem}
\crefname{theorem}{Theorem}{Theorems}
\crefname{assumption}{Assumption}{Assumptions}
\crefname{lemma}{Lemma}{Lemmas}

\newcommand{\C}{\mathbb{C}}
\usepackage{wrapfig}

\usepackage{titlecaps}
\Addlcwords{and}
\Addlcwords{to}
\Addlcwords{at}
\Addlcwords{of}
\Addlcwords{vs}
\Addlcwords{on}
\Addlcwords{for}
\Addlcwords{with}
\Addlcwords{through}
\Addlcwords{as}
\Addlcwords{due}
\Addlcwords{a}
\Addlcwords{the}
\Addlcwords{in}
\Addlcwords{from}
\Addlcwords{}
\Addlcwords{}

\newcommand{\parahead}[1]{\noindent\textbf{\titlecap{#1}}.\ }

\begin{document}
\title{Quantum-enhanced Computer Vision:\\Going Beyond Classical Algorithms} 
\author{
Natacha Kuete Meli$^1$ $\quad$ Shuteng Wang$^4$ $\quad$ Marcel Seelbach Benkner$^1$ $\quad$ Michele Sasdelli$^2$ \\ 
$\quad\quad\;$Tat-Jun Chin$^2$ $\quad\quad\;\;\;$ Tolga Birdal$^3$ $\quad\quad\quad\quad$ Michael Moeller$^1$ $\quad\quad\quad$ Vladislav Golyanik$^4$ \\
\vspace{5pt}
$^1$University of Siegen $\;\;$ $^2$University of Adelaide $\;\;$ $^3$Imperial College London $\;\;$ $^4$MPI for Informatics 
}

\IEEEtitleabstractindextext{%
\begin{abstract} 
Quantum-enhanced Computer Vision (QeCV) is a new research field at the intersection of computer vision, optimisation theory, machine learning and quantum computing. It has high potential to transform how visual signals are processed and interpreted with the help of quantum computing that leverages quantum-mechanical effects in computations inaccessible to classical (i.e. non-quantum) computers. In scenarios where existing non-quantum methods cannot find a solution in a reasonable time or compute only approximate solutions, quantum computers can provide, among others, advantages in terms of better time scalability for multiple problem classes. Parametrised quantum circuits can also become, in the long term, a considerable alternative to classical neural networks in computer vision. However, specialised and fundamentally new algorithms must be developed to enable compatibility with quantum hardware and unveil the potential of quantum computational paradigms in computer vision. 

This survey contributes to the existing literature on QeCV with a holistic review of this research field. It is designed as a quantum computing reference for the computer vision community, targeting computer vision students, scientists and readers with related backgrounds who want to familiarise themselves with QeCV. We provide a comprehensive introduction to QeCV, its specifics, and methodologies for formulations compatible with quantum hardware and QeCV methods, leveraging two main quantum computational paradigms, i.e. gate-based quantum computing and quantum annealing. We elaborate on the operational principles of quantum computers and the available tools to access, program and simulate them in the context of QeCV. Finally, we review existing quantum computing tools and learning materials and discuss aspects related to publishing and reviewing QeCV papers, open challenges and potential social implications.
\end{abstract}

}

\maketitle

\section{Introduction}
\label{sec:introduction}
\IEEEPARstart{C}omputer Vision (CV) studies automatic processing of visual and spatial information. Data representing such information is acquired in the form of 2D images, videos, depth maps, 3D point clouds, and different combinations of these inputs, possibly along with other sensory signals (\eg~data from inertial measurement units) \cite{szeliski2022computer,stockman2001computer}. 
The long-term aim of computer vision is to come up with intelligent, high-level interpretations of the observed scenes, inspired by the capabilities of the Human Visual System (HVS). 
Today, CV and Machine Learning (ML) constantly inform each other: modern CV strongly relies on ML techniques while developments in ML are often driven by unsolved problems in CV \cite{torralba2024foundations}. 

Many modern CV systems are inherently complex: They include multiple components and process large data volumes (\eg~during neural network training) \cite{PoYifanGolyanik2024, Awais2025}. 
Primary tools of computer vision, such as deep learning, have started to demand unreasonable and hard-to-satisfy GPU resources when classical computation is concerned. 
Neural architectures in many areas of CV are becoming increasingly larger, parameter-heavy and require more and more time to converge during training. 
Moreover, many CV problems contain combinatorial objectives that cannot be solved by an exhaustive search in a reasonable time. 
As a consequence, they require hardware that is able to fulfil high data processing demands (\eg~graphical processing units (GPUs)). 
In addition, theoretical guarantees of classical solvers are only locally optimal \cite{Sun2020}, and as such, may not be sufficient for certain problems.

Hence, in many cases, progress on the algorithmic side is a valuable alternative to increasing demands in storage and computational resources.
Such algorithmic improvements can be achieved in different ways, \eg~by approximative formulations and solutions instead of exact but computationally expensive or intractable ones. 
This survey investigates a specific type of those, namely related to how \textbf{quantum computing} can be leveraged in computer vision. 

\parahead{Quantum computing at a glance} 
Since the proposal of quantum computers in the 1980s \cite{feynman1982simulating,benioff1980computer,deutsch1985quantum}, substantial progress has been made in their practical experimental realisations. 
The new computational paradigm inspired a multitude of works on theoretical foundations of quantum computing (computer science) \cite{nielsen2002quantum,das2005quantum} and quantum hardware realisations (physics) \cite{easttom2024hardware}. 
Researchers have been actively working on various qubit technologies for the last 25 years, and quantum computers based on them are available now for research purposes. 
Such machines and hardware that are and will become available in the near future---with up to several hundred qubits---are often called \textit{Noisy Intermediate-scale Quantum (NISQ) computers} \cite{preskill2018quantum}. 
In 2019, a quantum computer, Google's Sycamore, was able to perform a  particular calculation tens of orders of magnitude quicker than a classical computer~\cite{arute2019quantum}. 

Quantum computation fundamentally revolves around evolving quantum systems into distinct states using quantum mechanical effects. 
To harness these effects for practical computation, two primary paradigms, governed by two physical principles, have emerged: Adiabatic Quantum Computing (AQC) \cite{das2005quantum,albash2018adiabatic} and gate-based quantum computing \cite{nielsen2002quantum,sutor2019dancing}. 
AQC relies on the smooth evolution or annealing of a so-called Hamiltonian to guide the system toward its lowest-energy state, making it naturally suited for optimisation problems, particularly those expressed in the Quadratic Unconstrained Binary Optimisation (QUBO) form. 
In contrast, gate-based quantum computing employs discrete unitary transformations, offering greater flexibility in algorithm design. 
Together, these paradigms define modern quantum computation, each with distinct advantages that depend on the specific problem domain. 

Broadly speaking, Quantum Computing (QC) allows designing algorithms to solve several classes of computationally challenging problems with possible computational gains depending on the type of QC and the problem. 
The type of computational gains can range from improvements in asymptotic complexity (compared to the best classical counterparts) \cite{watrous2008quantum} to the acceleration of computations in absolute terms (in the same complexity class) \cite{mohseni2022ising}. 
Quantum computers operate on \textit{qubits}, counterparts of classical bits that leverage quantum effects.
These qubits abstractly span a Hilbert space, where computation takes place.
In theory, a quantum computer can perform everything a classical computer can perform and vice versa.
However, the corresponding Hilbert space of multiple qubits is exponentially large ($2^n$-dimensional for $n$ qubits), due to the tensor product structure of quantum mechanics, and so-called entangled states where the qubits cannot be described separately. 
Classical computation of the corresponding exponentially large operators would, of course, be highly time-consuming. 

As of 2023-2025, we observe a transition in the field thanks to the proliferation of real quantum hardware: If previously (pre-2015), quantum hardware was accessible predominantly to researchers developing it, multiple quantum and quantum-inspired computers 
(\eg~adiabatic, gate-based machines, photonic machines, quantum simulators and quantum-inspired optimisers) 
can nowadays be accessed by researchers from different fields, and the developed methods can be tested on real quantum hardware.
This is in contrast to the vast majority of papers on quantum algorithms published before, 
including highly influential ones that have been tested on very small problems on real quantum hardware so far \cite{Shor1997, Lanyon2007, Skosana2021}. 
For the next two decades, experts predict a super-linear increase in the number of qubits \cite{Sevilla2020arXiv} and substantial improvements in the properties of the individual qubits (such as their decoherence and connectivity properties). 
We argue that these predictions should be taken seriously, because the investments in quantum computing and algorithm development are substantial; many national and international initiatives related to quantum computing were brought into being within the last several years. 
Quantum computing power on selected algorithms scales super-linearly (up to exponentially) with the number of qubits. 
Almost all the big technology companies including Intel, IBM, Google,  Microsoft, Amazon, NVIDIA and D-Wave are aiming at a steady increase in the number of qubits and are investing heavily in quantum technology, as the limits of Moore's law is approached\footnote{Researchers even started to use the term \textit{Neven's Law}~\cite{Hartnett2019}, \ie~ referring to the doubly-exponential growth in quantum compute power.}. 
For example, Google has publicly announced its goal to build a commercial quantum computer composed of 1M qubits by 2029\footnote{\url{https://quantumai.google/learn/map}}. 
The Willow chip achieving successful quantum error correction \cite{google_willow} constitutes an encouraging milestone for quantum computing.
In response to the aforementioned algorithmic challenges in computer vision and the opportunities of quantum computing, 
computer vision researchers started turning their attention to this new (for the research field) computational paradigm. 
From both theoretical and practical perspectives, it is both desirable and interesting to investigate new (quantum) algorithms for long-standing (vision) problems. 
Based on different principles than before, such algorithms can bring previously unexpected advantages and new properties to computer vision systems. 
To make the best use of the strongly increasing quantum computing power, we need to make the best use of quantum machines, and we need QeCV algorithms to be ready and scalable when practical quantum computing arrives. 
In the long term, Quantum Processing Units (QPUs) promise to extend the available arsenal of reliable computer vision tools and computational accelerators (with GPUs being an example of currently widely-used technology). 
Fig.~\ref{fig:qcv} provides an overview of different computer vision problems and quantum computational paradigms that can be used to address them, adiabatic and gate-based quantum computing; it also highlights the common steps of every quantum-compatible computer vision approach (\eg~problem embedding to the hardware, annealing or execution of quantum gates, and, finally, read-out of the solution). 
We discuss both quantum computational paradigms in Sec.~\ref{sec:operational_principles}. 

\begin{figure*}[t]
	\begin{center}
		\includegraphics[width=\linewidth]{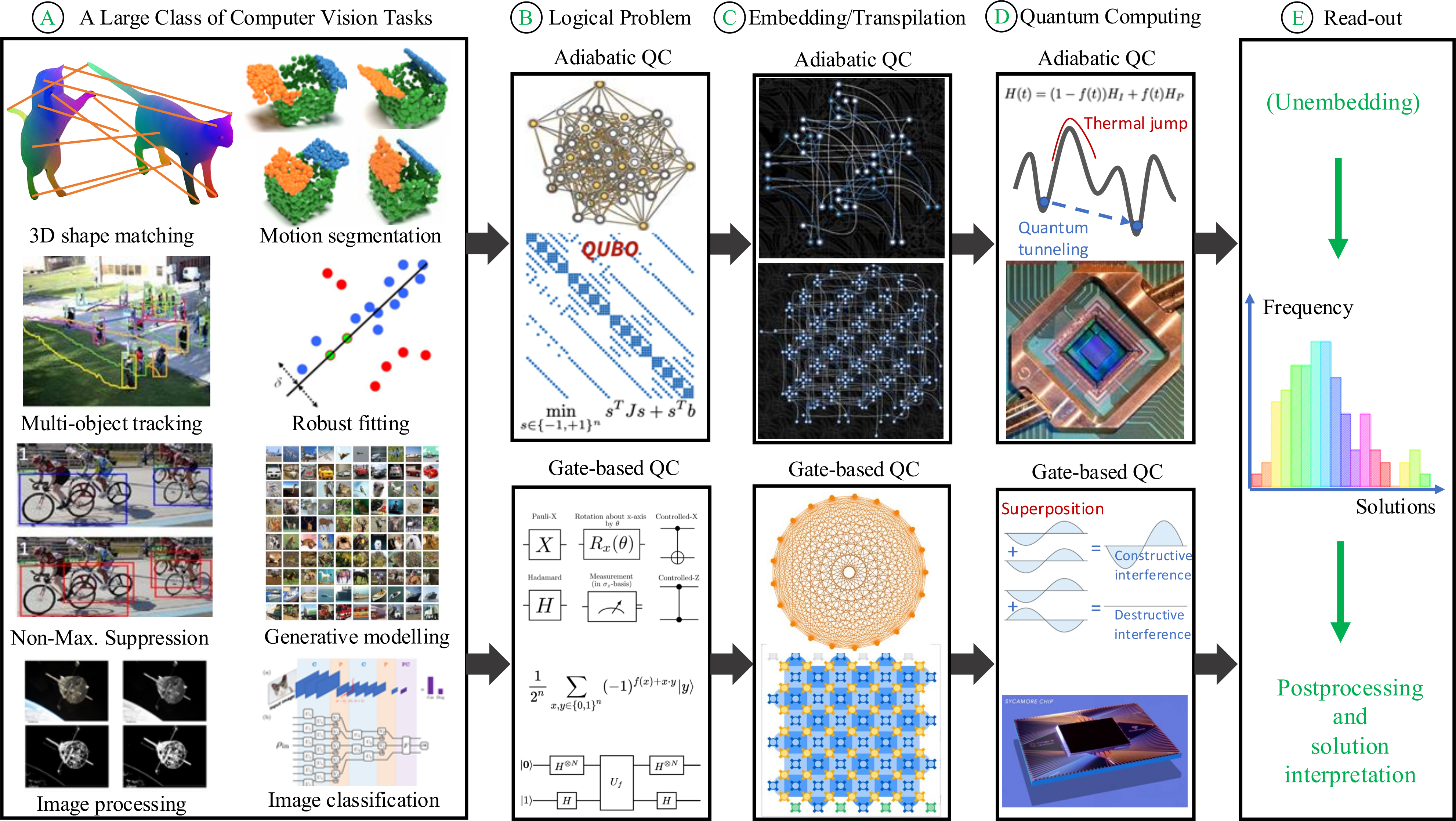}
	\end{center}
	\caption{\parahead{Quantum-\relax enhanced computer vision}. 
	(\textbf{A}): First, a target problem must be formulated in a form consumable by modern quantum machines, \eg~as a QUBO problem for AQC devices or as a gate sequence for gate-based QC.
	This operation is performed on a host (classical CPU).
	(\textbf{B}): In AQC, the resulting QUBO defines a \textit{logical} problem—binary variables that become qubits during optimisation on an idealised quantum annealer with full qubit connectivity. Alternatively, gate-based QC uses a gate sequence to drive the system into a solution-encoding state.
    (\textbf{C}): To run on a quantum computer with limited connectivity, a logical problem must be minor-embedded or transpiled. 
    During this mapping step, each logical qubit is assigned to one or more physical qubits to match hardware constraints.
	(\textbf{D}): An AQC device performs annealing for computation, while a gate-based QC device alternatively executes the algorithm describing gates. 
    Adiabatic computers leverage quantum mechanical effects of superposition and tunnelling to find optima of QUBOs.
    Gate-based computers can additionally harness entanglement and interference to speed up computations, surpassing the capabilities of classical ones.
    (\textbf{E}): Measured qubit values are unembedded from the hardware and aggregated in the AQC paradigm, or directly read out in gate-based QC. 
    The measurement is repeated several times, and a solution distribution is returned to the host. 
    The bit-strings are processed and interpreted in terms of the original problem. 
    Image sources, if applicable (from left to right and top to bottom in each step): 
    (\textbf{A}): 
    \cite[IEEE \textcopyright 2025]{QuantumSync2021},
    \cite[IEEE \textcopyright 2025]{huang2021multibodysync},
    \cite[IEEE \textcopyright 2025]{xiang2015learning},
    \cite{ransac},
    \cite[Springer Nature \textcopyright 2025]{LiGhosh2020},
    \cite[reproduced under the dataset’s academic license]{krizhevsky2009learning},
    \cite[Springer Nature \textcopyright 2025]{yan2023toward},
    \cite[Springer Nature \textcopyright 2025]{cong2019quantum},
    (\textbf{B}): 
    \cite[IEEE \textcopyright 2025]{QuantumSync2021},
    \cite{meyer2022survey},
    \cite[IEEE \textcopyright 2025]{yang2024robust},
    (\textbf{C}): 
    \cite[Springer Nature \textcopyright 2025]{yurtsever2022q},
    \cite[IonQ \textcopyright 2025]{ionq_aria},
    (\textbf{D}): 
    \cite[reproduced under the CC BY-AS 4.0 license]{dwave_wiki},
    \cite[reproduced under the CC BY 3.0 license]{sycamore_wiki}.
    } 
	\label{fig:qcv}
\end{figure*}

\subsection{Computer Vision meets Quantum Computing} 

In the broad sense, \textit{Quantum-enhanced Computer Vision (QeCV) encompasses computer vision methods and systems executed entirely or partially on quantum hardware.} 
The term ``enhanced'' refers to the fact that QeCV methods include classical parts to different degrees (\eg~data encoding, weight matrix preparation, classical neural network parts in hybrid quantum-classical neural architectures) and are boosted (or enhanced) through quantum computations. 
Hence, modern QeCV methods are hybrid and designed as an interplay between classical and quantum parts. 
QeCV falls into the category of applied quantum computing, perhaps among the first such disciplines across the fields. 

\begin{tcolorbox}[enhanced,attach boxed title to top center={yshift=-3mm,yshifttext=-1mm},
  colback=green!2,colframe=green!25!gray,colbacktitle=green!25!gray,
  title=Definition of QeCV,fonttitle=\bfseries, 
  boxed title style={size=small,colframe=green!25!gray} ]
  The goal of QeCV is the development of innovative computer vision techniques (improved or fundamentally new ones) leveraging quantum computational paradigms and surpassing classical methods in terms of processing speed, required resources, accuracy or the ability to learn patterns from complex visual data.
\end{tcolorbox}

We emphasise that this definition is intended to be aspirational, reflecting the evolving and exploratory nature of the field. 
The use of quantum hardware in QeCV approaches must be justified not only from a computational standpoint but also through demonstrable benefits in solution quality or predictive accuracy. 
In other words, it is insufficient to dub a method as \textit{quantum} solely because it can be executed on quantum hardware. 
Recall that a universal quantum computer can execute everything a classical binary machine can, but not the other way around (in reasonable time; classical computers can accurately simulate quantum computations with exponential growth of classical resources, which becomes infeasible for large numbers of qubits \cite{Zhou2020}). 
QeCV is an emerging field. 
The first QeCV method published at a primary computer vision conference was an approach for correspondence problems on point sets \cite{golyanik2020quantum}; it appeared on arXiv.org in 2019. 
The paper provides an introduction to modern quantum annealers and proposes algorithms for transformation estimation and point set alignment that can be executed on a quantum annealer once the inputs are pre-processed to a form admissible to it. 
The term \textit{Quantum Computer Vision}---or QeCV as it is more often called today---was coined later in Birdal and Golyanik~\etal \cite{QuantumSync2021} and since then, it is used to denote computer vision methods relying on quantum hardware. 

\parahead{Applications of quantum computers in vision} 
Not many problems in computer vision can be formulated in a form consumable by modern quantum hardware. 
The estimated number of qubits necessary for practical computer vision problems with gate quantum computing formulations that make use of provably advantageous quantum algorithms is typically larger than available on NISQ architectures. 
It is rare that QUBO forms are available and can be directly tested on an AQC. 
Often, it is the case that the target problem first has to be mapped to QUBO, and the solution has to be encoded in binary form. 
It is, therefore, of broad scientific interest that computer vision (and machine learning) problems can be formulated for quantum computing and efficiently solved with it, while offering advantages compared to their classical method counterparts. 
Moreover, using a quantum mindset to CV problems can provide a new perspective, leading to new insights for classical CV and new methods, especially for the field in which many things are empirical. 
We discuss these aspects in Sec.~\ref{sec:methods_domains}. 
\vspace{0.5mm}
\parahead{Related research fields} 
Several research fields related to QeCV can be identified in the literature, such as quantum-inspired computer vision, Quantum Image Processing (QIP), and Quantum Machine Learning (QML). 

The first category simulates quantum-mechanical effects or draws inspiration from quantum mechanics \cite{Aubry2011, Aytekin2014, Deng2014, SinghBose2021, Cosmo2022}.
These algorithms are not meant for execution on quantum hardware and are solely inspired by quantum phenomena. 
They should not be confused with techniques that can execute on quantum hardware. 
The second method category, \ie~QIP, is a field of quantum information processing focusing on representing and processing images as quantum states \cite{yan2016survey, chakraborty2018quantum, YanVenegasAndraca2020}. 
QIP provides several faster algorithms (in theory) for multiple widely-used linear algebra operations and common low-level operations on images \cite{Harrow2009, aaronson2015read, Gilyen2019}.

Both method categories above can be broadly seen as sub-fields of QeCV, though distinct from this survey’s focus on \emph{mid- and high-level} computer vision tasks (\eg~point set or mesh alignment, object tracking, and robust fitting). 
In contrast, just as classical ML is deeply intertwined with CV, quantum machine learning \cite{schuld2014quest,biamonte2017quantum,cerezo2022challenges} explores the intersection of quantum computing and machine learning to enhance learning algorithms. 
It holds promise for QeCV, potentially accelerating tasks such as image classification, object detection, and pattern recognition.
However, current QML methods remain largely theoretical, with limited practical applications due to quantum hardware and scalability constraints. 
Hence, QML remains distant from this survey’s focus on \emph{practical} applications to CV.

\subsection{Motivation and Scope of this  Survey}
\label{ssec:scope} 
In the broad sense, this survey shows by reviewing the first works in the field, \textit{how quantum computations and quantum hardware can be leveraged for computer vision.} 
It showcases recent and ongoing progress towards practical quantum computing and computer vision, discusses the current state of the art, limitations therein, expected progress and its impact on computer vision. 
Our goal is to provide a shared, computer-science-friendly language and mathematical formulation of quantum computing, covering its two modern paradigms, \ie~gate-based quantum computing and adiabatic quantum computation.
We identify and classify computer vision problems that can be addressed by quantum computers and analyse what they do have in common. 

We observe that the quantum technology acts more and more like a booster for algorithm development in CV. 
The fact that the developed methods could run on real quantum hardware often brings us to interesting discoveries. 
Thus, we can often obtain methods that work better than existing classical ones. 
There are several considerations regarding the gate-based model \textit{vs} adiabatic model. 
Gate-based QCs are currently severely limited in the number of qubits, their connectivity patterns and the accuracy of operations (gates) that can be applied to the qubits. Additionally, decoherence poses a significant obstacle. 
Thus, the largest number factorised on gate-based quantum hardware using the celebrated Shor's algorithm remains $21$ for more than ten years as of the moment of writing \cite{MartinLopez2013, SkosanaTame2021}. 
In contrast, the number of qubits in modern quantum annealers such as D-Wave is larger, which in combination with their connectivity and qubit characteristics allows for solving combinatorial optimisation problems in the Ising encodings of sizes relevant to computer vision and real-world applications. 
Also, qubit coherence times required for AQC are shorter than for gate-based machines, which partially explains the better scalability of quantum annealers compared to gate-based machines. 
Hence, the interest in especially applying AQC in computer vision has grown substantially over the last three years. 

This survey focuses on computer vision methods for gate-based quantum computers and quantum annealers that have been evaluated on real quantum hardware (in the case of adiabatic quantum annealers) or simulators of gate-based quantum computers. 
We include a few theoretical works on gate-based quantum computing in computer vision without experiments on a simulator though this survey is generally structured to reflect the applied focus of QeCV. 
Since computer vision is an applied science, we believe that the criterion of experimentation on quantum hardware is pivotal, especially because modern quantum machines already allow solving problems of sizes encountered in practical applications. 
That is why we believe it is important that the methods are evaluated and the exposition is not restricted to theoretical considerations. 

\vspace{0.5mm}
\parahead{Paper selection criteria} 
QeCV is a nascent field. 
In line with its aspirational definition, this survey presents an overview of methods designed for full or hybrid execution (\ie~classical-quantum) on quantum hardware, emphasising approaches that report theoretically grounded results with potential relevance to QeCV. 
Hence, the main paper selection criterion for this survey is experimental evaluation of the proposed techniques and, at least, some results obtained on real quantum hardware. 
Moreover, we focus on works published at computer vision conferences (CVPR,  ICCV and ECCV) and other venues (perhaps interdisciplinary) that are interested in the application of quantum computers in vision. 
We also include several technical reports on arXiv.org if they fulfil the main paper selection criterion. 
The authors made efforts to provide as complete review of the field as possible, but cannot claim completeness in every aspect, since the field is getting momentum and growing. 
We also recommend interested readers to read the discussed papers for more details. 
\parahead{Related surveys} 
Recently, Quantum Machine Learning (QML) has gained significant attention as it is now a fast-growing area.
The large body of works in the QML space has been reviewed in~\cite{divya2021quantum,ramezani2020machine}, whereas~\cite{li2020quantum,divya2021quantum,abbas2024challenges} also focus on quantum approaches for tackling optimisation problems. 
Out of all machine learning models, quantum deep learning requires special attention due to the impact of these learning machines in the field of AI. 
Massoli et al.~\cite{massoli2022leap} gather, compare and analyse the current state-of-the-art concerning Quantum Neural Networks (QNN). 
Yarkoni~\etal~\cite{yarkoni2022quantum}, 
on the other hand, look into possible industry applications of Quantum Annealing (QA); their survey is perhaps closest to our work in terms of focusing on an application area of QA.

Mohseni et al.~\cite{mohseni2022ising} review different methods for solving Ising problems (exactly or approximately) and discuss quantum annealing as one meta-heuristic. 
A few surveys focus on quantum annealing for physicists~\cite{Hauke_2020}, whereas our work is written for researchers with a computer vision background. 
The short survey by Larasati \textit{et al.}~\cite{Larasati2022} provides a summary of a few quantum computer vision methods (${<}5\%$ of what we cover) and 
is much less comprehensive than our article. 
Another brief paper by Mebtouche \textit{et al.}~\cite{mebtouche2024quantum} reviews the foundations of quantum computing, machine learning and multiple low-level vision and image processing techniques. 
In this regard, ours is the first comprehensive survey of its kind with an emphasis on QeCV. 

\parahead{Target audience} 
This survey \textbf{it is written for applied computer scientists, especially computer vision researchers and practitioners, who may have little to no background in theoretical physics.}
Unlike previous surveys~\cite{divya2021quantum,ramezani2020machine,li2020quantum,abbas2024challenges,massoli2022leap,mohseni2022ising,mebtouche2024quantum} on quantum algorithms, mostly published outside traditional computer vision venues, our goal is to make  QeCV accessible and actionable for a broader audience of computer vision researchers and engineers. 
By consolidating multiple contributions, we aim to bridge that gap and provide a comprehensive overview of existing QeCV methods, along with insights into future directions.

We present a practical ``cookbook'' for computer vision practitioners to begin exploring quantum-enhanced techniques today and to prepare for the growing impact of quantum technologies. 
While the survey avoids deep dives into physics and quantum theory, it includes sufficient technical detail to ensure a solid foundation for the presented results and the validity thereof.
Physical assumptions are clearly separated from the mathematical consequences relevant to algorithmic design.
The intended readership includes:
\begin{itemize}
    \item Computer vision researchers seeking to understand whether their work can benefit from quantum computing; 
    \item Computer vision practitioners interested in experimenting with QeCV algorithms or integrating quantum solvers into their workflows; 
    \item Computer science students at various levels who are curious about the intersection of computer vision and applied quantum computing.
\end{itemize}
We hope this survey will serve as a gateway for new researchers to enter the field and contribute to its development.

\subsection{Structure of this Survey} 
This survey is structured in six sections. 
Sec.~\ref{sec:operational_principles} reviews foundations of quantum computing relevant to quantum-enhanced computer vision methods including two main quantum computational paradigms, \ie~gate-based (or circuit-based) and adiabatic quantum computing, as well as the design and hardware of quantum computers. 
Moreover, as part of it, Sec.~\ref{sec:ProblMapMethGateBased} 
discusses the encoding of classical data as quantum states suitable for processing on quantum hardware, which is one of the essential steps in quantum-enhanced computer vision. 
Next, Sec.~\ref{sec:methods_domains} is devoted to algorithms and applications. 
It reviews methods for different problems such as point set alignment, mesh registration, object tracking, model fitting, quantum machine learning for vision, implicit representations and generative approaches, among others. 
We then discuss open challenges, specifics of the field and social implications in Sec.~\ref{sec:discussion}, and conclude in Sec.~\ref{sec:conclude}.

\section{Operational Principles of QCs}\label{sec:operational_principles} 
This section reviews the foundations of quantum computing necessary as a short introduction (or a refresher) to the field and sufficient for understanding the published literature reviewed in Sec.~\ref{sec:methods_domains}. 
Sec.~\ref{sec:fundamentals} introduces fundamentals such as notations and elementary operations on qubits. 
Secs.~\ref{ssec:gate_based} and \ref{ssec:adiabatic} describe the two main quantum computational paradigms—namely, gate-based quantum computing and quantum annealing, respectively—and Sec.~\ref{ssec:ConnectionsParadigms} establishes their connections and equivalency.
Sec.~\ref{eq:advantage} examines potential advantages of quantum computing over classical computing and Sec.~\ref{ssec:DesignArchitecture} discusses hardware designs and realizations of quantum computers. 
Note that we do not review complexity classes associated with the quantum computational paradigm, as their discussion is outside the scope. 
In this survey, it is sufficient to mention that quantum computers can accelerate computations relevant to computer vision and provide different and favourable properties in quantum-enhanced learning models; both of these characteristics provide strong motivation for the newly established research direction of QeCV. 
We recommend references~\cite{BernsteinVazirani1997, Watrous2009} to readers wishing to learn more about quantum complexity theory. 

\subsection{Fundamentals and Notations}
\label{sec:fundamentals} 
This background section introduces the fundamentals of quantum computing through both gate-based and adiabatic paradigms. 
It also establishes the notation and terminology used throughout the survey. 
For clarity and ease of reference, a comprehensive overview of symbols and acronyms is provided in Tabs.~\ref{tab:symbols} and~\ref{tab:acronyms}; note that the acronym list excludes method names and company identifiers.

\begin{table}
	\centering
	\caption{List of symbols used in the survey.} 
	\begin{tabular}{cl}\hline
		Symbol  & Description \\\hline\hline
		$\mathbb N$ & Set of natural numbers\\
		$\mathbb R$ & Set of real numbers\\
		$\mathbb C$ & Set of complex numbers\\
		$i$ & Complex unit (if not stated otherwise)\\
		$\hbar$ & Reduced Planck’s constant\\
		$\theta$ & Optimisable rotation angle\\
		$n$ & Number of qubits in the system\\
		& \\
		$\ket{\psi}$ & State vector of the system \\
		$\bra{\psi}$ & Conjugate transpose of $\ket{\psi}$\\ 
		$\langle\psi_1 |\psi_2\rangle$ & Complex inner product of $\ket{\psi_1}$ and $\ket{\psi_2}$ \\
		$\|\square\|$ & $L_2$-norm of $\square$\\
		$|\square|$ & Absolute value of $\square$\\
		$\rho$ & Density operator, \ie~$\rho = \ket{\psi}\bra{\psi}$\\
		& \\
		$\otimes$ & Kronecker product, Tensor product \\
		$G$ & Generator of a unitary gate \\
		$I$ & Identity matrix or operator (per context) \\
		$H$ & Hamiltonian or Hadamard gate (per context) \\
		$H_I$ & Initial Hamiltonian \\
		$H_P$ & Problem Hamiltonian \\
		$M$ & Measurement observable\\
		$U$ & Unitary operator\\
		$\sigma^{\{x, y, z\}}$ & Pauli-$X$, -$Y$ and -$Z$ operators \\
		$\mathrm{Tr}$ & Trace operator \\
		$\lambda$ & Eigenvalue or penalty factor (per context) \\
		& \\
		$T$ & Total time \\
		$\Delta t$ & Time step \\
		& \\
		$s$ & Ising variables, \ie~$s \in \{-1, 1\}^n$\\
		$x$ & QUBO variables, \ie~$x \in \{0, 1\}^n$\\
		$v$ & General binary variable, \ie~$v \in \{v_1, v_2\}^n$\\
		$J, b$ & Couplings and biases of an Ising problem\\
		$Q, c$ & Couplings and biases of a QUBO problem\\
		$W, w$ & Couplings and biases of a general quadratic\\
		& binary decision problem\\
		$A, b$ & Matrix and vector of linear constraints\\
		$f$ & Schedule or objective function (per context)\\
		$\mathcal L$ & Loss function of PQCs\\
		\hline
	\end{tabular}
	\label{tab:symbols} 
\end{table} 

\begin{table}
	\centering
	\caption{List of frequent acronyms used in the survey.} 
\begin{tabular}{cl}\hline
    Acronym  & Description \\\hline\hline
    AQC & Adiabatic quantum computing\\
    BNN & Binary neural networks\\
    CPU & Central processing unit\\
    CV & Computer vision\\
    CVPR & Computer Vision and Pattern Recognition\\
    ECCV & European Conference on Computer Vision\\
    GPU & Graphics processing unit\\
    GQC & Gate-based quantum computing\\
    ICCV & International Conference on Computer Vision\\
    MLP & Multi-layer perceptron\\
    ML & Machine learning\\
    NISQ & Noisy intermediate-scale quantum\\
    NN & Neural networks\\
    PQC & Parameterized quantum circuit\\
    QA & Quantum annealing\\
    QBO & Quantum binary optimization\\
    QC & Quantum computing\\
    QCT & Quantum complexity theory\\
    QCNN & Quantum convolutional neural network\\
    QCVML & Quantum computer vision and machine learning\\
    QDK & Quantum development kit\\
    QeCV & Quantum-enhanced computer vision\\
    QIP & Quantum image processing\\
    QML & Quantum machine learning\\
    QNN & Quantum neural networks\\
    QPU & Quantum processing unit\\
    QUBO & Quadratic unconstrained binary optimization\\
    SDK & Software development kit\\
    SSD & Sum of squared distances\\
    \hline
\end{tabular}
	\label{tab:acronyms} 
\end{table}

Many concepts in quantum computing have direct analogues in classical computing and optimization theory. 
For example, Hamiltonians represent energy functions, with eigenstates corresponding to energy levels and ground states denoting the lowest-energy configurations. 
Throughout the survey, we assume these physical Hamiltonians to be Hermitian operators, and unless otherwise stated, the quantum systems considered are closed—\ie, they do not exchange particles with their environment.
We adopt the bra–ket notation to concisely express familiar linear algebraic constructs such as row and column vectors, inner and outer products, and tensor products. 
This notation streamlines the presentation of quantum algorithms and aligns with conventions in both physics and quantum information theory.

\parahead{Single qubits} Let us start with the most fundamental building block and elementary information unit of a quantum computer, the \textit{qubit.}
\begin{assumption}{Representation of Qubits}{normalized}
	The information carrier in a quantum computing system called \textit{qubit} is described by a two-dimensional complex vector of length one, \ie~
	\begin{equation}
		\label{eq:normalized}
		\ket{\psi} \in \C^2, ~ \| \ket{ \psi }\|^2   = 1
	\end{equation}
\end{assumption}

	We adopt the widely-used \textit{bra--ket} notation common in physics and quantum computing to write vectors and their conjugate transposes: 
	$\ket{\psi}$ (ket) denotes a column vector and $\bra{\psi}=\ket{\psi}^*$ (ket) denotes its conjugate transpose. 
	The multiplication of a bra- and a ket-vector $\bra{\psi_1}$ and $\ket{\psi_2}$, denoted $\braket{\psi_1|\psi_2}$, results in their inner product in the $\mathbb{C}^2$ Hilbert space\footnote{Note that in quantum mechanics, states can also have higher dimension than qubits. For non-discrete quantities like momentum or position, these states are not finite-dimensional but functions from the Hilbert space $\mathcal{L}^2(\mathbb{R}^3)$, \ie~a so-called Lebesgue space. In this case, the bra vectors can be understood as functionals from $\mathcal{L}^2$. 
    Furthermore, even the formulation on $\mathcal{L}^2$ still has shortcomings and a full mathematical treatment would have to resort to rigged Hilbert spaces to describe distributions and deal with unbounded operators~\cite{riggedHilbert}.}. 
	We can write an arbitrary two-dimensional complex vector as a column vector:
	\begin{equation}
		\ket{\psi}= \begin{bmatrix}
			a +ib \\
			c + id \\
		\end{bmatrix},
	\end{equation}
	with real coefficients $a,b,c,d \in \R$. 
	The normalisation condition of the qubit's state vector in Eq.~\eqref{eq:normalized} then yields $a^2 + b ^2 + c^2  + d^2 = 1$.

	Next, we translate a quantum-physical phenomenon, namely that particles can also be in \textit{superposition} of states.
	Only after measurement do they collapse to one of the classical states with a certain probability. 
	This is modeled mathematically in the following way: Let $\ket{0} \in \C^2$ and $\ket{1}\in \C^2$ form an orthonormal basis of $\C^2$ (whose precise form depends on the setup of the physical system). For example, one can have in column vector notation: 
	\begin{equation}
		\ket{0} = \begin{bmatrix}
			1 \\
			0 \\
		\end{bmatrix},
		\quad 
		\ket{1} = \begin{bmatrix}
			0 \\
			1 \\
		\end{bmatrix}. 
	\end{equation}
	
	\begin{assumption}{Measurements}{measurements}
		When the state of a qubit $\ket{\psi} = \alpha \ket{0} + \beta \ket{1}\in  \mathbb{C}^2$ is \textit{measured} (with respect to the basis $\{\ket{0},\ket{1}\}$), it results in a state 
		\begin{equation}\label{eq:measurement}
			\begin{cases}
				\ket{0} & \quad \text{with probability } |\alpha|^2 = |\braket{0|\psi}|^2, \\
				\ket{1} & \quad \text{with probability } |\beta|^2 = |\braket{1|\psi}|^2 .
			\end{cases}
		\end{equation}
	\end{assumption}
	In other words, a qubit exists in a superposition of classical states as $\ket{\psi} = \alpha \ket{0} + \beta \ket{1}$, where $\alpha$ and $\beta$ are \emph{probability amplitudes}. 
	Upon measurement, the qubit collapses into either $\ket{0}$ or $\ket{1}$, with the \emph{measurement probability} given by the square of the absolute value of the respective amplitude: $\lvert \alpha \rvert^2$ for $\ket{0}$ and $\lvert \beta \rvert^2$ for $\ket{1}$. 
	This is also called \textit{collapse of the wave function}. 
	The act of measurement w.r.t.~a basis changes the state into one of the basis elements, with probabilities defined by the projections of the state onto the basis. 

	\vspace{0.5mm}
	\parahead{Bloch sphere} 
	As the probabilities of obtaining certain measurement outcomes depend only on the \textit{magnitude} of the coefficients, it is easy to see that a change of global phase, \ie~a multiplication by a global factor $ e^{i \phi} $ with $ \phi \in \mathbb{R} $, does not affect any measurable properties of a qubit. Due to this ambiguity in the complex phase, it is common to fix $ \alpha $, the coefficient for the first basis vector $ \ket{0} $, to be real, \ie~ $ \alpha \in \mathbb{R} $, thereby resolving the ambiguity. 
	
	Along with the normalisation condition $ |\alpha|^2 + |\beta|^2 = 1 $ as given in Eq.~\eqref{eq:normalized}, any qubit state can be expressed as
	\begin{equation}
		\ket{\psi} = \cos(\theta/2) \ket{0} + e^{i \varphi} \sin(\theta/2) \ket{1},
	\end{equation}
	The two angles $ \theta \in [0,\pi] $ and $ \varphi \in [0,2\pi] $ naturally define a point on the unit sphere in three dimensions, known as the \textit{Bloch sphere}. The state of a qubit $ \ket{\psi} $ is frequently visualised in this representation; see Fig.~\ref{fig:bloch_sphere}. 
	
	\begin{figure}[t]
		\centering
		\includegraphics[width=0.35\textwidth]{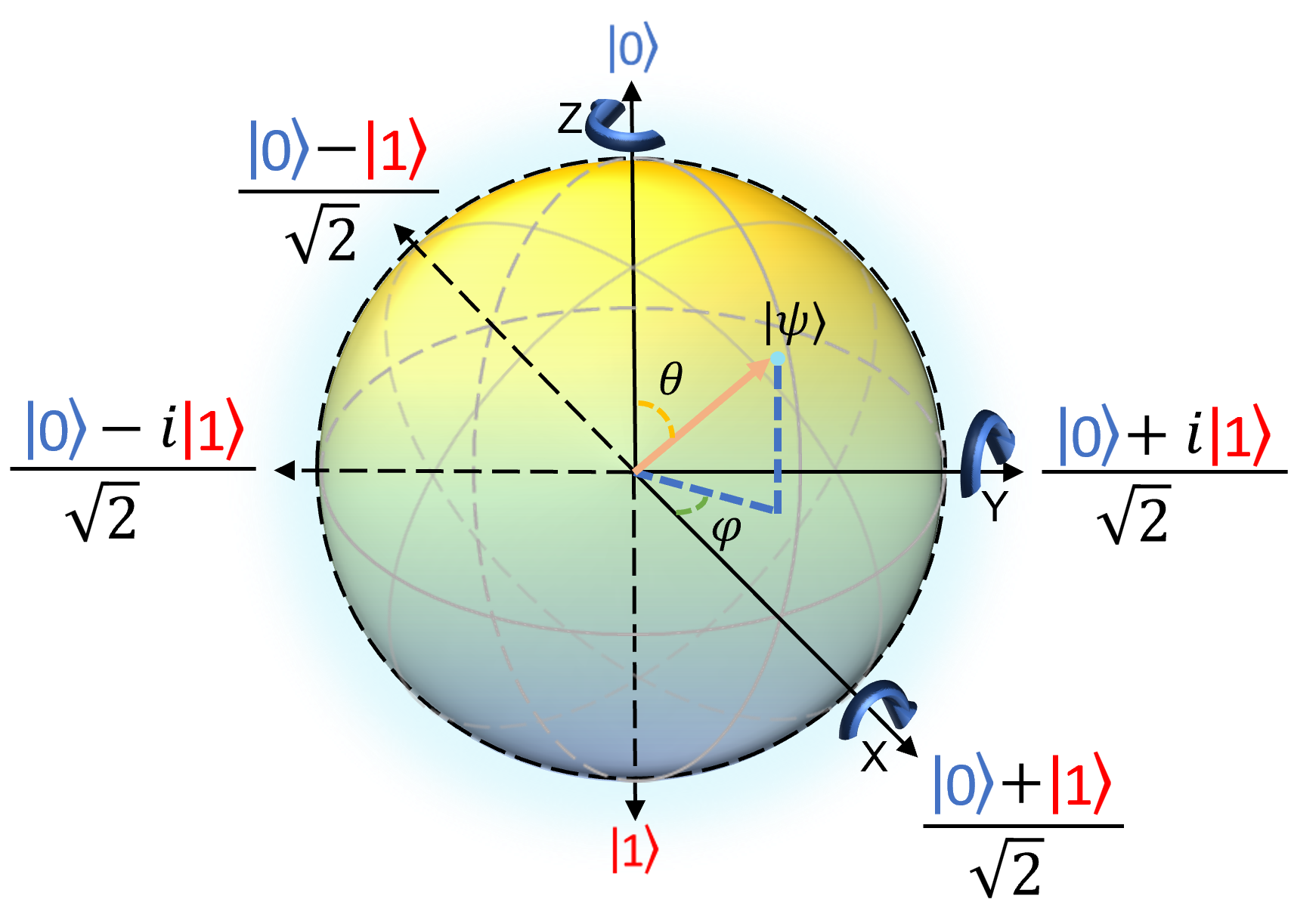}
		\caption{Visualising an arbitrary state of a qubit $\ket{\psi}$ on the Bloch sphere along with its several widely encountered states. Although in the original space $\mathbb{C}^2$ the states $\ket{0}$, $\ket{1}$ are orthogonal, they are visualised as opposite poles on the Bloch sphere. 
		} 
		\label{fig:bloch_sphere}
	\end{figure}
	
	\vspace{0.5mm}
	\parahead{Multi-qubit systems}
	When multiple qubits $\ket{\psi_1},\ket{\psi_2}, \hdots, \ket{\psi_n}$ are considered, their overall state $\ket{\psi}$ is described by the \textit{tensor or Kronecker product} of the individual qubit states and results in a vector consisting of all combinations of products of the individual coefficients:
	\begin{equation}
		\label{eq:tensorproduct}
		\ket{\psi} = \ket{\psi_1} \otimes \ket{\psi_2} \otimes \hdots \ket{\psi_n} \in \C^{2^n}. 
	\end{equation}
	For two qubits $\ket{\psi_1} = \alpha \ket{0} + \beta \ket{1}$ and $\ket{\psi} = \gamma \ket{0} + \delta \ket{1}$, this product state reads: 
	\begin{equation}\label{eq:kronecker2}
		\ket{\psi} = \ket{\psi_1} \otimes \ket{\psi_2} =
		\begin{bmatrix}
			\alpha \gamma \\
			\alpha \delta \\
			\beta \gamma \\
			\beta \delta \\
		\end{bmatrix}.
	\end{equation}
	A useful shorthand notation for the Kronecker product is
	\begin{equation}
		\ket{\psi} = 
		\ket{\psi_1} \otimes \ket{\psi_2} \otimes \hdots \otimes \ket{\psi_n} = \ket{\psi_1\psi_2 \hdots \psi_n},
	\end{equation}
	and such a system of multiple qubits is called \emph{quantum register}.

	Note that although the combination of multiple qubits is represented in a $2^n$-dimensional complex vector space, one would expect that actual (physically meaningful) vectors occupy a low dimensional subset of only those $2^n$ dimensional vectors that can be decomposed into a tensor product \eqref{eq:tensorproduct}. 
	Note, for example, that the tensor product of two vectors $\ket{\psi_1}$ and $\ket{\psi_2}$ is (a vectorisation of) their outer product, such that 2-qubit-states of the form \eqref{eq:kronecker2} can be identified with rank-1 matrices in the space of all complex $2{\times}2$ matrices. The extremely interesting and important concept of \textit{entanglement} implies that multiple-qubit systems are \textbf{not} limited to so-called \textit{separable} states that decompose as Eq.~\eqref{eq:tensorproduct}. 
	The state 
	\begin{equation}
		\ket{\psi} = \frac{1}{\sqrt{2}}(\ket{01} +\ket{10}) = \frac{1}{\sqrt{2}}
		\begin{bmatrix}
			0 \\
			1 \\
			1 \\
			0 \\
		\end{bmatrix},
	\end{equation}
	for example, cannot be decomposed as $\ket{\psi_1} \otimes \ket{\psi_2}$ from Eq.~\eqref{eq:kronecker2}
	since there are no $\alpha,\beta,\gamma,\delta$ that can simultaneously satisfy the equations $\alpha \gamma=0, ~ \alpha \delta =  1/\sqrt{2}, ~ \beta \gamma= 1/\sqrt{2}, ~ \beta \delta = 0$.
	This is one of the famous Einstein-Podolsky-Rosen states~\cite{nielsen2002quantum}.
	\begin{assumption}{Entanglement}{entangle}
		A system of $n$ many qubits can evolve into \textit{any} state in $\C^{2^n}$. 
        States that can be represented as tensor-product states via Eq.~\eqref{eq:tensorproduct} are called \textit{separable}, states that do not admit such a representation are called \textit{entangled}.  
	\end{assumption}
	Note that Assumption 2 extends to multi-qubit systems, irrespective of whether the system is separable or entangled. 
	In other words, the probability of measuring a specific multi-qubit basis state is equal to the squared magnitude of the corresponding coefficient.
	
	\parahead{Qubit evolution} 
	It is possible to manipulate the state of an $n$-qubit quantum physical system experimentally.
    For example, let $\ket{\psi(0)}$ be the state of the quantum system at time $t$.
	We can prepare an initial state $\ket{\psi(0)}$ and manipulate it at any time $t$ with external influences $H(t) \in \C^{2^n \times 2^n}$  determining the system's energies by letting the $n$-many qubits experience a certain coupling. 
	The way the state vector behaves under such an evolution is described by the famous Schrödinger Equation.
	\begin{assumption}{Schrödinger Equation}{schroed}
		The (time) evolution of a quantum state $\ket{\psi(t)}$ is described by the \textit{Schrödinger Equation} 
		\begin{equation}
			\label{eq:schroedinger_equation}
			i \hslash \frac{d}{dt}\ket{\psi(t)} = H(t) \ket{\psi(t)},
		\end{equation}
		where the \textit{Hamiltonian} $H(t) \in \C^{2^n \times 2^n}$ is a Hermitian matrix determined by the experimental setup, $i$ is the imaginary unit and $\hslash$ is the reduced Planck constant. 
	\end{assumption}
	A closer look at the solutions to this differential equation reveals that they 
	follow some unitary time evolution~\cite{nielsen2002quantum}.
	This is in particular consistent with what we learned that quantum states have norm 1, since the length of the components is related to measurement probabilities. 

    \subsection{Gate-Based Quantum Computers}
    \label{ssec:gate_based}
    
	As alluded to in Sec.~\ref{sec:fundamentals}, the Schrödinger equation~\eqref{eq:schroedinger_equation} governs the evolution of a closed quantum system over time. 
	Moreover, it can be proven that solving the Schrödinger equation always leads to a unitary time evolution~\cite[Chapter 2]{nielsen2002quantum}. 
	Gate-based quantum computers manipulate quantum states through a controlled sequence of unitary Hamiltonian evolutions---each abstracted as a \emph{quantum gate}---to perform computations. The specific algorithm implemented depends on the particular quantum gates chosen (Sec.~\ref{sec:broadoverviewgate}). While the earlier quantum algorithms, including the famous algorithms of Shor~\cite{Shor1997} and Grover~\cite{grover1996fast}, are based on ``handcrafted'' quantum gates, increasingly quantum gates are selected by ``training'' on data to optimise an objective function~\cite{cerezo2021variational,peruzzo2014variational,Farhi2014} (Sec.~\ref{sec:composegates}).

	\subsubsection{An Overview of Gate-based Quantum Algorithms}
	\label{sec:broadoverviewgate}
	
	Since algorithms under the gate-based model are often defined using a sequence of quantum gates, quantum algorithms are often framed as \emph{quantum circuits}. 
    Fig.~\ref{fig:qcircuit} provides an overview of the major steps in a quantum algorithm, displayed as an $n$-qubit quantum circuit. First, a basic quantum state (\eg~one of the classical states) is generated. Typically, the basic state is brought into superposition as a form of initialisation. Then, a sequence of quantum gates is applied to the initial state to achieve a final quantum state. 
    Since a sequence of non-interrupted (\eg~without intermediate measurements) unitary transformations is equivalent to a single unitary transformation, the successive quantum gates that define the algorithm can be seen as a single quantum gate ``block''. 
    The final quantum state is not directly usable until it is \emph{measured}, at which point it collapses to yield classical (binary) information. 
    This collapse reduces the degrees of freedom available for further computation, which is why many quantum algorithms repeat the initialise–transform–measure cycle multiple times—often with intermediate feedback—to amplify the probability of success and extract reliable outcomes from probabilistic quantum processes.
    Measurement mathematically manifests in projecting the quantum state onto the eigenbasis of a Hermitian operator known as an \emph{observable}, with the output given by the eigenvalue of one of its eigenvectors (see Eq.~\eqref{eq:measurement} for an elementary example). 
    The choice of observable depends on the algorithm, while the likelihood of obtaining a particular outcome is determined by the final quantum state.

	\begin{figure}[t]\centering
		\vspace{-.1cm}
		\includegraphics[scale=.98]{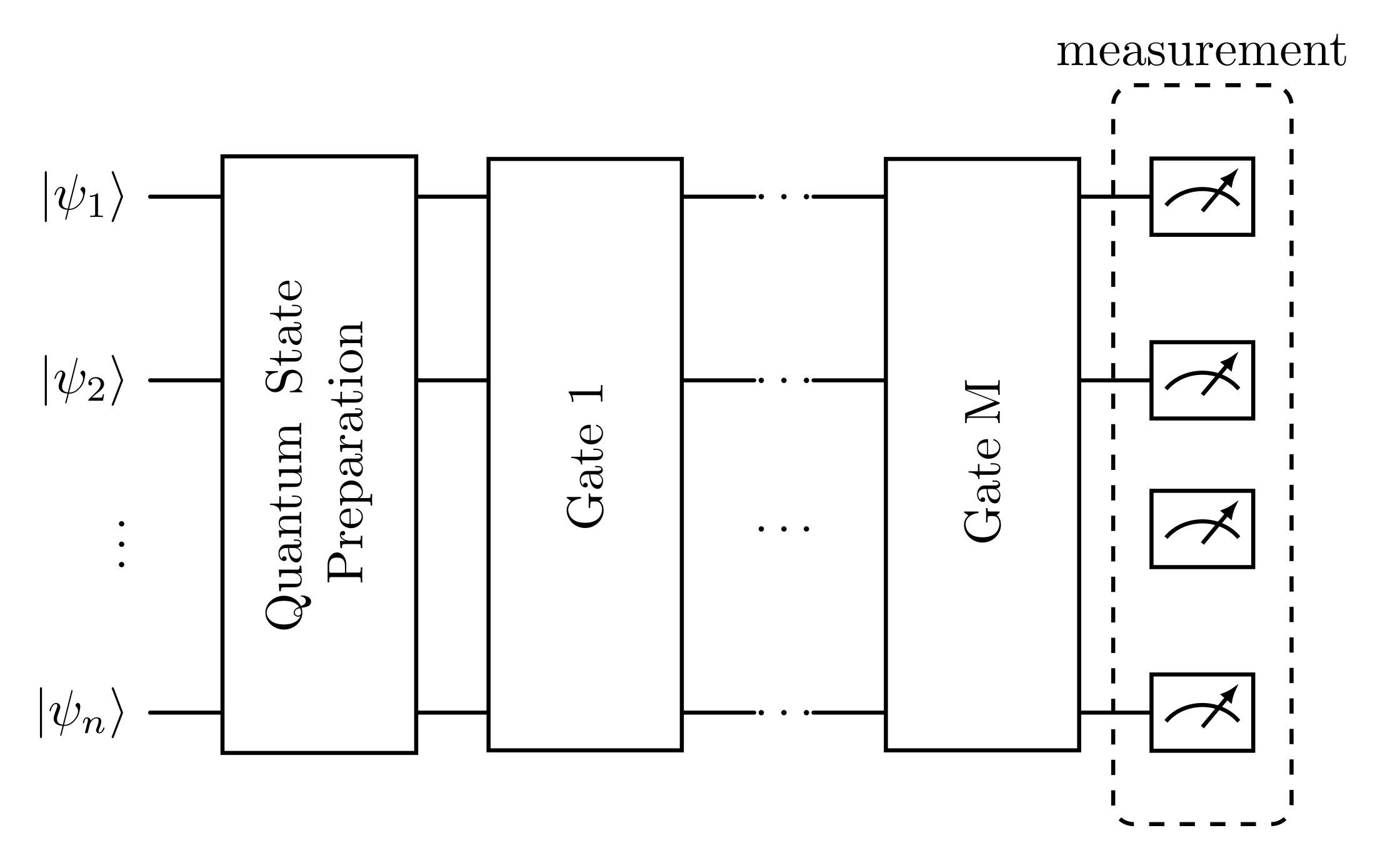}
		\caption{ 
			Common processing stages in a typical quantum circuit. Here, the quantum circuit operates on $n$ qubits, which at the beginning are often initialised to a basic state, \eg~$\ket{\psi_1 \psi_2 \dots \psi_n} = \ket{10\dots 0}$. Then, the basic state is prepared to yield an initial (usually superimposed) state, before a sequence of $M$ quantum gates is invoked. The specific gates employed define the algorithm implemented. The final quantum state is then measured to produce classical outcomes that correspond to the final results, which are often probabilistic. It is also common for a circuit to be repetitively invoked or iterated. 
		}
		\label{fig:qcircuit}
	\end{figure}

	Similar to classical algorithm design---where we are not concerned with the underlying implementation of the basic computational units (\ie~logical gates)---in quantum algorithm design, details of the physical realisation of the quantum gates and measurement devices are typically also abstracted away. 
	
	Constraining the operations to be unitary transformations can seem restrictive. 
	However, it can be shown that all (classical) logical gates can be made reversible without significantly adding more bits or logical operations, thereby allowing the logical gates to be interpreted as unitary transformations~\cite[Sec.~1.4]{nielsen2002quantum}. 
	It turns out that any efficient classical algorithm can also be efficiently executed on a quantum computer in a analogous way~\cite[Sec.~3.2]{nielsen2002quantum}. 
	However, significant interest in gate-based quantum computers draws from their ability to solve certain problems that are beyond the reach of classical machines~\cite{Shor1997,grover1996fast} (Sec.~\ref{sec:gatespeedup} in the Appendix provides a rudimentary example of an algorithm with a theoretical speedup due to quantum effects). 
	%
	Before reviewing some applications of gate-based quantum computing to QeCV, it is essential to elaborate deeper on quantum gates and quantum circuits. 

	\subsubsection{Quantum Gates and Circuits}
    \label{sec:composegates}
	
	As alluded above, all classical logical gates (\eg~AND, OR, NAND) can be made reversible, \ie~the inputs can be recomputed from the outputs. This can be achieved through the Toffoli gate, which has three input and three output bits\footnote{A basic requirement for a reversible gate is that the number of input and output bits are the same.}. By hardwiring selected inputs (to 0 or 1), the Toffoli gate can simulate all logical gates in a reversible manner. Details of reversible logical gates are beyond the scope of this survey; for that, we refer the reader to~\cite[Sec.~1.4]{nielsen2002quantum}. 

	\parahead{Single-qubit gates}
	Single-qubit gates are elementary gates acting on one single input qubit $\ket{\psi}$. 
	For example, the Pauli gates
	\begin{equation}
		\label{eq:pauli}
		X= \left[ \begin{matrix}
			0&1 \\
			1&0 
		\end{matrix} \right], \quad
		Y= \left[ \begin{matrix}
			0&-i \\
			i&0 
		\end{matrix} \right],
		\quad
		Z= \left[ \begin{matrix}
			1&0 \\
			0&-1 
		\end{matrix} \right],
	\end{equation}
	are used in many quantum algorithms.
	They are often graphically displayed as follows: 
	\begin{center}
		\begin{quantikz}
			\ket{\psi}&\gate{X}&\qw \;\;
			\ket{\psi}&\gate{Y}&\qw \;\;
			\ket{\psi}&\gate{Z}&\qw
		\end{quantikz}
	\end{center}
	Defining quantum gates as matrices, as in Eq.~\eqref{eq:pauli}, allows quantum operations to be performed as multiplications of unitary matrices, which can be verified: 
	Consider the vector form of a one-qubit state
	\begin{align}
		\ket{\psi} = \alpha\ket{0} + \beta\ket{1} = \alpha \left[ \begin{matrix} 1 \\ 0 \end{matrix} \right] + \beta \left[ \begin{matrix} 0 \\ 1 \end{matrix} \right] = \left[ \begin{matrix} \alpha \\ \beta \end{matrix} \right].
	\end{align}
	Applying the Pauli-$X$ on $\ket{\psi}$ implies conducting a matrix-vector multiplication, yielding $X\ket{\psi} = \beta\ket{0} + \alpha\ket{1}$.
	In particular, setting $\alpha = 1$ and $\beta = 0$ yields $X\ket{\psi}=\ket{1}$, which can be interpreted as ``flipping'' the basic state of $\ket{0}$ to $\ket{1}$ (and vice versa by if $\alpha = 0$ and $\beta = 1$). 
	Indeed, the Pauli $X$ gate is actually a (reversible) NOT gate. 

	Another basic quantum gate is the  Hadamard gate
	\begin{align}
		H= \frac{1}{\sqrt{2 }}\left[ \begin{matrix}
			1&1\\
			1&-1
		\end{matrix} \right],
	\end{align}
	which is commonly used to bring basic states into superposition (and vice versa). 
	\eg~applying $H$ on $\ket{1}$ yields $H\ket{1} = \frac{1}{\sqrt{2}}(\ket{0} - \ket{1})$,
	\ie~a state of uniform superposition. 
	Applying $H$ again on the above result yields the initial basic state: $H(H\ket{1}) = HH\ket{1} = \ket{1}$.
	As can be deduced, applying quantum gates successively is equivalent to performing a quantum operation that is defined by the multiplication of the corresponding unitary matrices. 
	The following quantum circuits depict the application of $H$ and $HH$ on the input qubits $\ket{\psi}$:
	\begin{center}
		\begin{quantikz}
			\ket{\psi}&\gate{H}&\gate{H}&\qw
		\end{quantikz} $\equiv$
		\begin{quantikz}
			\ket{\psi}&\gate{I_2}&\qw
		\end{quantikz}
	\end{center}

	\parahead{Composition of single-qubit gates}
	As introduced in Sec.~\ref{sec:fundamentals}, the state space of a $n$-qubit system is the tensor product space $\C^{2^n}$. Accordingly, single-qubit quantum gates (specifically, their matrix representations) should also be tensored to obtain the corresponding $n$-qubit quantum gate. 
	An example is the $2$-qubit quantum gate $H^{\otimes 2}$, which has the following circuit representation: 
	\begin{center}
		\begin{quantikz}
			\ket{\psi_1}&\gate{H}& \qw \\
			\ket{\psi_2}&\gate{H}& \qw
		\end{quantikz} $\equiv$ \begin{quantikz}
			\ket{\psi_1}&\gate[2]{H^{\otimes 2}}& \qw\\
			\ket{\psi_2}&&\qw
		\end{quantikz}
	\end{center}
	In matrix form, the $H^{\otimes 2}$-gate is given by
	\begin{align}
		H^{\otimes 2} = H \otimes H = \frac{1}{2}\left[ \begin{matrix} 1 & 1 & 1 & 1 \\
			1 & -1 & 1 & -1 \\
			1 & 1 & -1 & -1 \\
			1 & -1 & -1 & 1    
		\end{matrix} \right].
	\end{align}
	The application of a composite gate on a separable quantum state follows the calculation rule of the Kronecker product: With appropriate dimensions, the product $(A\otimes B) \cdot (x \otimes y)$ of two Kronecker products is the Kronecker product $ (Ax)\otimes (By) $ of two products.
	Hence, applying for example $H^{\otimes 2}$ on the 2-qubit quantum state $\ket{01}$ results in $H^{\otimes 2}\ket{01} = H\ket{0} \otimes H\ket{1}$, which is a 2-qubit state in uniform superposition. It is easy to verify that sequentially applying $H^{\otimes 2}$ twice is the same as performing the identity operation~$I_4$.
	
    \parahead{Controlled gates}
	Entangled quantum states cannot be constructed by applying composite single-qubit gates to a separable initial state.
	Those particular states are obtained with controlled gates, which condition the application of a gate to a set of qubits on the state of other qubits.
	A popular representative of controlled gates is the controlled-NOT (CNOT) gate.
	It has the following circuit representation:
	\begin{center}
		\begin{quantikz}
			\ket{\psi_1}&\ctrl{1}&\qw \\
			\ket{\psi_2}&\targ{}&\qw
		\end{quantikz}
	\end{center}
	In the above circuit, we apply the $X$ or NOT gate on qubit $\ket{\psi_2}$ if qubit $\ket{\psi_1}$ is in the $\ket{1}$-state and do nothing otherwise.
	The CNOT gate has the following matrix representation: 
	\begin{align}
		\mathrm{CNOT}  = 
		\left[ \begin{matrix} 1 & 0 & 0 & 0 \\
			0 & 1 & 0 & 0 \\
			0 & 0 & 0 & 1 \\
			0 & 0 & 1 & 0    
		\end{matrix} \right].
	\end{align}
	Example applications of CNOT on basis states yield $\mathrm{CNOT}\ket{00} = \ket{00}$ and $\mathrm{CNOT}\ket{10} = \ket{11}$.

	Similar to classical computers, where any logical circuit of arbitrary complexity can be composed from a small set of universal logical gates (\textit{e.g} NAND by itself is a universal logical gate), there exist \emph{universal quantum gates}, which are a small set of one- or two-qubit gates (\eg~the Hadamard, controlled-NOT and phase shift make up a set of universal quantum gates). Through more elaborate applications of sequential multiplication and tensor operations, a set of universal quantum gates can simulate any unitary transformation in $\C^{2^n}$ up to arbitrary accuracy~\cite[Chapter 4.5]{nielsen2002quantum}. Hence, in theory, a gate quantum computer needs only to physically implement a small set of unique quantum gates. 

	\parahead{Parameterised gates} Devising quantum algorithms by handcrafting quantum circuits can be non-intuitive. Increasingly, quantum circuits are learned or optimised from data, typically in conjunction with classical optimisation. To this end, parametrised quantum gates play a crucial role. 

	As we have seen so far, all quantum gates are unitary operators preserving the magnitude of the state vector.
	So, it is natural to think of them as rotations on the Bloch sphere around a specific angle and axis.
	Indeed, up to a global phase factor $e^{i\varphi}$ that is negligible in measurement, any single-qubit gate $U$ can be expressed as a special case of a parameter-dependent operator
	\begin{equation}
		\label{eq:genral_unitary}
		U(\theta) = \exp(i\theta G) = \cos(\theta)I + i \sin(\theta) G,
	\end{equation}
	where $\theta\in \mathbb{R}$ is the gate parameter, $I$ is the identity and $G$ is a certain unitary and Hermitian operator called \textit{generator} of $U$.
	The operator $U(\theta)$ is a rotation of angle $\theta$ around the axis supporting the eigenvectors of $G$.
	Common generators are Pauli-$X,Y,Z$ matrices, cf. Eq.~\eqref{eq:pauli}, which turn $U(\theta)$ into a rotation of angle $\theta$ and about the $x,y,z$ axis respectively.
	It is easy to verify that $U(\theta)$ in  Eq.~\eqref{eq:genral_unitary} is unitary. 
	Parameterised quantum gates play an important role in variational quantum computation~\cite{cerezo2021variational,peruzzo2014variational,Farhi2014}. 
	They can be used to calculate the ground state (\ie~ the Hamiltonian eigenvector which returns the lowest eigenvalue) of a certain Hermitian operator or Hamiltonian $M$, which cannot be efficiently calculated classically. 
	The approach is to design a so-called \emph{Parameterised Quantum Circuit (PQC)} consisting of a unitary block $U(\theta)$, made up of several parameterised and/or controlled gates, which acts on an initial state vector $\ket{\psi}$, resulting in a parameterised eigenvector $\ket{\psi(\theta)} = U(\theta)\ket{0}$. 
	The parameter vector $\theta$ is then the unknown that needs to be tuned in a gradient-based or gradient-free manner to minimize the cost function $\braket{\psi(\theta)|M|\psi(\theta)}$.
	Similar to classical learning-based paradigms, PQCs 
	can also be \emph{trained} to perform other tasks with proper $M$ and circuit $U(\theta)$ defined. Typically, the loss function for training a PQC has the following form:
	\begin{equation}
		\mathcal L(\theta) = \mathbb{E_x}\left[ f_x\left(\mathrm{Tr}(\rho_x(\theta)M_x)\right) \right],
	\end{equation}
	where $x$ are training samples, $\rho_x(\theta) = \ket{\psi_x(\theta)}\bra{\psi_x(\theta)}$ is the density operator describing the quantum state before measurement, $M_x$ is a certain problem-dependent observable, $\mathrm{Tr}$ is the trace operator and $f_x$ is the function that pushes the quantum circuit to learn the pattern of the data.
	Standard $f(x)$ include mean-squared errors, mean-averaged errors for regressions and others for classifications. 
	In simple terms, $\mathrm{Tr}(\rho_x(\theta)M_x)$ is the expectation value on the observable $M_x$ on the quantum state $\ket{\psi_x(\theta)}$, so $\mathrm{Tr}(\rho_x(\theta)M_x) = \braket{\psi_x(\theta)|M_x|\psi_x(\theta)}$.
	PQCs are differentiable and their gradient can be evaluated using the so-called parameter-shift rule~\cite{mitarai2018quantum}.
	Small-scale PQCs can even be simulated classically, in which case auto-differentiation and back-propagation can be used for the training. 
	PQCs belong to a larger class of quantum algorithm named \emph{Quantum Machine Learning (QML)}~\cite{cerezo2022challenges,biamonte2017quantum,schuld2015introduction,schuld2018supervised}.

	\subsection{Adiabatic QC and Quantum Annealing}
    \label{ssec:adiabatic} 
	Adiabatic Quantum Computing (AQC) is another quantum computational paradigm different from gate-based quantum computing discussed in the previous section. 
	Instead of representing unitary transformations using a sequence of well-defined (pre-defined or learned) elementary unitary transformation blocks, AQC performs useful calculations through a continuous transition between \emph{Hamiltonians}.
	To put it simply, in classical computing terms, Hamiltonians can be thought of as energy functions; those are mathematical expressions that describe how energy is distributed in a system.
	A transitioning, or a time-dependent Hamiltonian, then acts as an evolving energy landscape, with the goal being to guide the quantum system toward an optimal solution. 
	The choice of this Hamiltonian is crucial and structured in a way that naturally aligns with the optimization problem being solved.
	Next, we discuss the operational principle of AQC grounded on the adiabatic theorem of quantum mechanics ~\cite{BornFock1928}. 

	Suppose we are given a fixed initial Hamiltonian $H(0)=H_I$ and that it is physically possible to create an initial quantum state $\ket{\psi(0)}$ that is an eigenvector to the smallest eigenvalue of the operator $H(0)$. 
	Moreover, it is possible (with certain system-dependent limitations) to create $n^2$ many interactions, so-called \emph{couplings}, $J_{i,j}\in \mathbb{R}$ between the individual qubits along with $n$ \emph{biases} $b_i \in \mathbb{R}$ acting on each qubit that amount to a Hamiltonian of the form 
	\begin{eqnarray}
		\label{eq:problem_hamiltonian}
		H_P &=& \sum_{i,j} J_{i,j} \sigma_i^z \sigma_j^z + \sum_i b_i \sigma_i^z,\\
		\sigma_i^z &=& \underbrace{I \otimes I \hdots \otimes  I}_{(i-1)\text{-many times}} ~\otimes \sigma_z \otimes \underbrace{ I \otimes  \hdots \otimes I}_{(n-i)\text{-many times}}, \\
		\sigma_z &=& \begin{bmatrix} 1 & 0 \\ 0 & -1 \end{bmatrix}, \,\text{and} \ \ I = \begin{bmatrix} 1 & 0 \\ 0 & 1 \end{bmatrix}. 
	\end{eqnarray} 
	One can show that these $n^2+n$ many terms lead to a Hamiltonian $H_P$ that is a diagonal $2^n \times 2^n$ matrix whose diagonal entries are the enumeration of all costs obtained via $s^\top  J s + s^\top  b$ for $s \in \{-1,+1 \}^n$; see Sec.~\ref{sec:eigenvalues_of_H_P} for details. 
	Thanks to this enumeration, the eigenvector $\ket{\psi}$ to the smallest eigenvalue of $H_P$ is a unit vector whose only entry equal to one can be identified with a particular $s \in \{-1,+1 \}^n$ that is the solution to 
	\begin{equation}
		\label{eq:optimization_problem}
		\min_{s \in \{-1,+1 \}^n} s^\top  J s + s^\top  b.
	\end{equation}
		Problem
		\eqref{eq:optimization_problem} is a widely encountered combinatorial optimization problem known as Ising problem.
		An equivalent formulation over binary variables, called quadratic unconstrained binary optimisation (QUBO), is found by variable substitution $x_i = (s_i +1 ) / 2$, yielding
		\begin{equation}
			\label{eq:qubo}
			\min_{x \in \{0, 1\}^n} x^\top  Q x + x^\top  c,
		\end{equation}
		for $Q = J/4 $ and $c = (b + \mathbb{1}^\top J) / 2$, where $\mathbb{1}^\top J$ sums the columns of $J$.
	
	The central idea of adiabatic quantum computing then merely follows from the free evolution of the state vector $\ket{\psi(t)}$ (Assumption \ref{assum:entangle} in Sec.~\ref{sec:fundamentals}) under the Schrödinger equation (Assumption \ref{assum:schroed} in Sec.~\ref{sec:fundamentals}) with an experimentally constructed Hamiltonian of a form 
	\begin{equation}
		\label{eq:hamiltonian_transition}
		H(t) = (1-f(t)) H_I + f(t) H_P 
	\end{equation}
	for a function $f: [0,T] \rightarrow [0,1]$ that slowly transitions between the two Hamiltonians in time $T$, \eg~$f(t) = t/T$ for linear schedule (other schedules such as piece-wise linear, quadratic or schedule with a pause are also possible). 
	The slow transition requirement and the speed of the transition are grounded on the adiabatic theorem of quantum mechanics. 
	%
	
	\begin{theorem}{Quantum Adiabatic Theorem}{adiabatic}
		If a system evolves under the Schrödinger equation~\eqref{eq:schroedinger_equation} starting from the $k$-th eigenvector $\ket{k(0)}$ ---the eigenvector to the $k$-th smallest eigenvalue $\lambda_k(0)$--- of $H_I$, then a sufficiently slow transition in Eq.~\eqref{eq:hamiltonian_transition} ensures that the state vector $\ket{\psi(t)}$ remains in the $k$-th eigenvector $\ket{k(t)}$ to $H(t)$ during the entire evolution until $H_P$, if the corresponding eigenvalue $\lambda_k(t)$ retains a sufficiently large gap to all other eigenvalues during the entire evolution. 
	\end{theorem}

	%
	In summary, one first has to 
	prepare the system in an eigenstate of $H_I$ corresponding to its smallest eigenvalue.  
	After that, one evolves $H_I$ to the problem Hamiltonian Eq.~\eqref{eq:problem_hamiltonian} sufficiently slowly via Eq.~\eqref{eq:hamiltonian_transition}.
	In the end, measuring the state should yield minimizers of Problem \eqref{eq:optimization_problem} (or at least be an excellent meta-heuristic for computing those) with high probability.
	The minimum time necessary for an adiabatic transition can be thought of as the runtime of the quantum algorithm.

	Conceptually, one can think of AQC as the minimisation of a function that is smoothly interpolated from a trivial function to the QUBO of interest, as depicted in Fig. \ref{fig:quantum_annealing}. 
	This interpolated function represents the energies of the system in the different basis states under the transition Hamiltonian $H(t)$. 
	As time $t$ increases, these energies gradually evolve from trivial values towards the energies defined by the problem Hamiltonian $H_P$.
	During this evolution, the adiabatic theorem guarantees that the system stays in the state corresponding to the minimum energy, which consequently corresponds to a gradual boost of the probability of measuring the ground state of the QUBO Hamiltonian $H_P$. 

	\begin{figure}[t]
		\centering
		\includegraphics[width=1.\linewidth]{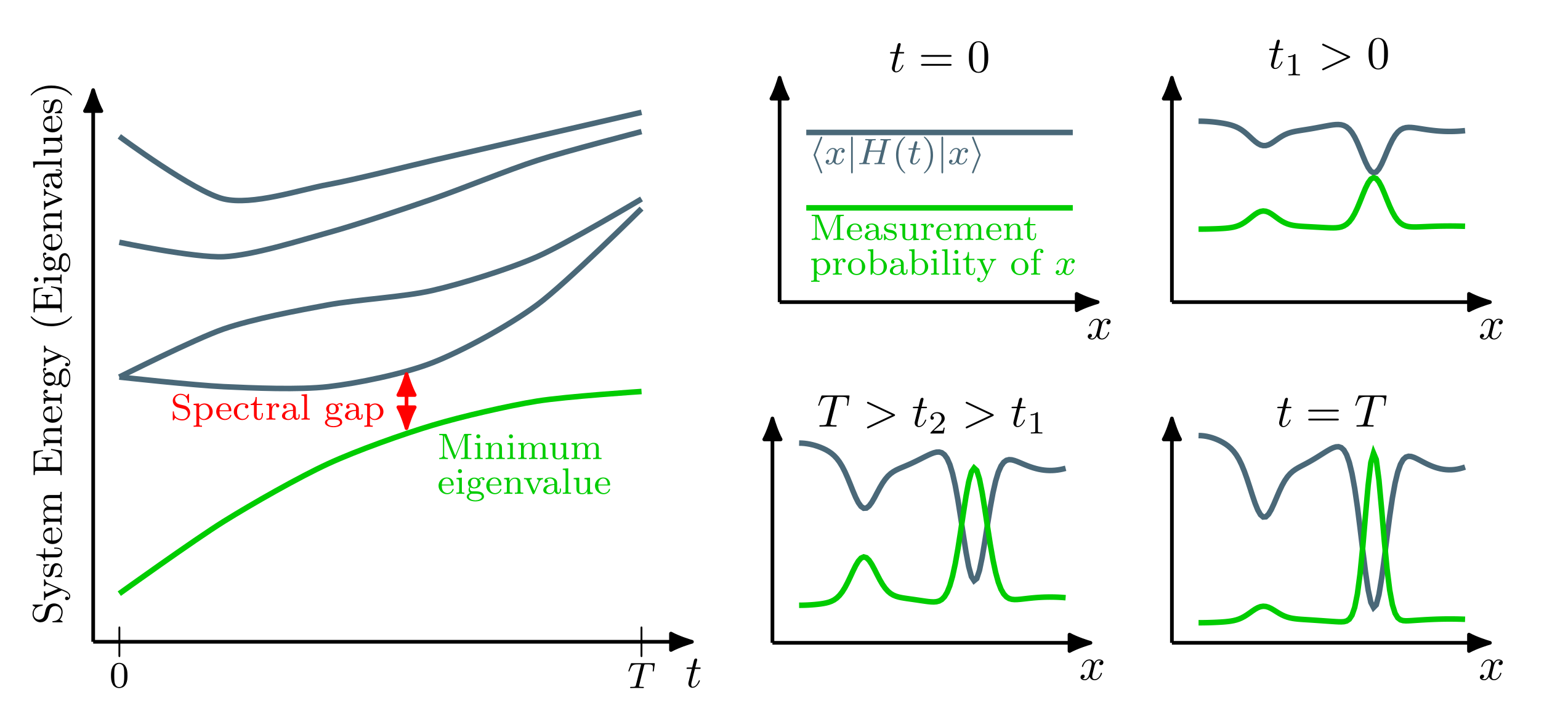}
		\caption{
			Conceptual illustration of AQC.
			(Right): Depiction of the spectral gap as the smallest difference between the lowest and the first excited energy states of the system over the time $t$.
			(Left): The AQC process of minimising a function $\braket{x|H(t)|x}$ that is interpolated from a trivial initialisation objective to a QUBO problem of interest. 
			The minimum eigenvalue on the left corresponds to a superimposed (eigen)state that determines the measurement probabilities of the basis states on the right. 
			For a strictly positive spectral gap, a sufficiently slow transition process maximally boosts the likelihood of measuring the ground state of the QUBO.  
		}
		\label{fig:quantum_annealing}
	\end{figure}

	Let us now prove that the adiabatic theorem is indeed merely a mathematical consequence of our assumptions. 
	We will sketch the proof here to give an understanding of what \textit{approximate eigenvector} and \textit{``sufficiently slow/large''} mean; a detailed version of the proof is provided in the appendix, Sec.~\ref{sec:adiabatic_quantum_theorem}. 
	\begin{proof}
		Note that all $H(t)$ are Hermitian matrices, such that for all $t$ there exists an orthonormal basis $\ket{k(t)}$ of eigenvectors to eigenvalues $\lambda_k(t)$, \ie
		\begin{equation}
			\label{eq:eigenvectors}
			H(t)\ket{k(t)} = \lambda_k(t) \ket{k(t)}.
		\end{equation}
		Thus, at each time we can represent a quantum state $\ket{\psi(t)}$ that is evolving under Eq.~\eqref{eq:schroedinger_equation} as 
		\begin{equation}
			\ket{\psi(t)} = \sum_k c_k(t)\ket{k(t)}
		\end{equation}
		for suitable coefficients $c_k(t)$. Subsequently, we insert this expansion into the Schrödinger equation and show that it results in a component-wise equation in the coefficients of the (time-dependent) eigenbasis,
		\begin{equation}
			\label{eq:componentwise}
			i \hslash (c_k'(t) + \langle k(t), k'(t) \rangle  c_k(t)) \approx \lambda_k(t) c_k(t),
		\end{equation}
		up to a sum over terms involving $\langle l(t), k'(t) \rangle$, $l\neq k$. Note that if Eq.~\eqref{eq:componentwise} held exactly, then 
		\begin{equation}
			\label{eq:estimate_ck}
			c_k(t) = c_k(0)\text{exp}\left(-\frac{i}{\hslash}\int_0^t  \lambda_k(s) - i \hslash\langle k(s), k'(s) \rangle~ds\right)
		\end{equation}
		would yield the claim. To obtain a bound on  $\langle l(t), k'(t) \rangle$ one can differentiate Eq.~\eqref{eq:eigenvectors} and take the product with $\bra{l(t)}$ to obtain
		\begin{equation}
			\langle l(t), k'(t) \rangle = \frac{\langle l(t),H'(t) k(t) \rangle}{\lambda_l(t) - \lambda_k(t)}, \label{eq:DividedByEigDiff}
		\end{equation}
		which is small if the change of the Hamiltonian $H'$ is small in comparison to all gaps of eigenvalues $\lambda_l(t) - \lambda_k(t)$.
		In this case, it follows from Eq.~\eqref{eq:estimate_ck} that if $c_k(0)$ has held for all but one index $k$, it will approximately hold for all $t \leq T$. 
		However, crossing eigenvalues, \ie~ $\lambda_l(t) = \lambda_k(t)$ for some time $t$, makes the claim invalid. 
	\end{proof}
	
	We can see from the above proof, an important quantity for the idea of AQC is the smallest difference of the second smallest to the smallest eigenvalue of $H(t)$ during the evolution in Eq.~\eqref{eq:hamiltonian_transition}, which is also called the \textit{spectral gap}, see Fig.~\ref{fig:quantum_annealing}.  
	Fortunately, for the following and common choice for the initial operator $H_{I}$:
	\begin{eqnarray}
		\label{eq:initial_hamiltonian}
		H_I &=& -\sum_{i}  \kappa \sigma_i^x, \,\\
		\sigma_i^x &=& \underbrace{I \otimes I \hdots I}_{(i-1)\text{-many times}}\otimes\,\sigma_x \otimes \underbrace{ I \otimes I \hdots \otimes I}_{(n-i)\text{-many times}}, \,\\ 
		\sigma_x &=& \begin{bmatrix} 0 & 1 \\ 1 & 0 \end{bmatrix},\text{and}\, \ \ I = \begin{bmatrix} 1 & 0 \\ 0 & 1 \end{bmatrix},
	\end{eqnarray}
	and under rather general conditions as stated in the next theorem, one can show that the spectral gap (and thus the denominator in Eq.~\eqref{eq:DividedByEigDiff}) does not become zero if the solution to the original problem \eqref{eq:optimization_problem} is unique. 

	\begin{theorem}{No crossing of eigenvalues}{nocrossing}
		For $H_P$ given by Eq.~\eqref{eq:problem_hamiltonian} and $H_{I}$ by Eq.~\eqref{eq:initial_hamiltonian} 
		the lowest and the second lowest eigenvalue of $H(t)$ from Eq.~\eqref{eq:hamiltonian_transition} do not cross during the time evolution, assuming that the ground state of $H_P$ is not degenerate.  
	\end{theorem}
	
	\begin{proof}
		The basic idea is to show that the smallest eigenvalue of $H(t)$ is simple using the Perron-Frobenius theorem. A full proof is given in the appendix, Sec.~\ref{sec:no_crossing_eigenvalues}. 
	\end{proof}

	Note that the constant $\kappa$ in the definition of $H_I$ allows for scaling all eigenvalues and has an effect on the speed $H'(t)$ of the transition in Eq.~\eqref{eq:hamiltonian_transition}. 

	While a zero spectral gap can be avoided for all problems with unique solutions, the question how the total time $T$ necessary for an \textit{adiabatic transition}, \ie~ one where the system remains in its ground state, scales with the dimension of the overall problem, is crucial for the complexity of the presented AQC framework. 
	In particular, if the spectral gap becomes exponentially small, an exponential increase in the transition time $T$ will be needed, indicating a very difficult problem even for a quantum computer. 
	Since QUBOs are NP-hard in general, it is believed that even quantum computers will require an exponentially increasing runtime for solving them in general. 
	For an interesting article in this context, we refer to Aaronson~\cite{aaronson2005guest}. 

	For a linear transition between $H_I$ and $H_P$ in Eq.~\eqref{eq:hamiltonian_transition} it has been proven that certain instances indeed require an exponentially increasing runtime~\cite{van2001powerful}. On the topic of how to change the annealing path and the Hamiltonians to avoid exponentially small gaps, if possible, a lot of research has been done~\cite{albash2018adiabatic}. 
	Another interesting research area is how the success probability depends on the degree of entanglement of the state $\ket{\psi(t)}$ for $t\in( 0,T)$. 
	In~\cite{batle2016multipartite} it was shown that quantum systems that give a speedup via adiabatic quantum computing can have a low amount of multipartite entanglement. 
	Therefore, the relation between those quantities is complex and one cannot claim that entanglement fully explains the performance of AQC algorithms. 
	Another research direction is to deliberately deviate from a sufficiently slow transition of Hamiltonians to yield Bayesian estimates rather than optimal solutions, see \eg~\cite{mccormick2022multiple}. 
	As an exhaustive discussion on adiabatic quantum computing itself goes beyond the scope of this work, we refer the interested reader to the overview~\cite{Hauke_2020} and rather focus on some practical aspects relevant to quantum computer vision. 

	\parahead{Quantum annealers} 
	Quantum annealers are modern experimental realisations of AQC devices.
	They take advantage of the theory of AQC from Sec.~\ref{ssec:adiabatic} and realise the transition from $H_0$ to a given $H_P$ experimentally. 
	The main distinction between quantum annealing and AQC is that quantum annealing also refers to experiments that are not fully adiabatic ~\cite{AdiabaticQuantumComputingandQuantumAnnealing}. 
	Hence, quantum annealers return ground states with probabilities ${\ll}1.0$ for difficult problems. 
	To increase the probability of finding the ground state at least once, annealing is repeated multiple times. 
	Current experimental realisations of quantum annealers can sample solutions to QUBO problems. 
	\parahead{QA \textit{vs} classical methods} 
	A comparison between quantum annealing and classical optimisation algorithms is also insightful. 
	QUBOs can be solved by several classical algorithms with the help of variants of branch and bound and various heuristics~\cite{Kochenberger2014}. 
	Some of the early methods solve max-cut~\cite{Barahona1989}, \ie~reduced QUBOs with $c$ as a zero-vector, 
	or QUBOs with positive definite matrices~\cite{Pardalos1990}. 
	Other methods rely on semidefinite relaxations and cutting planes~\cite{goemans1995improved,poljak1995solving,alizadeh1995interior,HelmbergRendl1998}. 
	The Goemans and Williamson algorithm~\cite{goemans1995improved} solves the max-cut problem with a provable approximation ratio of $0.87$, which was later shown to be optimal~\cite{karloff1996good}.
	Later techniques adapt tabu search algorithms~\cite{Glover1998, Palubeckis2006}, genetic local search~\cite{Katayama2000}, memetic policies~\cite{MerzKatayama2005} and simulated annealing~\cite{Alkhamis1998, KatayamaNarihisa2001}. 
	Exact (predominantly earlier) methods remain limited to several hundred binary variables, while later techniques report successful results with the number of binary variables of the order $10^4$ though do not provide global optimality guarantees. 

	\parahead{Simulated annealing} 
	Simulated annealing~\cite{Kirkpatrick1983} shares operational analogies with quantum annealing. 
	The simulated annealing algorithm works iteratively. 
	At the beginning of a step, a random neighbour from the current iterate is chosen. 
	If the energy of the randomly chosen neighbour is better than that of the current iterate, it gets accepted and is the new iterate. 
	If the energy gets worse, the randomly chosen neighbour is only accepted with a probability 
	\begin{equation}
		P = \min \left \{1,\exp ((E(x^{(i-1)})- E(x^{(i)}_{\text{candidate}}) )/T_i ) \right \}.
	\end{equation}
	The sequence $T_i$ is called temperature schedule and is monotonically decreasing with every iteration step $i$. 
	As a result, it becomes progressively more unlikely that worse solutions are accepted. 
	At some point, the algorithm will then end up in a local optimum from which it cannot escape. 
	An early research question for AQC was for which energy potentials it is beneficial compared to simulated annealing. 
	Farhi~et al.~\cite{farhi2002quantum} present an energy landscape with a spike, where simulated annealing takes exponential time to get to the global optimum, while AQC finds the global optimum in polynomial time. 

	\parahead{Simulated \textit{quantum} annealing} 
	Another question one could ask is how quantum algorithms compare to classical algorithms that are \emph{inspired} from quantum annealing. 
	This question is explored, \eg~in Crosson et al.~\cite{crosson2016simulated}. 
	They look at an algorithm where low-energy states are estimated using classical Markov chain Monte Carlo methods.
	In contrast to simulated annealing, the authors considered not only one, but multiple state configurations of the quantum system at a time to describe the system. 
	This multi-state configuration then allows simulating quantum effects like tunnelling, superposition and entanglement. 
	In Sec.~\ref{app:montecarlosampling} of the Appendix, we provide a short summary of this quantum-inspired simulation.

	\subsection{Connections between Gate-based Quantum Computing and Adiabatic Quantum Computing}
    \label{ssec:ConnectionsParadigms}
	This section discusses how the AQC and gate-based quantum computational paradigms relate. 
	It is believed and can be mathematically shown that 
	the gate-based model of quantum computing is polynomially equivalent to the adiabatic model (and vice versa) in terms of computational complexity~\cite{albash2018adiabatic,nielsen2002quantum,aharonov2008adiabatic}:
	\begin{theorem}{Equivalence of gate-based and AQC}{equiv_aqc_gqc}
		The adiabatic model of quantum computation is polynomially equivalent to the
		gate-based model of quantum computation.
	\end{theorem}
	In other words, given an arbitrary quantum circuit, it is possible to design a Hamiltonian whose ground
	state equals the output of the circuit starting from an easily prepared initial state.
	And, in the other direction, that the evolution governed by the Schrödinger equation under a given AQC-Hamiltonian can be implemented using a quantum circuit.

	\begin{proof}
		In Sec.~\ref{app:equivalence_gb_aqc} of the appendix, we sketch the proof of the theorem and refer the reader to~\cite{van2001powerful,aharonov2008adiabatic} for a full discussion.
	\end{proof}

	It is noteworthy that while gate-based and adiabatic computing models are mathematically equivalent, current hardware does not support seamless switching between the two. 
	Consequently, from a computational viewpoint, it is valuable to investigate which prototypical CV sub-problems are particularly suited to each model and to compare their respective solution strategies.

	\parahead{Quantum approximate optimisation algorithm} 
	A variational method linking adiabatic and gate-based quantum computing is the \emph{Quantum Approximate Optimization Algorithm (QAOA)}~\cite{Farhi2014}. 
	QAOA seeks to find the ground state of the Ising Hamiltonian, as defined in Eq.~\eqref{eq:problem_hamiltonian}, by emulating the adiabatic evolution described in Eq.~\eqref{eq:hamiltonian_transition}. 
	The idea is to discretise the time-dependent Hamiltonian of AQC and to solve, at least approximately, the Schrödinger equation for the resulting discretised Hamiltonian.
	This evolution, see also in Sec.~\ref{app:equivalence_gb_aqc}, Eq.~\eqref{eq:aqc_to_circuit}, can be approximately simulated through a sequence of gates. 
	These gates alternate between evolving the system under the cost Hamiltonian $H_P$ and the initial (mixing) Hamiltonian $H_I$ for short intervals $\Delta t$. 
	In the QAOA algorithm, the time step $\Delta t$—alternatively denoted as $\gamma_t$ or $\beta_t$, depending on whether it governs the application of the Hamiltonians $H_P$ or $H_I$ at time $t$—is treated as an optimisation parameter. It is adjusted by a classical outer loop, yielding the variational principle underpinning QAOA.
	This results in $2p$ parameters for a depth-$p$ discretised evolution. 
	An overview of the QAOA circuit is provided in Fig.~\ref{fig:QAOA}.

	\begin{figure}[t]
		\centering
		\includegraphics[width=0.95\linewidth]{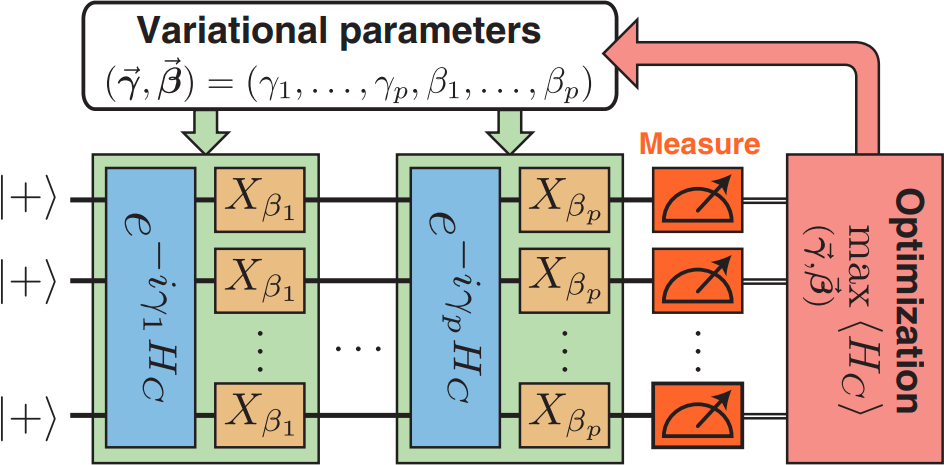}
		\caption{Diagram of QAOA. The Hamiltonian $H_C$ in the figure corresponds to $H_P$ in this survey. Image source:~\cite[reproduced under the CC BY 4.0 license]{PhysRevX.10.021067}.
	} 
	\label{fig:QAOA} 
\end{figure}

\subsection{Computational Advantages of Quantum Computing} 
\label{eq:advantage}
The computational advantages of QC algorithms can be analysed with the help of \emph{Quantum Complexity Theory (QCT)} which is the counterpart of classical complexity theory and which studies the hardness of problems solved on quantum computers~\cite{BernsteinVazirani1997, Watrous2008arXiv}. 
Even though an in-depth coverage of QCT in this survey is out of scope,
we summarise the core aspects useful for researchers in the QeCV field. 
Recall that major advances in computer vision are driven by the accuracy on benchmark datasets and not by reduced method complexity. 
As a consequence, computer vision and machine learning scientists do not often use computational complexity theory and analyse the target problems in its terms. 
To summarise, we find the following aspects related to QCT important for QeCV researchers. 
First, it was noticed that finding quantum algorithms---that improve upon the computational complexity of the best classical methods addressing the same problem---is notoriously challenging~\cite{Shor2003}. 
Nonetheless, different multiplicative terms in the same complexity class can make a substantial difference in practical applications (\textit{\eg~} faster neural network training or faster optimisation of an objective function) and this concerns both the gate-based and the adiabatic models. 
Second, not all properties of quantum methods can be quantified in terms of QCT. 
For instance, theoretical and experimentally observed advantages such as fewer required parameters in quantum neural networks (gate-based model) and high probability to measure globally-optimal solutions for certain problem types (adiabatic model) hold additionally for quantum formulations. 
Third, there are many open research questions on how existing computer vision problems can benefit from quantum formulations. 
The community is just at the beginning of understanding which problems can be solved with QPUs, in the sense of computational tools (similar to graphics processing units). 
Identifying problems that can be solved with the help of quantum formulations better in a practical sense (than relying on classical formulations and irrespective of theoretical complexity gains) is the ultimate goal of QeCV. 

\begin{figure*} [t]
	\centering 
	\subfloat[IBM Heron (2024).]{\includegraphics[width = 0.26\textwidth ]{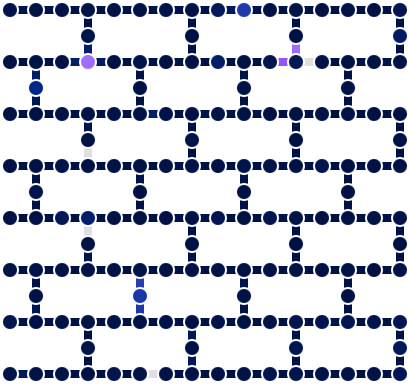}}
	\hspace{1.cm}
	\subfloat[IonQ Aria trapped ion (2022).]{\includegraphics[width = 0.265\textwidth, trim={0cm 0cm 0cm 0cm}, clip]{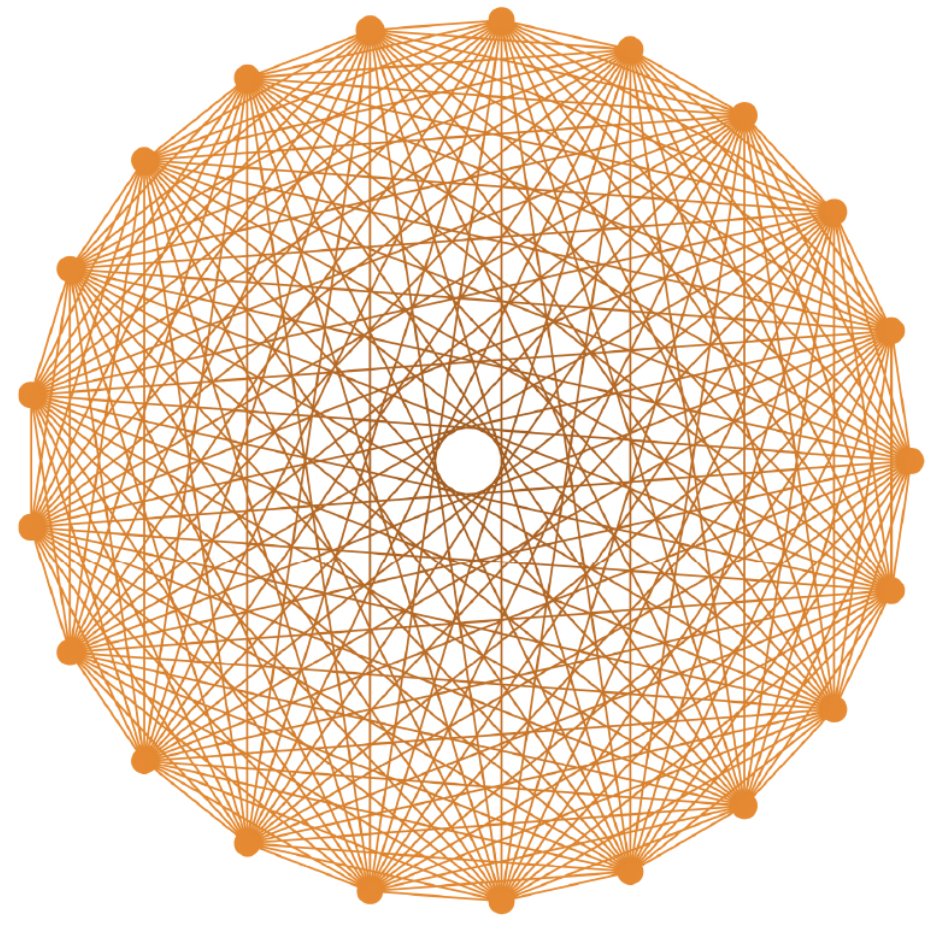}}
	\subfloat[Google Willow (2024).]{\includegraphics[width=.4\textwidth, trim={-1.5cm 0cm -1cm 0cm}, clip]{ 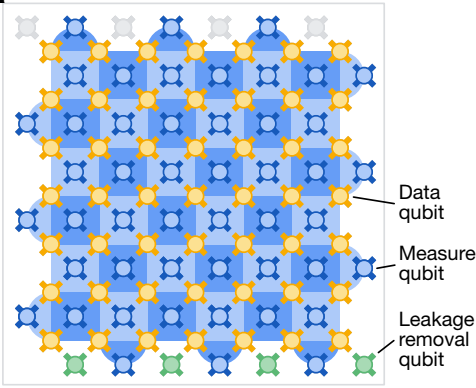}}
	\hspace{-1.cm}
	\caption{Examples of architectures of gate-model quantum computers.
		Launch dates are indicated in brakets.
		(a) A IBM Quantum Heron processor with 156 qubits (nodes) connected in the heavy hex topology, source:~\cite[reprint courtesy of IBM corporation \textcopyright 2025]{ibm_eagle}. 
		(b) IonQ Aria processor with 21 fully-connected qubits (nodes), source:~\cite[reprint courtesy of IonQ \textcopyright 2025]{ionq_aria}.
		(c) Google Willow processor with 105 qubits of degree $\sim4$ each, comprising data qubits (gold), measure qubits (blue) ---used to detect errors in the neighboring data qubits, and additional leakage removal qubits (green) ---used to detect and correct leakage errors, source:~\cite[reproduced under the CC BY NC ND 4.0 license]{google_willow}
	} 
	\label{fig:gate_architectures} 
\end{figure*} 

The most prominent example of a quantum algorithm that outperforms classical counterparts is probably Shors algorithm for factoring primes~\cite{Shor1997}. This algorithm has polynomial runtime for a problem class that is widely believed to be outside of $\mathcal{P}$. But also, cases where there is polynomial improvement are very interesting. In the case of the Grover's algorithm there is a quadratic improvement for searching in an unstructured database~\cite{grover1996fast}. 
That the same speed up is achievable in the model of adiabatic quantum computing was shown in~\cite{roland2002quantum,van2001powerful}. 
This is especially remarkable, since in the general case time estimates for adiabatic quantum computing are often hard to find, since the spectral gap is hard to compute, see Sec.~\ref{ssec:adiabatic} for a discussion on the spectral gap.

Although it is widely believed that AQC cannot solve NP-complete problems in polynomial time, it may still offer acceleration in solving QUBO problems of the form \eqref{eq:qubo}. 
In practice, an AQC samples solutions for QUBOs and returns multiple low-energy states. 
Depending on the problem structure, one might even achieve global optima faster than classically foreseeable. 
However, this becomes increasingly challenging, as the spectral gap typically decreases relatively fast as the problem size increases. 
As a result, for larger problems, the probability of finding an optimal solution can become very low. 
Improvements in quantum hardware are expected to increase the problem sizes that obtain substantial non-zero probabilities of measuring globally optimal solutions in a single sample. 
\subsection{Design and Architecture of Quantum Computers}\label{ssec:DesignArchitecture} 
Since their conceptual proposal in 1980s by Feynman~\cite{feynman2018simulating} and Deutsch~\cite{deutsch1985quantum,deutsch1989quantum}, substantial progress in experimental hardware realisations of quantum computers has been demonstrated.
A first working two-qubit quantum computer was proposed by IBM in $1997$ following the requirements of DiVincenzo~\cite{divincenzo1997topics}.
D-Wave released the first adiabatic quantum computer of $128$ qubits in $2011$~\cite{johnson2011quantum}.
A major obstacle in quantum hardware manufacturing is the occurrence of errors, often caused by interactions with the environment. 
Completely isolating a quantum system is practically difficult, making it challenging to preserve quantum coherence, \ie~the time period during which qubits retain their quantum properties for performing calculations. 
As a result, qubits are noisy, uncorrected and prone to imperfect operations. 
Due to the no-cloning theorem, researchers recognised that quantum error correction had to be fundamentally different from classical methods, which typically rely on redundancy through copying information multiple times. 
A groundbreaking approach to quantum error correction was introduced by Shor in 1995, a year after his pioneering work on factoring integers~\cite{georgescu202025,shor1995scheme,calderbank1996good}. 
However, practical implementations of quantum error correction remain highly challenging, primarily because they require a significant number of additional qubits to encode and protect information---often far beyond the number needed for basic computation.  
Thus, large-scale fault-tolerant quantum computing remains an ongoing challenge rather than a solved problem~\cite{cai2023quantum}.
Early quantum machines—both current and near-future—are categorised as \emph{Noisy Intermediate-Scale Quantum (NISQ)} devices~\cite{preskill2018quantum}, inspiring researchers to explore how useful such systems could be, even without full error correction. 
In the late 2010s and early 2020s, the first claims of computational supremacy or advantage by quantum computers emerged. 
In particular, Google presented an experiment~\cite{arute2019quantum} suggesting that a classical computer would require an unreasonable amount of time to sample the outputs. 
However, subsequent research demonstrated that the sampling task could indeed be solved using classical computing~\cite{SupremacyBeaten}. 
On the quantum annealing side, D-Wave identified a problem where their annealer seemed to outperform all known classical algorithms in 2024~\cite{king2024computational} and in 2025~\cite{king2025beyond}.  

\begin{figure*} [t]
\centering 
\subfloat[2000Q (2017).]{\includegraphics[width = 0.33\textwidth, trim={0cm 1cm 19cm 0cm}, clip]{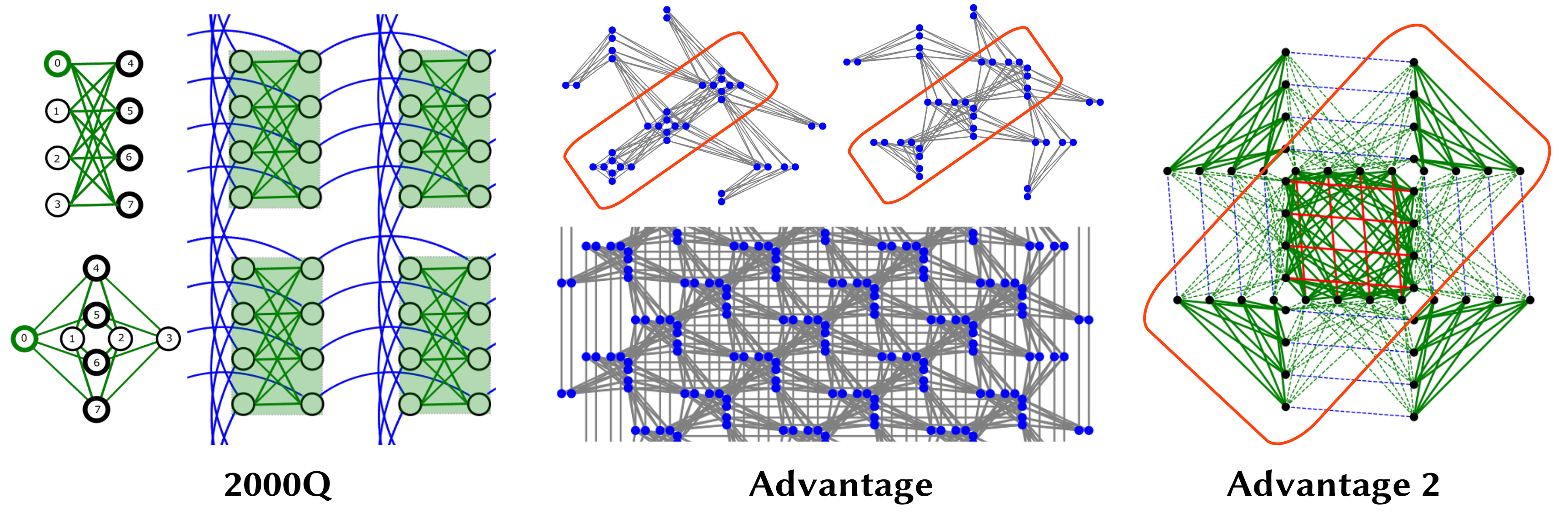}}
\subfloat[Advantage (2020).]{\includegraphics[width = 0.35\textwidth, trim={10cm 1cm 8.5cm 0cm}, clip]{figures/figure_D-Wave_Architectures.png}}
\subfloat[Advantage 2 (2025).]{\includegraphics[width = 0.29\textwidth, trim={20.6cm 1cm 0cm 0cm}, clip]{figures/figure_D-Wave_Architectures.png}}
\caption{Different architectures of D-Wave quantum annealers and their qubit connectivity patterns.
	Launch dates are indicated in brakets.
	(a) 2000Q with 2048 qubits (nodes) of degree 6 each ---4 internal (green) and 2 external (blue) couplers---arranged in the Chimera topology;
	(b) Advantage with 5640 qubits (dots) of degree 15 each ---12 internal (gray) and 3 external couplers and additional odd ones--- arranged in Pegasus topology;
	And 
	(c) Advantage 2 with 4400 qubits(dots) of degree 20 each ---16 internal (green) and 4 external (blue) couplers and additional odd (red) ones--- arranged in the Zephyr topology.
    Source:~\cite[reproduced with permission]{dwave_architechture}.
} 
\label{fig:D-Wave_architectures} 
\end{figure*} 

\subsubsection{Gate Quantum Computers Architectures}
Many prominent quantum systems are realization of gate-based architectures, see Fig.~\ref{fig:gate_architectures} for the graph connectivity of several architectures. In the following we give some information about the competing architectures:
\begin{enumerate}
\item IBM: IBM's flagship System One quantum computers are quantum processors built with superconducting qubits.
IBM managed to produce a large number of systems of different sizes accessible over the cloud.
These systems come with several types of connectivity graphs and assigning physical qubits to logical qubits often requires procedure called transpilation. 
In December 2023, IBM announced a system with 1,121 qubits \cite{IBMroadmap2025}.

\item Google: Google's try to achieve quantum supremacy was executed with the Sycamore processor a superconducting quantum computer with 53 qubits. 
Reference~\cite{bravyi2022future}
provides further descriptions of technological challenges for superconducting qubits.

\item IonQ: IonQ's systems employ trapped ions as qubits, demonstrating exceptionally low error rates.
An advantage of IonQ trapped ion architecture are a fully connected architecture, however, the architecture can not yet perform mid-circuit measurements. The largest Aria architecture has $21$ qubits, while the IonQ Forte has 36 qubits.  
Many aspects about trapped ion QCers are presented in detail in the review article~\cite{bruzewicz2019trapped}.

\item Rigetti: Rigetti's Ankaa-3 quantum computer offers cloud-based access to computers with up to $84$ superconducting qubits~\cite{Rigetti}.
\item Honeywell: Honeywell's trapped-ion quantum computer emphasizes qubit quality and stability in qubit manipulation~\cite{Honeywell}.
\item PsiQuantum, Xanadu: Both work in photonic quantum computing. A review about experiments with  linear optical elements  can be found in~\cite{zoller2005quantum}. 
\end{enumerate}

In addition to real quantum hardware, simulators of more than $30$ qubits are available for testing quantum circuits over the cloud from several providers such as Amazon Web Services (AWS) and Microsoft Azure, IBM and Google Cloud.

Also note that while the number of qubits is often cited as a benchmark, it provides limited insight in isolation. 
Also important are factors such as qubit connectivity, gate fidelity, coherence time and error correction capabilities. 
All these characteristics considered together offer a more meaningful picture of quantum device performance. 
\subsubsection{Quantum Annealing Architectures}
\label{ssec:DWave_Architecture} 

Quantum annealing is possible with a variety of qubit technologies.
The D-Wave~\cite{Dwave} company provides the most advanced annealers, which we describe in the following.

To obtain the qubits in the D-Wave annealer, Niobium loops are cooled to temperatures below $17$mK, so that they become superconducting. This allows the current to flow in two directions, which physically realises a qubit. 
For example, the state $\ket{0}$ corresponds to current flowing clockwise in the loop. 
The coupling between individual qubits is also done via Niobium loops. 
The cooled niobium loops have to be carefully shielded from interaction with the environment.
For the D-Wave machines, the coherence time is of the order of nanoseconds. 
This is shorter than the actual default annealing time of $20\mu s$~\cite{Dwave}.
How this affects the solution quality is still under research~\cite{albash2015decoherence}. 

Since 2016 at the latest, there is experimental evidence that D-Wave machines leverage quantum mechanical effects to perform computations~\cite{denchev2016computational, Gibney2017DWaveUH, Mandra2017, King2021}. 
Nowadays there seems to be a consensus that quantum mechanical effects are present in the D-Wave annealers.
For instance, Denchev \textit{et al.}~\cite{denchev2016computational} investigated the computational utility of having quantum mechanical tunneling in the dynamic of the system. 
They demonstrated that for some energy landscapes, quantum annealing can be ${\sim}10^8$ times faster than simulated annealing. 
Currently, three generations of D-Wave quantum annealers can be accessed and used remotely; Those are 2000Q (Chimera), Advantage (Pegasus) and Advantage2 (Zephyr). 
Two 2000Q QPUs (located in North America) and three Advantage QPUs (two in North America and one in Europe) are available for remote access. 
A QPU of the Zephyr topology was completed in late 2023 and released in $2025$, with $4400$ qubits. 
All three generations of D-Wave annealers differ in their architecture. 
The differences include the compositions of the cells, the total number of unit cells and qubits, the nominal lengths of the qubits, the number of external couplers per qubit and the available types of couplers. 
The qubits are organised in unit cells. 
Fig.~\ref{fig:D-Wave_architectures} summarises the current D-Wave architectures. 
The architectures determine how QUBOs are embedded into hardware with so-called \textit{Minor embedding} algorithms (such as Cai \textit{et al.}~\cite{Cai2014}) that can calculate a mapping from the logical problem graph to the physical hardware graph.
In this way, multiple physical qubits are enforced to behave as one logical qubit with higher connectivity. 
One often calls this process building a chain. 
The physics behind how this works is that one adds a term to the Hamiltonian that penalises unequal measurement results in the computational basis between the qubits in the chain. 
The strength of how much the physical qubits are enforced to behave as one logical qubit is a free parameter called the chain strength. 
Finding optimal embeddings influences not only the number of physical qubits needed for a problem and the chain length but also the probability of measuring an optimal solution in a single read (anneal). 
Advantage can embed and sample substantially larger QUBO problems compared to 2000Q: while the latter has approximately 2040 working qubits—each connected to up to six others—Advantage features over twice as many (around 5500), with connectivity to up to $15$ other qubits. 
Advantage2 (released in 2025) with 4400 qubits has fewer qubits than Advantage but offers improved connectivity with a degree of 20. 
Note that some qubits are nonoperational, which minor embedding algorithms must be taken into account (each QPU can be queried for the list of such qubits). 
These defective qubits also play a large role in the computational complexity of finding the embedding~\cite{lobe2024minor}. 
It is worth mentioning that the default annealing schedule in D-Wave machines is not given by a simple linear combination as introduced in Eq.~\eqref{eq:hamiltonian_transition};  see~Ref,\cite{shedule}. 
Moreover, D-Wave annealers are programmable machines allowing the user to set up solver parameters in addition to the problem-specific couplers and biases. 
Among others, the user can define the chain strength, the total annealing time, the annealing schedule and the number of shots, \ie~ the number of repetitions of the annealing process. 

\section{Quantum-enhanced Computer Vision}\label{sec:methods_domains} 
After having reviewed the fundamentals of quantum computing, we now delve deeper into the use of quantum computers in enhancing computer vision techniques. 
As the computer vision community is driven by experimental results, we emphasise methods evaluated on real quantum hardware. 
Quantum hardware has been accessible to a wider community of researchers for several years—often through free-tier accounts—and opportunities to access quantum computers are expected to increase further in the near to mid-term.
Hence, it is expected that the number of studies demonstrating results on real quantum hardware will steadily increase. 
In the following, we begin with a discussion on problem mapping methodologies, specifically modelling and data encoding techniques for AQC and gate-based QC in Sec.~\ref{sec:ProblMapMeth}. 
Subsequently, we review the published literature on QeCV methods utilizing AQC in Sec.~\ref{ssec:MethodsAQC}, followed by approaches based on the gate-based paradigm in Sec.~\ref{ssec:MethodsGateBased}. 
Finally, Sec.~\ref{sec:others_and_discussion} presents high-level insights and observations derived from the reviewed literature.

\subsection{Problem Mapping Methodologies} 
\label{sec:ProblMapMeth}

\subsubsection{Mapping to AQC}
\label{sec:ProblMapMethAnnealers}
As discussed in Sec.~\ref{ssec:adiabatic}, quantum annealing is a heuristic designed to solve 
(or sample solutions to) problems of the form \eqref{eq:optimization_problem} and \eqref{eq:qubo}. 
Hence, every QeCV approach leveraging modern quantum annealers has to \textbf{provide a QUBO objective} that is then sampled on a quantum annealer.  
The weights of the QUBO form (qubit couplers and biases) have to be calculated on classical hardware for a given problem and data. 
The resulting weights define the \textbf{logical problem} and the logical problem graph, in which each binary variable corresponds to a logical qubit (\ie~ a qubit on a hypothetical quantum machine). 
Since the connectivity patterns in real quantum annealers are restricted in the sense that not all qubits can be directly connected to all other qubits, a form of transpiling\footnote{Generally, transpiling in quantum computing refers to finding a low-level and hardware-oriented equivalent to a quantum circuit or a high-level problem for quantum hardware.} has to be performed known as a \textbf{minor embedding}. 
In graph theory, an undirected graph $H$ is called a minor of the graph $G$ if $H$ can be formed from $G$ by deleting edges, vertices and by contracting edges. 
Hence, minor embedding is finding a graph minor for a logical problem graph to the hardware graph of a quantum annealer. 
Once a minor embedding is found using algorithms such as Cai \textit{et al.}~\cite{Cai2014}, quantum annealing can be repeated multiple times (on the order of hundreds or thousands) to sample low-energy bitstrings (with the hope that the optimal solution would be among them). 
This step is also known as \textbf{solution sampling}. 
As mentioned in Sec.~\ref{ssec:DWave_Architecture}, although it may sound substantial, each sampling run—or \textit{annealing shot} on a D-Wave machine typically lasts only $20\mu s$. 
Therefore, even when performing hundreds to thousands of shots, the total time overhead remains relatively minimal.
Once multiple bitstrings are sampled, one or several of them can be selected for \textbf{unembedding}, in which a single binary variable is obtained/measured for each logical qubit. 
\textbf{Bitstring selection} can use different criteria, with the most widely used one being based on the lowest energies associated with the samples. 
Once one or several bitstrings are selected, they can be interpreted as solutions to the original problem. 
The \textbf{solution interpretation} is, hence, directly derived from solution encoding. 
This sequence is illustrated in Fig.~\ref{fig:qcv}.
In summary, every QeCV algorithm would have the following six steps: (i) QUBO preparation, (ii) minor embedding (at least calculated once if the structure of the weight matrix does not change), (iii) sampling by quantum annealing, (iv) unembedding, (v) bitstring selection and (vi) solution interpretation. 
Fig.~\ref{fig:qsync} visualises the six steps of QeCV approaches for quantum annealers in the context of the QSync approach~\cite{QuantumSync2021}. 

Encoding solutions as binary bitstrings and finding QUBO forms accounting for the problem, data and additional constraints (regularisers) constitute an active area of QeCV research, specially driven by the availability of real quantum annealers capable of solving problem sizes relevant to CV applications.
Many QeCV approaches define a single QUBO objective and sample solutions to it once on a quantum annealer. 
This method type with a single QUBO preparation step is known as \textit{one-sweep} approaches. 
QeCV methods can also be \textbf{iterative}, they can repeat the above-mentioned six steps, \eg~until solution convergence or aggregation. 
Fig.~\ref{fig:qcvhybrid} visualises this principle: Once the initially prepared QUBO is sampled, a new QUBO is prepared based on the obtained solution and so on until the problem-specific convergence criterion is met. 
Most of the time, the iterations will interpret the result using classical computation performed on CPU/GPU. 
We next describe common methodologies for finding and deriving QUBO forms in QeCV. 

\begin{figure}[t]
\begin{center}
	\includegraphics[width=\linewidth]{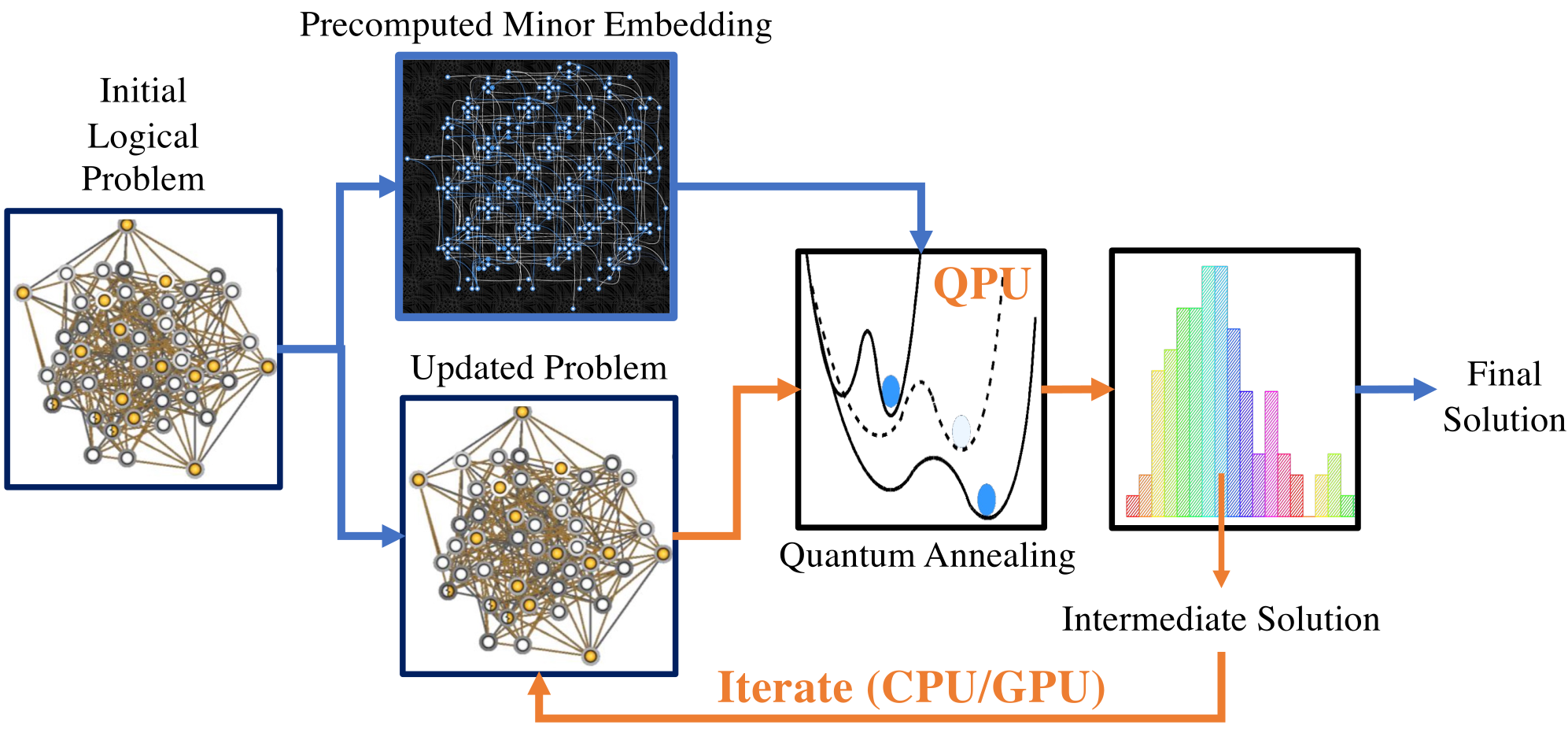}
\end{center}
\caption{
	Schematic overview of iterative quantum computations with a quantum annealer: After the initial logical problem has been defined, it is minor-embedded on an annealer. 
	Once the first annealing is finished, the logical problem is updated, minor-embedded and annealed again. 
	These three steps are repeated until convergence following a problem-specific criterion. 
} 
\label{fig:qcvhybrid}
\end{figure}

\parahead{QUBO and Ising problems}
We note that one can easily convert any optimization problem over $n$ values $v_1, v_2 \in \mathbb{R}$ and w.l.o.g. $v_2>v_1$ into the Ising form \eqref{eq:optimization_problem} with $ s\in \{-1,1\}^n$ as variables. 
For this, we first introduce the variable transform from $s$ to the variables $v\in \{ v_1, v_2 \}^n$ as
\begin{equation}\label{eq:variable_transform} 
v = \frac{1}{2} ((v_2-v_1) s + (v_2+v_1 ) \mathbb{1}),
\end{equation}
where $\mathbb{1}$ is the length-$n$ vector filled with ones.

Using Eq.~\eqref{eq:variable_transform}, the conversion to the Ising form is given by
{\small
\begin{align}
	 \argmin_{v \in \{ v_1, v_2 \}^n  }  &v^\top W v + w^\top v \\
	=  \argmin_{s \in \{ -1, 1 \}^n  }  
	&\frac{1}{2} ((v_2-v_1) s + (v_2+v_1 ) \mathbb{1} )^\top  W \cdot \\ \nonumber 
	& \cdot \frac{1}{2} ((v_2-v_1) s + (v_2+v_1 ) \mathbb{1} ) \nonumber\\
	&+ w^\top \frac{1}{2} ((v_2-v_1) s + (v_2+v_1 ) \mathbb{1} ) \nonumber \\
	=   \argmin_{s \in \{ -1, 1 \}^n  }  
	&\frac{(v_2-v_1)^2}{4}  s  ^\top  W  s \\ 
	+ \frac{1}{2} (v_2+v_1 &)(v_2-v_1) \mathbb{1}^\top  W s  +\frac{v_2-v_1}{2} w^\top   +\textnormal{const.} 
\end{align}
}
\hspace{-10pt} and we see that in the last optimization problem, only quadratic and linear terms occur. 

In the above expression (and throughout the survey), we always assume quadratic cost matrices to be symmetric. This is because we can freely distribute matrix entries between a matrix element $W_{i,j}$ and $W_{j,i}$. Simply writing down the sum explicitly reveals: 
\begin{align}
v^\top Wv&= \sum_ i \sum_j v_iv_j W_{i,j } \\ 
& = \sum_{(i,j) \in\{ (i,j ) \in \{ 1,...,n \}^2 | i < j  \}  } (W_{i,j} + W_{j,i})v_i v_j \nonumber  \\
& + \sum_i W_{i,i} v_i, 
\end{align}
\ie~the optimization problem only depends on the sums $(W_{i,j} + W_{j,i})$.

Also note that each formulation has a particular invariance, \eg~in problem \eqref{eq:optimization_problem}, the diagonal of $J$ can be changed by arbitrary constants due to $s_i^2=1$. 
In the frequently used QUBO form \eqref{eq:qubo}, one can place all diagonal terms of $Q$ into $c$ since $(x_i)^2 = x_i$ for $x_i \in \{0,1\}$. 

\parahead{Handling constraints} 
The usual QUBO problem \eqref{eq:qubo} is unconstrained by definition. 
However, to solve most vision problems, we need to factor constraints in, yielding the more general \emph{Quadratic Binary Optimisation} (QBO).
We next discuss the integration of equality and inequality constraints into QUBOs through Lagrange multipliers and ancilla binary variables. 

\parahead{Linear equality constraints} 
The standard way to incorporate equality constraints is by adding a so-called \emph{penalty} or \emph{rectification} term. 
Suppose one wants to optimize $x^\top  Qx$ over binary variables $x\in \{0,1\}^n$ with the restriction $Ax= b$, were $A\in \mathbb{R}^{m\times n}$ and $b\in \mathbb{R}^m$:
\begin{equation}
\label{eq:constrained_to_unconstrained_0}
\argmin_{ \{ x\in \{0,1\}^n| Ax= b \}} x^\top  Qx.
\end{equation}
Then, for a sufficiently large penalty parameter $\lambda \in \mathbb{R}$ , the constrained optimization problem \eqref{eq:constrained_to_unconstrained_0} can be converted into a \emph{softly} regularised QUBO in the following way: 
\begin{equation}
\label{eq:constrained_to_unconstrained_1}
\argmin_{  x\in \{0,1\}^n } \quad \underbrace{x^\top  Qx}_{\text{objective term}} 
+ \ \underbrace{ \lambda ||Ax-b ||^2}_{\text{penalty term}}.
\end{equation}
Note that the term $||Ax-b||^2$ has only quadratic dependence on the binary variables $x$. 
While from a purely theoretical perspective, the $\lambda$ parameter can be arbitrarily large, in practice, a penalty parameter $\lambda$ that is too large can cause numerical problems. 
If, on the contrary, $\lambda$ is too small, the imposed constraints will have virtually no effect on the outcome. 
Choosing one penalty parameter $\lambda_i$ for each row of $A$ can be beneficial to derive tighter bounds for each $\lambda_i$ in some cases. 
The QUBO problem then takes the form 
\begin{align}
\label{eq:constrained_to_unconstrained_2}
\begin{split}
	\argmin_{x\in \lbrace 0,1\rbrace^{n} } \quad  x^\top Q x + c^\top  x + \sum_i \lambda_i \left(   (Ax)_i-b _i \right)^2,
\end{split}
\end{align}
see, \eg~Ref.~\cite{SeelbachBenkner2020}. 
Still, it is in most cases impossible to know in advance the optimal values of the rectification weights $\lambda_i$. 
The selection of the $\lambda_i$ is often performed by a grid search on small problem instances~\cite{QuantumSync2021, Arrigoni2022}, though there is no guarantee that such weights will generalise to other and larger problem instances (\textit{\eg~} in synchronisation problems). 
Zaech \textit{et al.}~\cite{zaech2022adiabatic} proposed a heuristic to find the minimally sufficient value $\lambda_{i, \text{min}}$ of the rectification weights. 
Their approach can be applied to an arbitrary number of Lagrange multipliers: They model $\lambda_{i}$ (\textit{\ie~} each Lagrangian) as 
\begin{equation} 
\lambda_{i} = \lambda_b + \lambda_{i}^{'} + \lambda_{\text{off}}, 
\end{equation} 
where $\lambda_b$ is a small base value, $\lambda_{\text{off}}$ is a spectral gap  offset and $\lambda_{i}^{'}$ is estimated during optimisation. 
Optimisation starts with $\lambda_{i}^{'} = \lambda_{\text{off}} = 0$ and calls the rectified QUBOs iteratively until a solution of low energy satisfying the constraints is found. 
This approach monitors the cost reduction due to violation of individual constraints  (due to too low $\lambda_{i}$) after each QUBO sampling and increases  $\lambda_{i}^{'}$ in every iteration. 
A publication solely focusing on using augmented Lagrangian methods to deal with penalty parameters in quantum annealing is Djidjev~\cite{djidjev2023logical}. 

\parahead{Linear inequality constraints} 
Linear inequality constraints are often incorporated with so-called slack variables $s$ :
\begin{equation}
Ax -b\leq 0  \quad \Leftrightarrow \quad  Ax -b= s,
\end{equation}
which are restricted to the values
$s_i\in\{ (Ax -b)_i| x \in  \mathcal{S} \}$, where $\mathcal{S}$ is the feasible set for $x$ here. In practice, one can extend the value set of $s_i$ with other non-positive integers. These values will, however, not give the desired solution of $Ax -b= s$. 
A reference to find this reformulation of inequality constraints is~Witt et al.~\cite{witt2022tactile}, in which the authors look into solving integer linear programs on the quantum annealer. 

Recently, Montanez-Barrera~\cite{montanez2024unbalanced} proposed an unbalanced penalisation that does not require additional slack variables to encode inequality constraints. 
Their core idea lies in penalising $\lambda \exp(Ax-b)$, which is higher the more the constraint $Ax-b < 0$ is violated. 
However, as the term $\exp(Ax-b)$ is non-quadratic in $x$, they regularise the QUBO with $\lambda f(Ax-b)$, where $f$ is the second-order Taylor approximation of $\exp(\cdot)$. 
\parahead{Hard constraints} 
Although regularization can lead to accurate results with proper choices of the regularization coefficients, it still does not lead to exact constraint satisfaction. 
Therefore, Yurtsever and colleagues~\cite{yurtsever2022q} follow another approach to solving \emph{constrained} QBO problems and introduce \emph{Quantum Frank Wolfe} (Q-FW). 
They begin by showing that the problems presented in constrained-QBO forms are equivalent to \emph{copositive programs} (CPs). 
Specifically, consider:
\begin{equation}\label{eq:QBO}
	\begin{aligned}
&\min_{x \in \lbrace 0,1\rbrace^{n^2}} && \quad x^\top Q x + 2\,c^\top x \\
&\ \quad \text{s. t.} && \quad a_i^\top x = b_i, ~~~ i = 1,\ldots,m.
	\end{aligned}
\end{equation}

Then,~\eqref{eq:QBO} is tightly equivalent to the copositive problem:
\begin{equation}\label{eqn:copo}
\begin{aligned}
	\min_{x, X} & \quad \mathrm{Tr}(QX) \\ 
	\text{s. t.} &\quad a_i^\top x = b_i, ~~~ i = 1,\ldots,m, \\[-0.25em]
	& \quad \mathrm{Tr}(A_i X) = b_i^2, ~~~ i = 1,\ldots,m, \\
	& \quad \mathrm{diag}(X) = x, \text{~and~} \begin{bmatrix}
		1 & ~ x^\top \\
		x & ~ X~
	\end{bmatrix}
	\in \Delta^{n+1}.
\end{aligned}
\end{equation}
where $A_i := a_i a_i^\top$ and $\Delta^{n}:=\mathrm{conv}\{xx^\top: x \in \mathbb{Z}_2^n\}$ is the \emph{copositive cone}. One can further compactify this form into a convex (but still NP-hard) optimization:
\begin{equation}\label{eqn:copo-compact}
\begin{aligned}
	\min_{W \in \Delta^p} ~ \mathrm{Tr}(CW) \quad \text{s. t.} \quad 
	& \mathcal{A}W = v,
\end{aligned}
\end{equation}
where $\mathcal{A}$ and $v$ are respectively a linear map and a vector combining all affine constraints in problem \eqref{eqn:copo}.

To solve problem \eqref{eqn:copo-compact}, the authors utilize an \emph{\textbf{a}ugmented-\textbf{L}agrangian} variant of \textbf{F}rank-\textbf{W}olfe (also known as conditional gradient method)~\cite{FrankWolfe1956,yurtsever2019conditional} and develop the hybrid quantum-classical \emph{Q-FWAL} algorithm. 
Despite suffering from sub-linear convergence, Q-FWAL offers several advantages:
(i) It is an algorithm for solving general CPs and as such any problem that can be formulated into a CP can leverage Q-FWAL as an optimizer, which is more general than QUBOs;
(ii) The authors have shown how Q-FW can allow for the integration of \emph{inequality constraints}; 
And (iii) having the CP formulation can potentially allow for solving non-binary problems~\cite{prakhya2024convex}. The possibilities that Q-FWAL opens up are to be investigated in the future. 

\parahead{Handling higher-order terms}
QUBO problems are quadratic by definition. 
Often in CV, however, one has to optimise so-called \emph{pseudo-Boolean functions}, \ie~to solve problems of the form 
\begin{equation}
\label{eq:pseudo_boolean}
\argmin_{x\in \{0,1\}^n } \sum_{S \subset [n]} w_S \prod_{j\in S} x_j,
\end{equation} 
where the subsets $S$ of $[n] = \left\{1, \ldots, n\right\}$ can be such that $|S| > 2$. 
QUBOs are particular cases of such problems with $|S| \leq 2, \forall S \subset [n]$. 
Solving the more general class of problems in \eqref{eq:pseudo_boolean} with quantum annealers involves \emph{quadratising} the objective function, typically at the cost of ancilla (additional) qubits. 
Next, we briefly review some quadratisation techniques. 
A broader overview of those can be found in~Refs.\cite{boros2014quadratization,dattani2019quadratization}. 
For negative terms, Freedman and Drineas~\cite{freedman2005energy} proposed the term-wise reduction
\begin{equation}
\label{eq:quadratization}
-x_1 \cdots x_d = \min_{z\in \{0,1\}} z\left[(d-1) - \sum_{j=1}^d x_j  \right],
\end{equation}
which introduces one ancilla variable $z$ per term. 
By considering the negated variables $\Bar{x}_j = 1-x_j$, one finds from Eq.~ \eqref{eq:quadratization} a reduction for positive terms as well, where each term of degree $d$ needs $d-2$ ancilla qubits; see~Rodr{\'\i}guez-Heck~\cite{rodriguez2018linear}. 
The shortcoming of such term-wise reduction, in addition to increasing the number of variables, is that they are ``local'' in nature and do not capture the ``global'' replacement effect of the terms. 
Ishikawa~\cite{ishikawa2009higher,ishikawa2010transformation} proposed a more compact quadratisation for positive terms, using only about half as many variables as the previous term-wise methods. 
In~Ref.~\cite{fix2011graph}, a more global reduction technique is proposed to transform a group of terms at once, using asymptotically the same number of ancilla variables as Ishikawa~\cite{ishikawa2009higher}. 
A method that does not require additional variables was recently proposed by Kuete Meli \textit{et al.}~\cite{meli2025qucoop}; see Fig.~\ref{fig:qucoop}. 
Their framework termed QuCOOP enables tackling
binary parameterised problems on quantum annealer. 
The starting problem is constrained, quadratic and of the form
\begin{equation}
\label{eq:original_problem}
\argmin_{s \in \mathcal S} \ s^\top Q s,
\end{equation}
where $ \mathcal{S} \subseteq \mathbb{R}^k $ is a potentially non-convex set. 
If $ \mathcal{S} $ can be parametrised by a binary variables $ x \in \{0,1\}^n $ through a function $ g: \{0,1\}^n \to \mathcal{S} $, the problem reformulates as: 
\begin{equation}
\argmin_{x\in \{0,1\}^n } \ g(x)^\top Q g(x),
\end{equation}  
which is, for $ g $ linear in $ x $, a standard QUBO. 
QuCOOP comes into play when $ g $ is non-linear in $x$. 
To enable optimisation, the authors consider a smooth approximation of $ g $ and apply first-order Taylor expansion to construct a sequel of QUBOs: 
Initialise $x^0$ and repeat for $t = 0, 1, 2, \ldots$
\begin{align}
g^t (x) &:= g(x^t) + \Braket{\nabla g(x^t), x - x^t}\label{eq:linearization}\\
\begin{split}
	\label{eq:minimization_step}
	x^{t+1} &\gets \argmin_{x \in \mathcal \{0,1\}^n} \ g^t(x)^\top Q g^t(x) \quad \mathrm{s.t.} \quad g^t (x) \in \mathcal S,
\end{split}
\end{align}
where the constraint $g^t(x) \in \mathcal{S}$ is expected to be a quadratic form in the QUBO objective via penalty parameters.
The algorithm then returns $ s = g(x^{t+1}) $ as a solution of problem \eqref{eq:original_problem}. 
Despite assuming smoothness, in theory, only binary values of $ x $ are meaningful parameters, justifying the use of quantum annealers for sampling. 
The authors prove that QuCOOP monotonically decreases the objective function. 

\begin{figure}[t] 
\begin{center} 
	\includegraphics[width=.98\linewidth] {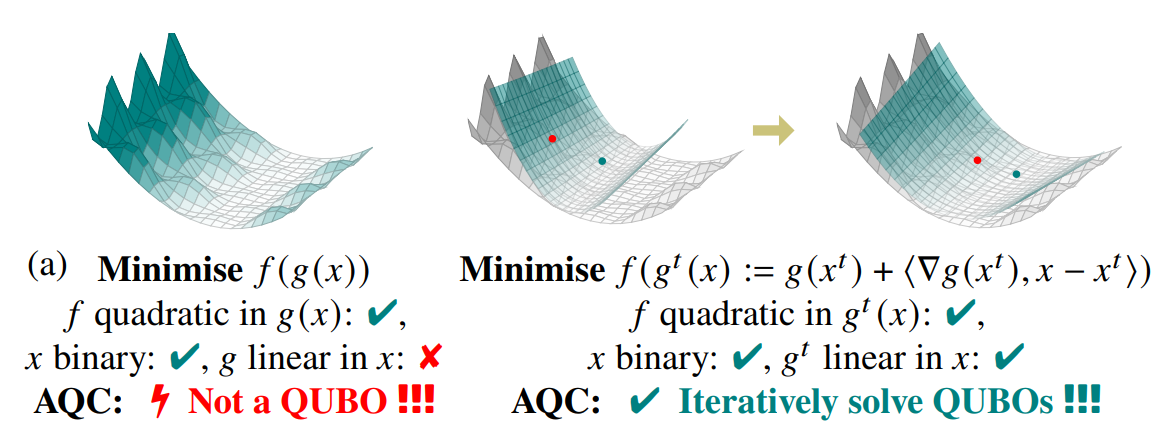} 
\end{center} 
\vspace{-.2cm}
\caption{QuCOOP optimizes binary-parametrised problems by iteratively solving QUBOs on quantum annealers~\cite[IEEE \textcopyright 2025]{meli2025qucoop}.} 
\label{fig:qucoop} 
\end{figure} 

\subsubsection{Mapping to Gate-based QC} 
\label{sec:ProblMapMethGateBased}
\begin{figure*}
\centering
\includegraphics[width=1.0\linewidth]{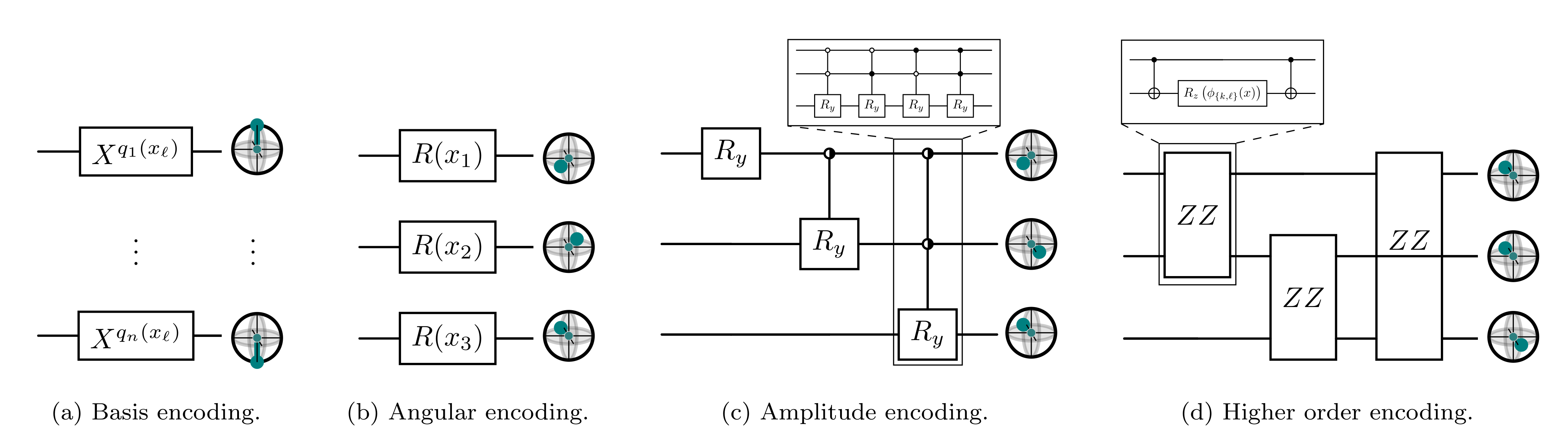}
\caption{Examples of quantum circuits for widely used encoding schemes of classical data into quantum states. 
	All input qubits are assumed to be in the initial $\ket{0}$ state. 
	(a) Basis encoding of a length-$n$ bit-string $\ket{x_\ell} = \otimes_{j = 1}^{n} \ket{q_j(x_\ell)}$ corresponding to the binary representation of entry $x_\ell$ of a data vector $x$.
	The operators $X^{q_j(x_\ell)}$ are either $I$ or $X$ according to  $q_j(x_\ell) \in \left\{0, 1\right\}$, which correspond to leaving the input Bloch vector unchanged or flipping it, respectively. 
	(b) Angular encoding of a length-$3$ vector $x$.
	Each qubit $\ell$ is rotated by an angle equal to the entry $x_\ell$ of $x$. 
	(c) Amplitude encoding of a length-$8$ vector $x$. 
	We refer to Schuld and Petruccione~\cite[Chap. 5]{schuld2018supervised} 
	for explicit formulae of how the rotation angles depend on the input vector. 
	(d) A $ZZ$-feature map as an example of higher-order encoding of a length-$3$ vector $x$; see also Havl{\'\i}{\v{c}}ek \textit{et al.}~\cite[Fig. 1c]{havlivcek2019supervised}.
}
\label{fig:data_encoding}
\end{figure*}

Unlike adiabatic quantum computers, there are many ways of mapping problems into gate-model quantum algorithms. 
Yet, all gate-based algorithms require the data to be encoded in quantum states. 
Quantum-enhanced CV methods fall into the category ``classical data and quantum processing'', since data is stored classically and first needs to be converted to a qubit statevector representation suitable for processing on quantum hardware. 
The requirement of encoding imposes certain restrictions on how classical data can be encoded. 
Thus, all quantum data representations have to be a valid qubit state, \ie~normalised. 
Suppose that we are given an $m$-dimensional real signal $x = (x_1, \ldots, x_m)$ that we wish to encode in a quantum state vector. 
Several encoding schemes have been proposed in the literature so far, differing in qubit properties that encode the input classical inputs. 
Common encoding schemes for gate-based quantum computers include basis, angular, amplitude and high-order encodings (see Fig.~\ref{fig:data_encoding}); while quantum annealers typically employ Ising encoding, as introduced in problem~\eqref{eq:optimization_problem}.
\vspace{0.5mm}
\parahead{Basis encoding} 
Assume that each entry $x_\ell$ of $x$ has a (possibly padded) binary representation $(q_1(x_\ell), \ldots, q_{n}(x_\ell))$, $q_j(x_\ell)\in \left\{0, 1\right\}$ such that $x_\ell$, can be directly mapped into a computational basis state $\ket{x_\ell} = \otimes_{j = 1}^{n} \ket{q_j(x_\ell)}$. 
For encoding the whole vector $x$ in basis encoding, one prepares a superposition of such computational basis states, which component-wise corresponds to the binary input pattern. 
Hence, the quantum state encoding of the data $x$ is given by 
\begin{equation}
\label{eq:basis_encoding}
\ket{\psi(x)} = \frac{1}{\sqrt{m}}\sum_{\ell = 1}^m \ket{x_\ell} 
\end{equation}
and requires $n$ qubits in total. 
As a system of $n$ qubits has $2^n$ basis states, we need $n$ such that $m\leq 2^n$ for encoding $x$. 
The quantum state in Eq.~\eqref{eq:basis_encoding} is not well defined if $\ket{x_\ell} = \ket{x_k}$ for some $\ell{\neq}k$. 
One can address this with a so-called quantum random access memory (QRAM), which introduces an index register to prepare the entangled state 
\begin{equation}
\ket{\psi(x)} = \frac{1}{\sqrt{m}}\sum_{\ell = 1}^m \ket{l}\ket{x_\ell},
\end{equation} 
where $\ket{\ell}$ is the $\log_2(m)$-length binary encoding of~$\ell$. 
Basis encoding offers the most computational freedom to apply arithmetical and trigonometrical operations on the data. 
Just as a Toffoli gate stores intermediate results in an extra qubit to revert classical operations, one adds a quantum register to the basis-encoded data to hold the outcome of the operation.
Hence, an operator $A$ would act as $A\ket{x_\ell}\otimes \ket{0} \mapsto \ket{x_\ell}\otimes \ket{f(x_\ell)}$, where $\ket{f(x_\ell)}$ encodes the computation result. 
Basis encoding is intuitive for encoding discrete data such as integers or binary strings. 
For non-integer data, it is, however, less suitable without further pre-processing. 
\vspace{0.3mm}
\parahead{Angular encoding} 
Angular encoding uses $m$ qubits to encode the data $x$. 
Each entry $x_\ell$ serves as one parameter of the single-qubit gate $U$ acting on the initial state $\ket{0}$; cf. Eq.~\eqref{eq:genral_unitary}. 
The quantum state holding the data is given by
\begin{equation}
\ket{\psi(x)} = \otimes_{\ell = 1}^m \exp(i x_\ell G) \ket{0},
\end{equation}
for some generator $G$. 
Angular encoding is straightforward to implement and is not restricted to integer values. 
Also, it is more resource-efficient than basis encoding. 
However, angular encoding can be more restrictive compared to other encoding schemes in terms of operations applicable to the data. 
\parahead{Amplitude encoding} 
Amplitude encoding loads the vector $x$ into the amplitude of a $\log_2(m)$-qubit quantum state, which typically requires normalising the vector $x$. 
The quantum state holding the data is given by 
\begin{equation}
\ket{\psi(x)} = \frac{1}{\|x\|} \sum_{\ell=1}^{m} x_\ell\ket{\ell},
\end{equation} 
where $\ket{\ell}$ is the $\log_2(m)$-length binary encoding of $\ell$.

Amplitude encoding is a resource-efficient encoding scheme, as it uses exponentially fewer qubits than the dimension of the data to encode it. 
It benefits from quantum parallelism to transform all dimensions of $x$ at once. 
However, the gate complexity to prepare a quantum state in amplitude encoding can be high and the preparation sometimes requires complex circuit architectures. 
%

Fig.~\ref{fig:3DQAE} illustrates an example of amplitude encoding for a $3$D point. 
The $3$D vector is extended to a power-of-$2$ length vector by adding an auxiliary value. 
This resulting vector is then normalised and converted into a $2$-qubit quantum state vector through amplitude encoding. 

\begin{figure}[t] 
\begin{center} 
	\includegraphics[width=1.0\linewidth] {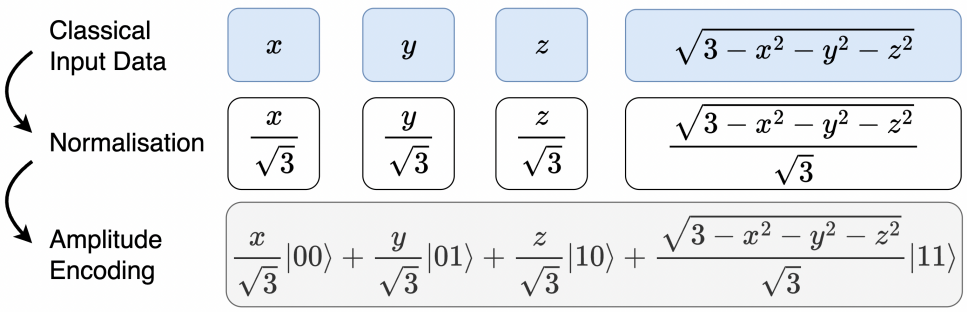} 
\end{center} 
\caption{Amplitude encoding for 3D points as proposed in 3D-QAE~\cite[courtesy of the authors]{Rathi2023}.} 
\label{fig:3DQAE} 
\end{figure}

\vspace{0.5mm}
\parahead{Higher-order encoding} 
Higher-order encoding can be used to transform the data into a higher-dimensional feature space, where it is easier to process, \eg~by using kernel methods. 
An example of higher-order encoding is the $ZZ$-feature map introduced in Ref.~\cite{havlivcek2019supervised} to enhance support vector machines on classification tasks. 
Given first a feature map $\phi: \mathbb{R^n} \to \mathbb{R}$ for some $n\leq m$, the $ZZ$-encoding encodes the data in a unitary operator of the form 
\begin{equation}
U(x) = \exp \left(i \sum_{S\subseteq [m]} \phi_S(x) \prod_{\ell\in S} Z_\ell\right),
\end{equation} 
where $Z_\ell$ is the Pauli-$Z$ operator acting on the $\ell$-th qubit of the system and $S$ a subset of the index set $[m] = \left\{1, \ldots, m\right\}$.
An example of $\phi$ is defined as $\phi_{\left\{k,\ell\right\}}(x) = (\pi - x_k)(\pi -x_\ell)$. 
The quantum statevector encoding the data is then given by 
\begin{equation}
\ket{\psi(x)} = U(x) \ket{0}^{\otimes m}. 
\end{equation} 
Hence, the number of qubits used in the $ZZ$-feature map is maximally equal to the dimension $m$ of the input vector. 
Other generators than the Pauli-$Z$ operators can also be used for higher-order encoding. 
Higher-order encoding generally suits well algorithms operating in a feature space. 
\vspace{0.5mm}
\parahead{Ising encoding} 
The encoding schemes above are merely concerned with the data and not the problem. 
In contrast, Ising encoding as defined in problem \eqref{eq:optimization_problem}, which is required for AQC-based methods, encodes both the problem and the data in qubit biases $b$ and couplings $J$; binary vectors sampled on an AQCs encode solutions to the problem. 
Nothing prevents Ising encodings from also being used in gate-based techniques and QML. 
Common circuit architecture designs shown in Fig.~\ref{fig:data_encoding} are implementation examples for the encoding schemes mentioned above. More resource-efficient circuit architectures exist; see Ref.~\cite{schuld2018supervised} and references therein. 
As a rule, QeCV practitioners do not have to design or implement data encoding schemes in most cases, as they are generally provided off the shelf. 

\subsection{Methods Relying on AQC}
\label{ssec:MethodsAQC} 
We now describe several prominent QeCV methods grounded in the adiabatic quantum computational paradigm. 
A high-level overview of the works we consider is provided in Table~\ref{tab:papers_aqc}. 
Predominantly, these methods are optimization- or energy-based, with experimental execution of the algorithms on real quantum hardware.

\begin{table*}
\centering
\caption{Overview of AQC-based methods for CV problems.
	We report problem sizes and the number of logical qubits tested, if mentioned in the papers.
	``$^\star$'': According to the experiments reported, the methods can also solve larger-sized problems. ``\# Qubits'' refers to the number of logical qubits. 
	\\ 
	Keys: 
	\begin{tabular}[t]{lllll}
		``TE'' & transformation estimation &``(R)MMF''& (robust) multi-model fitting \\
		``PSR''& point set alignment & ``(B)KC''& (balanced) k-means clustering \\
		``GM'' & graph matching & ``SISR''& single image super-resolution \\
		``PS''& permutation synchronisation & ``MCSVM''& multi-class support vector machines \\
		``MA''&  mesh alignment & ``UIM''& unsupervised image segmentation \\
		``OD, OT''& object detection/tracking & ``MS''& motion segmentation \\
		``RF''& robust fitting & ``SM''& stereo matching.
	\end{tabular}
}
\begin{tabular}{ccccccc}\hline
	Method & Problem         & Input type & Problem size$^\star$& \# Qubits & Iterative & Fig.\\\hline\hline
	QA, CVPR'20~\cite{golyanik2020quantum}&TE, PSR&point clouds & $\leq$5k$^\star$ points & $\sim$140&&\\
	IQT, CVPR'22~\cite{Meli_2022_CVPR}&TE&point clouds & $\leq$1.5k$^\star$ points  & $\sim$15 &\checkmark &\\
	QuAnt, ICLR'23~\cite{seelbach2022quant}&TE, PSR, MA&point clouds, meshes & $\leq$2k$^\star$ points, 5$^\star$ mesh vertices & $\sim$15&&Fig. \ref{fig:qcvQuAnt}\\
	QuCOOP, CVPR'25~\cite{meli2025qucoop}&PSR, MA&point clouds, meshes & $\leq$50$^\star$ points, 502$^\star$ mesh vertices & $\sim$36&\checkmark&Fig. \ref{fig:qucoop}\\\hline
	QGM, 3DV'20~\cite{SeelbachBenkner2020}&GM&graphs & $\leq$4 graph nodes & $\sim$50&&\\
	Q-Match, ICCV'21~\cite{SeelbachBenkner2021}&MA&meshes & $\leq$502 mesh vertices & $\sim$250&\checkmark&\\
	CcuantuMM, CVPR'23\cite{bhatia2023ccuantumm}&MA&meshes & $\leq$1k mesh vertices & $\sim$40&\checkmark&Fig. \ref{fig:CCuantuMM_triplets} \\\hline
	QSync, CVPR'21~\cite{QuantumSync2021}&PS, GM&permutation matrices & $\leq$ 3x3 perm. mat., 8 views & $\sim$72 &&Fig. \ref{fig:qsync}\\\hline
	QSQS, 
	ECCV'20~\cite{LiGhosh2020}&OD&bounding boxes & $\leq$45 bounding boxes & $\sim$45&&Fig. \ref{fig:QUBO_suppression} \\
	QMOT, CVPR'22~\cite{zaech2022adiabatic}&OT&tracks, detections & $\leq$3 tracks, 5 frames & $\sim$100&&\\\hline
	Doan \textit{et al.}, CVPR'22~\cite{Doan2022}&RF&data points & $\leq$100 points & $\sim$100&\checkmark&Fig. \ref{fig:Doan_robust_fitting}\\
	DeQUMF, CVPR'23~\cite{Farina2023}&MMF&data points & $\leq$1k models, 250 points & $\sim$100&\checkmark&\\
	Pandey \textit{et al.}, CVPRW'25~\cite{pandey2025outlier}&RMMF& data points & $\leq$2k models, 10k points & $\sim$120 &\checkmark\\\hline
	Bauckhage \textit{et al.}, LWDA'18\cite{Bauckhage2018AdiabaticQC}&KC&data points & $\leq$16 points & $\sim$16&&\\
	Arthur and Date, QIP'21\cite{arthur2021balanced}&BKC&data points & $\leq$21 points, $k=3$ means & $\sim$64&&\\
	Nguyen \textit{et al.}, ArXiv'23\cite{nguyen2023quantum}&KC&feature vectors & - & -&&\\
	Zaech \textit{et al.}, CVPR'24\cite{zaech2024probabilistic}&BKC&data points & $\leq$45 points, $k=3$ means & $\sim$45&&\\\hline
	Choong \textit{et al.}, CVPR'23~\cite{choong2023quantum} &SISR&images, dictionary & $\leq$15x20 LR to 45x60 HR images & $\sim$100&&Fig. \ref{fig:QSISR} \\\hline
	QMSVM, IEEE'23~\cite{delilbasic2023single}&MCSVM&feature vectors & $\leq$60 vectors, 3 classes  & $\sim$360&&\\
	Zardini \textit{et al.}, ArXiv'24~\cite{zardini2024local}&MCSVM&feature vectors & $\leq$24 vectors, 3 classes  & $\sim$144&&\\
	Q-Seg, IEEE'24~\cite{venkatesh2023q}&UIM&image patches & $\leq$32x32 images & -&&\\
	QuMoSeg, ECCV'22~\cite{Arrigoni2022}&MS&landmark points & $\leq$16 landmarks, 2 motions & $\sim$128&&\\\hline
	Santos \textit{et al.}, MDPI'18~\cite{cruz2018qubo}&SM&stereo images & $\leq$15x15 image patches & -&&\\
	Heidari \textit{et al.}, IVCNZ'21~\cite{Heidari2021}&SM&stereo images & $\leq$383x434 image patches & -&&Fig. \ref{fig:Q_stereo}\\
	Braunstein \textit{et al.}, 3DV'24\cite{braunstein2024quantum}&SM&stereo images & - & -&\checkmark&\\
	\hline
\end{tabular} 
\label{tab:papers_aqc} 
\end{table*}

\parahead{Correspondence problems on point sets}
Correspondence estimation between point sets is a fundamental task in many computer vision applications, such as 3D reconstruction~\cite{takimoto20163d}, medical image processing~\cite{modersitzki2009fair} and robotics~\cite{calli2017yale}. 
The goal is to find a transformation $\mathcal{T}$ that minimizes the distance between points in a reference set $X = \{x_i\}_{i=1}^n$ and points in a template set $Y = \{y_i\}_{i=1}^m$. 

Golyanik and Theobalt~\cite{golyanik2020quantum} introduced an AQC-based method for rigid point set alignment, where the transformation is constrained to be a rotation. 
They approximated the rotation matrix as a binary-weighted sum of basis matrices and minimised, over the binary coefficients, the sum of squared distances (SSD) between the reference and transformed points---a formulation that naturally leads to a QUBO.
Building on this, Kuete Meli \textit{et al.}~\cite{Meli_2022_CVPR} proposed a strategy to reduce the number of binary coefficients required and iteratively refine the rotation matrix approximation. 
They opted for a first-order Taylor approximation to linearise the rotation matrix around current rotation parameters, applying quantum annealing to sample the next binarised coefficients that minimise the SSD error. 
By restricting the rotation parameter to a tight interval around the current estimate, they ensured the matrix remained orthogonal at each iteration.
Later, Kuete Meli \textit{et al.}~\cite{meli2025qucoop} generalised this method in the QuCOOP framework and extended the application to rigid point set alignment without correspondences known in advance.

\parahead{Graph and shape matching}
Graph matching involves finding node correspondences between two graphs that preserve structural similarity, meaning matched nodes exhibit similar connectivity patterns. 
At scale, it can express correspondence estimation between geometric shapes (\eg~3D human scans in different poses~\cite{Bogo:CVPR:2014}). 
Here, the goal is to establish pointwise matches between geometrically similar parts of two meshes $M$ and $N$, considering point vicinity.
The problem can be formulated as a quadratic assignment problem with the energy function $ E $ of the form \eqref{eq:constrained_to_unconstrained_0}. 

Seelbach Benkner \textit{et al.}~\cite{SeelbachBenkner2020} executed graph matching on an AQC device and applied it to the correspondence estimation. 
It was already conceived in 2008~\cite{neven2008imagerecognitionadiabaticquantum} that one can perform image matching with quantum annealers in a similar fashion. 
The approach in Ref.~\cite{SeelbachBenkner2020} used penalty terms to enforce permutation matrix constraints and constructed a QUBO in the way described in Section \ref{sec:ProblMapMethAnnealers}, Eqs.~\eqref{eq:constrained_to_unconstrained_1} and \eqref{eq:constrained_to_unconstrained_2}. 
The linear system of equations $Ax=b$ can enforce permutations if $A=  \begin{bmatrix}
I \otimes \mathds{1} ^\top   \\
\mathds{1}^\top  \otimes  I   
\end{bmatrix}$,
and the vector $\mathbf{b}$ contains only ones. The upper part of $A$ is for enforcing that the rows sum up to one, while the other part does the same for the columns.
In Ref~\cite{SeelbachBenkner2020},  upper bounds for $\lambda$ or $\lambda_i$ were developed.
In a third formulation, the authors eliminated some binary variables by solving one of the linear equations directly for those variables.
The performance of their formulation was inhibited by the experimental errors in the couplings, which effectively reduced the precision. The paper only achieved success probabilities better than random guessing for $N{=}3$. 
To improve this, penalty parameters could also be chosen empirically as in QSync~\cite{QuantumSync2021} or with an iterative method~\cite{zaech2022adiabatic}. 
Another way to avoid penalty terms is by considering some fusion moves comparable to Hutschenreiter et al.~\cite{hutschenreiter2021fusion} as was done in Q-Match~\cite{SeelbachBenkner2021}. 
Q-Match~\cite{SeelbachBenkner2021} is a method using a quantum annealer to match two shapes. 
In this formulation, the weight matrix $ Q \in \mathbb{R}^{n^2 \times n^2} $ can be populated with differences in geodesic distances between point pairs of $ M $ and $ N $~\cite{SeelbachBenkner2021}.
To avoid the linear penalty terms imposing permutation matrix constraints, Q-Match uses \textit{cyclic alpha expansion} whose main idea is to start from an initial valid permutation matrix and perform a series of iterative updates that guarantee that the current solution remains a permutation matrix in every iteration. 
Hence, instead of solving the regularised problem \eqref{eq:constrained_to_unconstrained_1} directly, Q-Match considers in each step 
\begin{equation} 
P^{(k+1)}= \argmin_{ \{ P\in \mathbb{P}_n |  \exists  \mathbf{\alpha} \in  \{ 0,1\}^m: \ P = \left(\prod_i c_i^{\alpha_i} \right) P^{(k)} \}  } E(P), \label{eq:IterationStep} 
\end{equation} 
where $\mathbb{P}_n$ is the set of permutation matrices and $c_i$ are from the given set of cycles.
The authors later expressed the optimisation variable $P(\alpha) = \left(\prod_i c_i^{\alpha_i} \right) P^{(k)}$ linearly in $\alpha$ as $p(\alpha) = P^{(k)} + \sum_{i=1}^m \alpha_i(c_i - I)P^{(k)}$ for disjoint cycles $c_i$, leading to a QUBO when plugged into the energy function $E$.
The method was tested on the FAUST~\cite{Bogo:CVPR:2014} dataset as well as on QUAPLIB~\cite{burkard97qaplib}.
\begin{figure}[t] 
\begin{center} 
	\includegraphics[width=.98\linewidth] {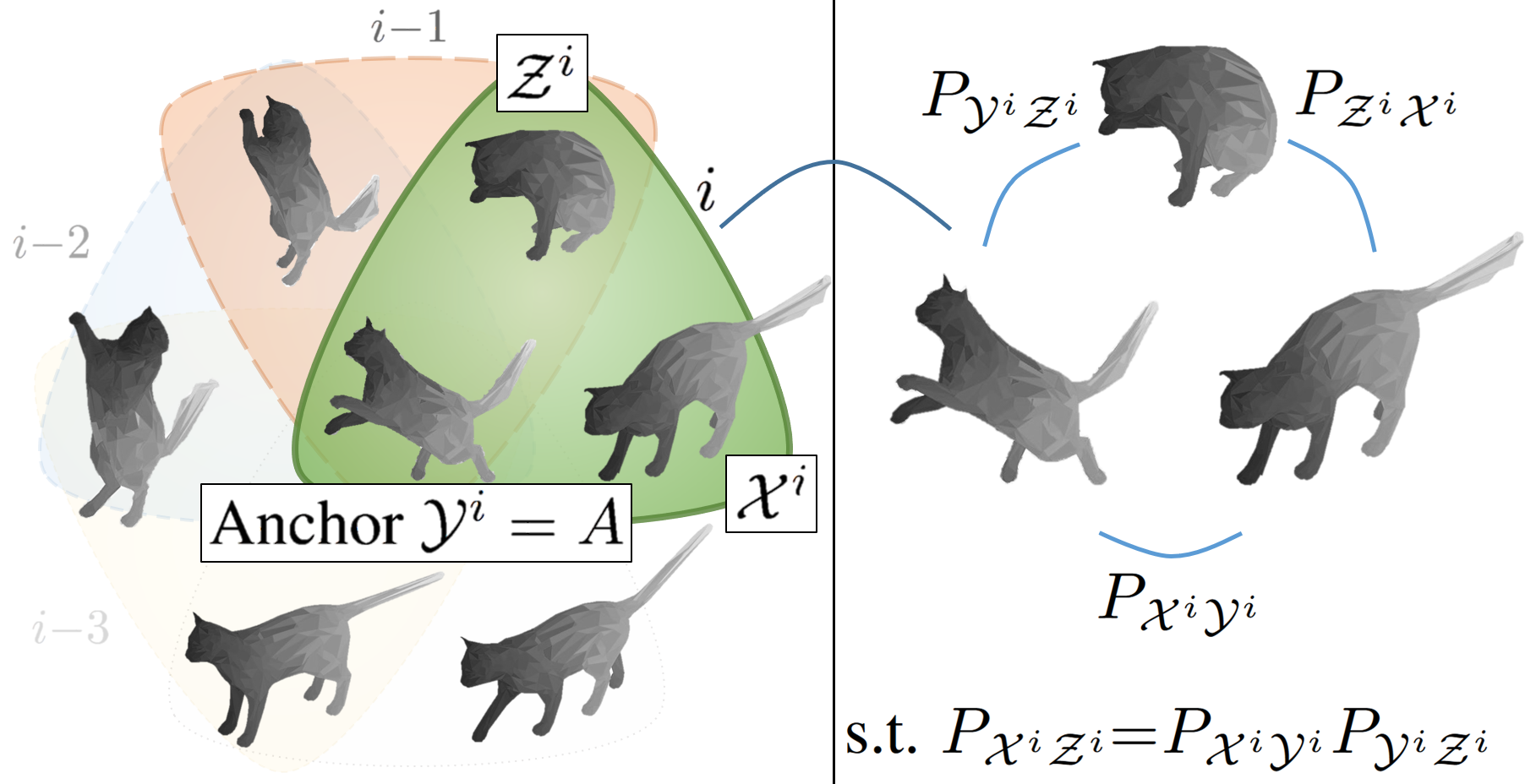} 
\end{center} 
\caption{CCuantuMM~\cite[IEEE \textcopyright 2025]{bhatia2023ccuantumm} for multiple-shapes matching.
	The method reduces the matching of N-shapes to that of shape triplets and guarantees cycle consistency.
} 
\label{fig:CCuantuMM_triplets} 
\end{figure} 

Bhatia \textit{et al.}~\cite{bhatia2023ccuantumm} extended Q-Match 
to match multiple shapes with guaranteed cycle consistency. 
In this method, optimization is always carried out with three shapes at a time; Fig.~\ref{fig:CCuantuMM_triplets} visualises this principle. 
Extending the Q-Match algorithm to handle three shapes in a naive manner leads to higher-order terms beyond quadratic, which cannot be directly mapped onto the D-Wave quantum annealer without further processing; see quadratisation techniques discussed in Sec.~\ref{sec:ProblMapMethAnnealers}.
In Ref.~\cite{bhatia2023ccuantumm}, however, the QUBO form is obtained by neglecting these higher-order terms. 
The evaluation confirms empirically that this is feasible. 
The experiments were performed on the shape datasets FAUST~\cite{Bogo:CVPR:2014}, TOSCA~\cite{bronstein2008numerical} and SMAL~\cite{Zuffi:CVPR:2017}. 
One key advantage of this method is that it scales linearly with the number of shapes. In this way, it was possible to match up to 100 shapes for FAUST.
The generalised QuCOOP framework~\cite{meli2025qucoop} was also applied to shape matching and showed to be competitive with Q-Match~\cite{SeelbachBenkner2021} on the FAUST dataset. 

\parahead{Multiple object tracking}
Object tracking is a data association problem aiming to assign measurements to targets or clutter (classified as false alarms or noisy measurements). 
Several recent works address this combinatorial data association problem and, in particular, multiple target tracking (MTT)~\cite{McCormick2022,mccormick2022multiple} with the help of QUBO formulations~\cite{McCormick2022} or Bayesian diabatic quantum annealing (DQA)~\cite{mccormick2022multiple}. 
Differently from tracking in computer vision, these works operate on sets of sensor measurements (not necessarily images) in a general context. 
The Quantum-Soft QUBO suppression is an early approach by Li and Ghosh~\cite{LiGhosh2020} for eliminating false positives in multi-object detection inspired by an earlier pedestrian detection approach~\cite{RujikietgumjornCollins2013}; see Fig.~\ref{fig:QUBO_suppression}. 
The authors demonstrate the advantages of their QUBO formulation compared to greedy and tabu search. 
Zaech \etal~\cite{zaech2022adiabatic} introduced the first AQC algorithm for multi-object tracking (MOT) based on assignment matrices. 
In MOT, object detections need to be associated through time when following the tracking-by-detection paradigm, often resulting in NP-hard discrete optimisation problems. 
Their MOT algorithm can be summarised as follows:
The input to the MOT method is a set of individual object detections in each video frame and the pairwise appearance similarities between these detections (obtained by an MLP). 
The goal of the method is to assign each detection to a track so that the total similarity score per track is maximised. 
The assignment problem is then formulated as a quadratic optimisation objective over assignment matrices mapping detections to tracks. 
These matrices introduce costs if two detections belong to a common track. 
To formulate a QUBO, the authors~\cite{zaech2022adiabatic} propose a policy for finding optimal rectification weights for the constrained assignment matrices. 
\begin{figure}[t] 
\begin{center} 
	\includegraphics[width=\linewidth] {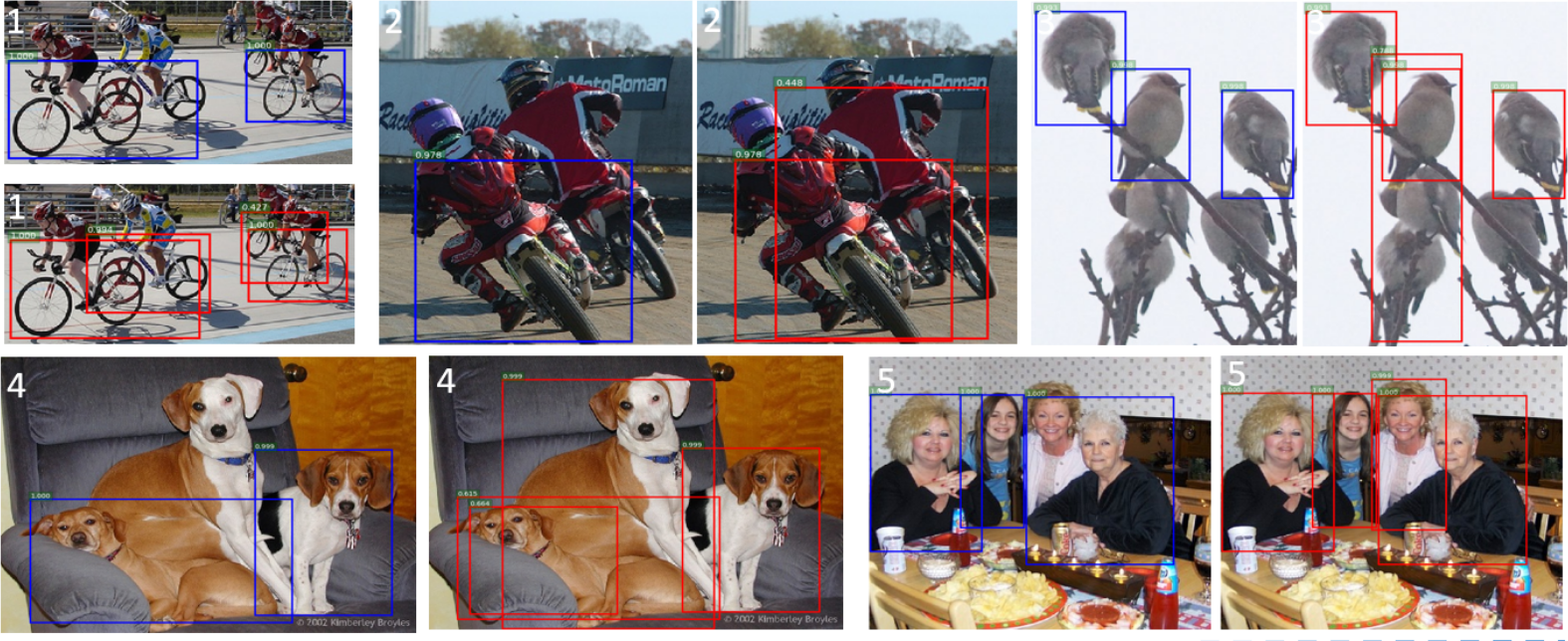} 
\end{center} 
\caption{Qualitative results of QSQS~\cite[Springer Nature \textcopyright 2025]{LiGhosh2020}, (red boxes) in comparison to a standard non-maximum suppression method (blue boxes). QSQS performs better on these images.\vspace{-3mm}} 
\label{fig:QUBO_suppression} 
\end{figure}

\parahead{Model fitting} 
Model fitting refers to the task of finding model parameters that best fit the data~\cite{meer1991robust}.
Methods for model fitting can be classified into single- and multi-model fitting approaches. 

The application of single-model fitting on quantum computers was investigated first for the gate-based paradigm~\cite{Chin_2020_ACCV}; see more details in Sec.~\ref{ssec:MethodsGateBased}.
Doan \textit{et al.}~\cite{Doan2022} proposed a hybrid-classical approach for robust fitting using its minimum vertex cover formulation, \ie~the dual of the consensus maximisation formulation: 
A binary vector $z$ is optimised to select the minimum number of points such that there exists a model for which the remaining points are all inliers. 
Other binary vectors $b_{(k)}$ define infeasible bases, each being a binary vector marking a small group of points that fail to support any model. 
The goal is then to solve $\min_z \|z\|_1$ so that $b_{(k)}^\top z \geq 1, \ \forall k$. 
The authors obtain a QUBO by introducing a binary slack variable to convert inequalities to equality constraints, which are then plugged into the objective with penalisation parameters; see the discussion in Sec.~\ref{sec:ProblMapMethAnnealers} on constraint handling. 
One of the core advantages of the method is that it either terminates with a globally optimal solution or a sub-optimal one along with a known error bound. 
Approaches such as that by Doan \textit{et al.} can be used for 1D linear regression, fundamental matrix estimation and multi-view triangulation (see Fig.~\ref{fig:Doan_robust_fitting}). 
Farina \textit{et al.}~\cite{Farina2023} subsequently introduced QuMF, a quantum-enhanced approach supporting multiple models. 
Their method is based on the disjoint set cover and a preference-consensus matrix, formulated as an $L_1$-minimisation over a binary vector subject to equality constraints, by assuming the number of target models to be known and the data not to be outlier-contaminated. 
Their QUBO is obtained by extending the objective with a penalty term enforcing the constraint. 
Importantly, the authors introduce a decomposition policy that allows splitting the original problem into parts of smaller sizes compatible with quantum annealing architectures.
This allows solving problems of substantial and practical sizes. 
Next, Pandey \textit{et al.}~\cite{pandey2025outlier} make several important improvements to QuMF, making multi-model fitting more practical, as their method does not assume the number of models to be known in advance and is explicitly designed to cope with outliers in the input data (which is often observed in practice). 
Outlier robustness is achieved thanks to the incorporation of explicit labels for outliers into the QuMF formulation, at the cost of a moderate increase in the number of qubits required for the Ising encoding. 

\begin{figure}[t]
\begin{center}
	\includegraphics[width=\linewidth]{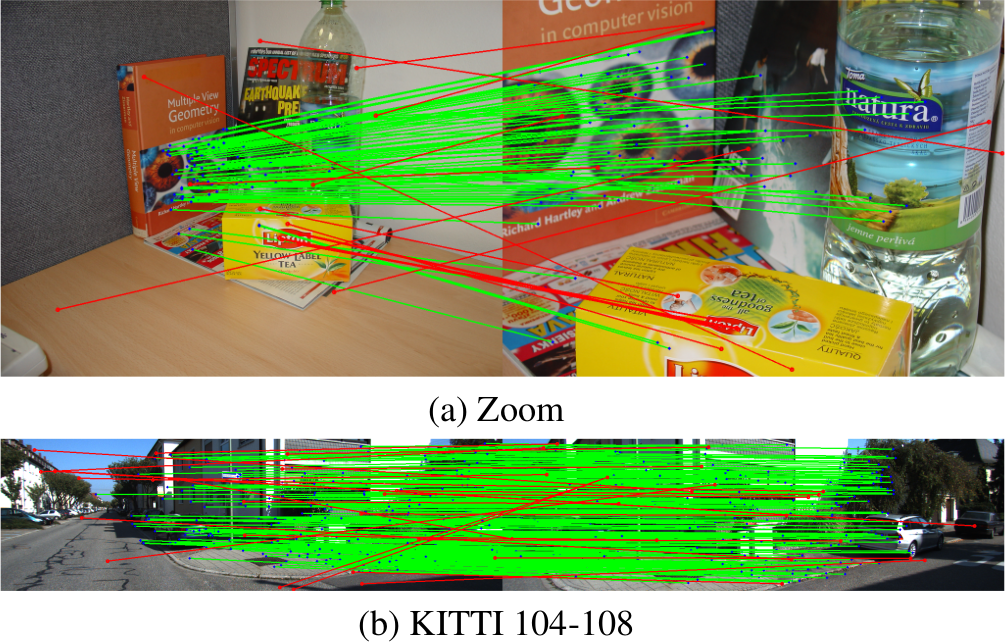}
\end{center}
\caption{Results of Doan \textit{et al.}'s approach~\cite[IEEE \textcopyright 2025]{Doan2022} for the linearised fundamental matrix fitting.
Green and red lines showcase inliers and outliers, respectively.} 
\label{fig:Doan_robust_fitting}
\end{figure}

\parahead{Clustering} 
The goal of clustering is to partition data into many classes in a way to optimises that data points within a class are close to each other, while data points from different classes should be far away~\cite{xu2005survey}.

The first works showing that clustering problems can be solved with adiabatic quantum computing appeared already in 2017 in different contexts~\cite{Bauckhage2018AdiabaticQC,ushijima2017graph}. 
In~\cite{Bauckhage2018AdiabaticQC}, $k$-means clustering with $k = 2$ was reformulated as a QUBO problem, and quantum annealing was simulated by numerically solving the Schrödinger equation.
On the other hand, Ushijima \etal~\cite{ushijima2017graph} already performed graph partitioning experiments on a quantum annealer. 

More recently, Arthur and Date~\cite{arthur2021balanced} developed a balanced $k$-means clustering formulation that allows for solving the problem exactly and not approximately, in contrast to other works.
The QUBO is obtained by minimising pairwise distances between points assigned to the same cluster:
Each point-to-cluster assignment is encoded in a binary matrix $W$, where column $w_j$ indicates membership in cluster $j$; a cost matrix $D$ is populated as $d_{mn} = \|x_m-x_n\|^2$ for input data points $(x_n)_{n=1,2,\ldots}$. 
The objective minimises $w_j^\top D w_j$ over all columns of $W$ and is subject to the constraints that each cluster must contain approximately the same number of points and that each point must belong to exactly one cluster.
Even though the proposed method does not outperform classical solutions, it demonstrates promising characteristics that could become useful as the quantum hardware continues to scale up in the future. 
Nguyen \textit{et al.}~\cite{nguyen2023quantum} even presents a formulation where the number of clusters does not have to be fixed in advance. 
Last but not least, Zaech~\etal~\cite{zaech2024probabilistic} presented a quantum version of \emph{balanced k-means} where the output samples are \emph{post-calibrated} to match the posterior distribution of the original problem. 
This allows for using the whole histogram of samples to improve the result or provide uncertainty estimates for the obtained solution. 
\parahead{Synchronisation problems} 
Many problems in multi-view computer vision involve integrating individual, pairwise estimates into consistent global estimates. The process of converting relative information into absolute information is known as \emph{synchronisation}, which is typically formulated as an optimisation problem~\cite{arrigoni2020synchronization}. 
The goal is to distribute discrepancies across a graph so that the estimates at each node remain consistent with one another, ensuring that the graph becomes \emph{cycle-consistent}.

Birdal \etal~\cite{QuantumSync2021} introduced a quantum-based approach called \emph{QSync} (Quantum Synchronisation) to tackle synchronisation problems involving permutations, illustrated in Fig.~\ref{fig:qsync}. 
They applied this method to a collection of point sets, each containing a fixed number of points, with a one-to-one correspondence between points in every pair of sets.
Given noisy relative permutations between these point sets, the goal is to solve a multi-graph matching problem by minimising a cycle-consistency loss. 
This is a constrained quadratic optimisation problem over permutation matrices, which is turned into a QUBO by incorporating the constraints in the objective via penalty parameters. 
QFWAL~\cite{yurtsever2022q} further improved upon this work by factoring in the permutation-ness as a hard constraint, as described in Sec.~\ref{sec:ProblMapMethAnnealers}. 
The authors identify several further applications that could benefit from the method, including structure-from-motion~\cite{birdalSimsekli2018, govindu2001combining, chatterjee2018robust}, point cloud registration~\cite{huang2022multiway, gojcic2020learning}, shape matching~\cite{gao2021isometric} and motion segmentation~\cite{arrigoni2022multi, huang2021multibodysync}, among others. 

\parahead{Image super-resolution} 
Image super-resolution consists in reconstructing a high-resolution image from one or more low-resolution inputs~\cite{lepcha2023image}. 
The goal is to enhance visual quality, often by recovering fine textures, edges and structures that are lost due to downsampling or degradation.

\begin{figure}[t]
\centering
\includegraphics[scale=0.5]{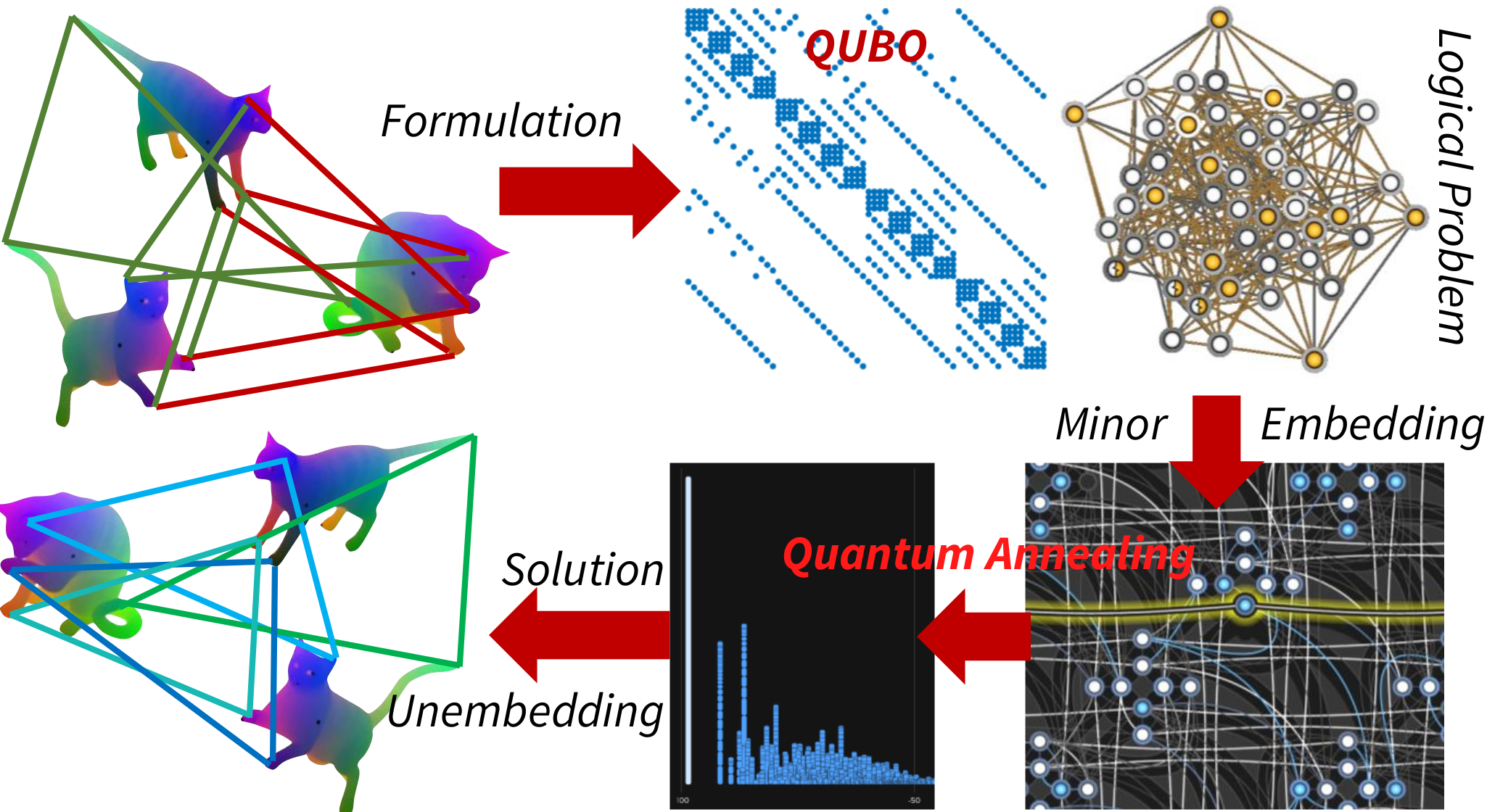}
\caption{The diagram of QSync~\cite[IEEE \textcopyright 2025]{QuantumSync2021}: A quantum permutation synchroniser operating in six QeCV steps: QUBO preparation, minor embedding, sampling, un-embedding, bitstring selection and interpretation.} 
\label{fig:qsync}
\end{figure}

Choong \textit{et al.}~\cite{choong2023quantum} proposed a quantum annealing-based algorithm for single-image super-resolution (SISR); see Fig.~\ref{fig:QSISR}.
Their approach relies on sparse representation learning and coding. 
In a prior step, two dictionaries $ D_l \in \mathbb{R}^{m \times d} $ and $ D_h \in \mathbb{R}^{n \times d} $, with $ m {<} n $, are learned from a dataset composed of image patches.
These dictionaries are then used to reconstruct low- and high-resolution image patches in a supervised manner by taking sparse linear combinations of the columns of $ D_l $ and $ D_h $, using a shared coefficient vector $ \alpha {\in} \mathbb{R}^d $. 
Given a new low-resolution image $y$, the goal is to solve for $ \alpha $ using $ D_l $ and then reconstruct the high-resolution version $x$ using $ D_h $ and the optimised $ \alpha $: 
$x = D_h\alpha^\star$ for $\alpha^\star = \argmin_\alpha \|D_l\alpha - y\|_2^2 + \lambda \|\alpha\|_1$, where the regularisation term enforces sparsity.  
Further, restricting $ \alpha $ to be a binary vector scaled by a tunable parameter enables the authors to derive the QUBO formulation, which can be solved using a quantum annealer~\cite{choong2023quantum}. 
\begin{figure}[t]
	\begin{center}
		\includegraphics[width=\linewidth]{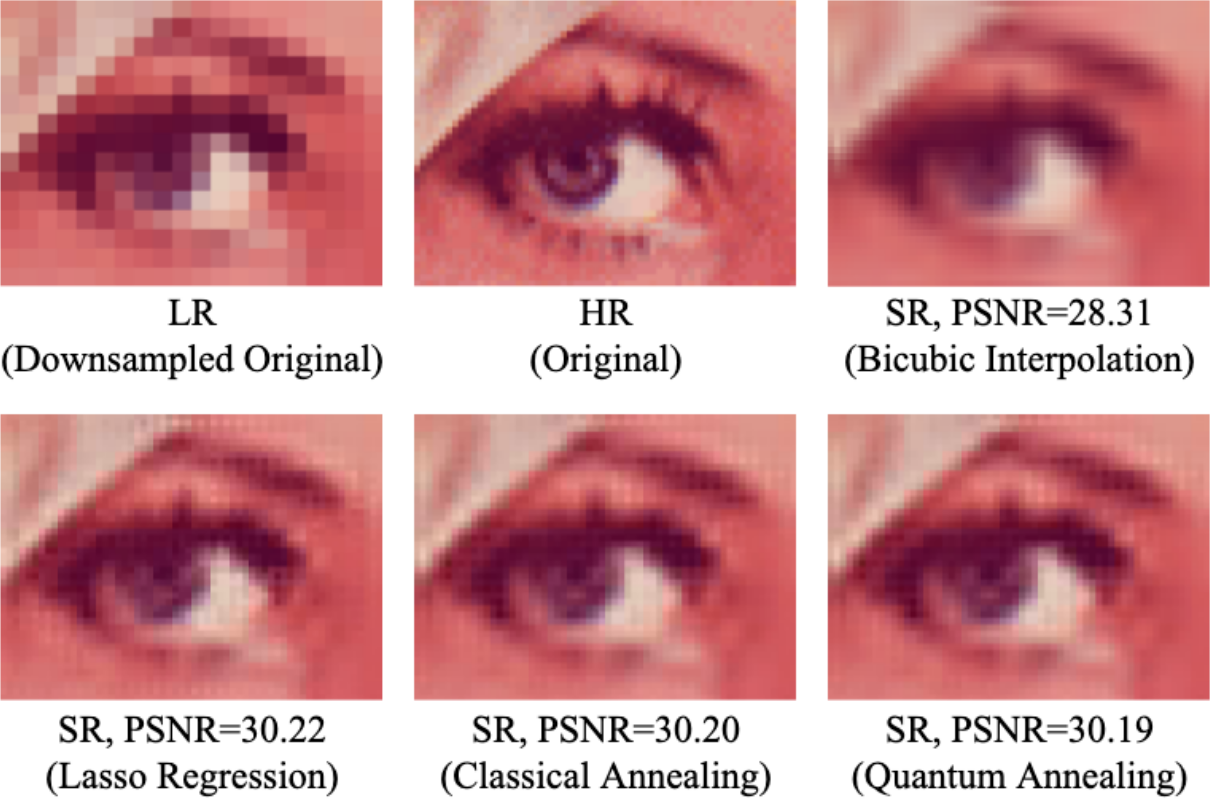}
	\end{center}
	\caption{Quantum results of single-image super-resolution in comparison to classical baselines (bi-cubic interpolation, lasso regression and simulated annealing)~\cite[IEEE \textcopyright 2025]{choong2023quantum}.} 
	\label{fig:QSISR} 
\end{figure}

\parahead{Segmentation and grouping} 
Segmentation refers to the task of classifying entities in an input signal, such as pixels in an image or points in a point cloud, into semantically meaningful categories or groups~\cite{wang2022comprehensive}. 

A method for remote sensing image segmentation was proposed by Delilbasic~\cite{delilbasic2023single}, that reformulates the classical Crammer-Singer multiclass support vector machine (SVM) ~\cite{crammer2001algorithmic} into a QUBO, enabling optimisation on a quantum annealer.
Building on this, Zardini \textit{et al.}~\cite{zardini2024local} introduced locality techniques~\cite{blanzieri2006adaptive,segata2010fast} to iteratively solve subproblems, mitigating scalability limitations. 
Beyond SVM-based approaches, Venkatesh \textit{et al.}~\cite{venkatesh2023q} proposed an unsupervised image segmentation method leveraging quantum annealers. Their method represents an image as a fully connected grid, where edge weights encode pixel similarity.
Naturally, the segmentation problem can be formulated as finding the cut of the grid such that the aggregated weights of the edges removed are minimised (\ie~min-cut problem), which can be run on annealers according to a QUBO reformulation.
Beyond static scene segmentation, Arrigoni \textit{et al.}~\cite{Arrigoni2022} introduced \emph{QuMoSeg}, a QUBO-based approach for \emph{motion segmentation}, which classifies points across multiple images into independent motion groups. Unlike traditional segmentation methods, motion segmentation relies on identifying and grouping coherent motion patterns rather than spatial structures.
By leveraging compact partial segmentation metrics, \emph{QuMoSeg} enables the processing of problems with up to 120 points on real quantum annealers, demonstrating its scalability within current hardware constraints. 

\parahead{Stereo matching}
Stereo matching is a technique for extracting depth information from visual observation by comparing two images taken from slightly different viewpoints, typically known spatial camera locations in a stereo pair~\cite{zhou2020review}. 
The goal is to compute disparities that indicate relative object distances with respect to the camera baseline and focal length by finding corresponding point pairs in the two images (constrained by epipolar geometry). 
Such disparity maps are widely used in applications like 3D reconstruction and autonomous navigation.

The first formulation of stereo matching as a QUBO problem was presented in Ref.~\cite{cruz2018qubo}, where the corresponding min-flow/max-cut formulation is encoded as a degree-three pseudo-Boolean problem. 
This is subsequently reduced to a quadratic form solvable on a quantum annealer using a quadratization method by Ishikawa~\cite{ishikawa2010transformation}, which introduces additional ancilla variables; see discussion also in Sec.~\ref{sec:ProblMapMethAnnealers}. 
Heidari \textit{et al.}~\cite{Heidari2021} refined this approach by reducing the qubit requirements: Their model uses linear regularisers and allows for obtaining optimal solutions in polynomial time. 
Although the formulation is efficiently solvable classically, their work suggests that alternative smoothness assumption functions can lead to NP-hard extensions suitable for quantum optimisation. 
Their method successfully computes disparity maps for stereo pairs from the Middlebury dataset~\cite{Scharstein2001}, with image resolutions of \(383{\times}434\) pixels and up to 39 disparity levels; see Fig.~\ref{fig:Q_stereo} for qualitative results. 

\begin{figure}[t]
\begin{center}
	\includegraphics[width=.98\linewidth]{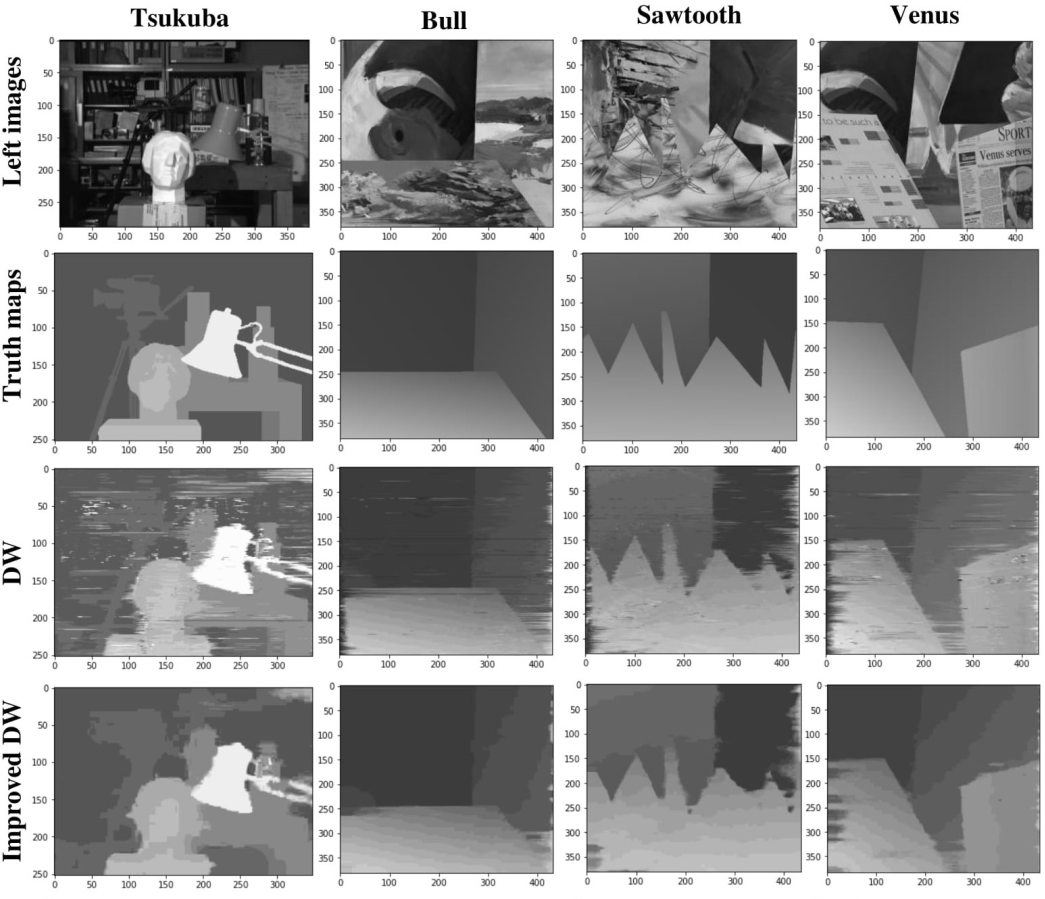}
\end{center}
\caption{
	Stereo matching by Heidari \textit{et al.}~\cite[IEEE \textcopyright 2025]{Heidari2021} on D-Wave quantum annealer (``DW'' and ``Improved DW''). 
	Colour scheme: The brighter, the higher the disparity at that pixel. 
}
\label{fig:Q_stereo}
\end{figure}

Several follow-up publications have explored NP-hard extensions of the stereo matching~\cite{heidari2022equivalent,heidari2023hybrid,heidari2024quantum,braunstein2024quantum}. 
Heidari \textit{et al.}~\cite{heidari2022equivalent} investigated the minimum multi-way cut problem as a model for stereo matching, contributing a new combinatorial perspective to the formulation. 
Later, Ref.~\cite{heidari2023hybrid} introduced a hybrid quantum-classical method that enabled the solution of significantly larger instances on the D-Wave quantum annealer (see also Ref.~\cite{heidari2024quantumThesis} for extended discussion).  
Braunstein \textit{et al.}~\cite{braunstein2024quantum} proposed a formulation based on Markov random fields and conducted a comprehensive comparison with the hybrid approach of Heidari \textit{et al.}~\cite{heidari2023hybrid}, highlighting strengths and limitations across various problem settings. 

\parahead{AQC with and for neural networks} 
An emerging research direction explores the integration of NNs with AQC to enhance learning and optimisation. 

One notable approach is \textit{QuAnt} by Seelbach Benkner \textit{et al.}~\cite{seelbach2022quant}, which leverages a neural network to predict QUBO coupling matrices, thereby eliminating the need for handcrafted QUBO formulations for different problem instances. 
Instead, a trained network infers the required coupling matrices for various problem classes. 
In the training phase, the QuAnt network learns to generate a QUBO matrix \( A \) from a given input \( P \), representing the problem instance; supervision is provided by comparing the solution \( x^* \), obtained from minimising the QUBO, with the ground-truth binary solution. 
A contrastive loss function guides the network to produce QUBO matrices that yield correct or near-optimal solutions under quantum annealing. 
QuAnt allows for a more flexible and automated approach to quantum optimization.
The effectiveness of this method has been demonstrated in tasks such as graph matching and point set registration, as illustrated in Fig.~\ref{fig:qcvQuAnt}. 

Another promising application of AQC lies in training binary neural networks~(BNNs), where weights are constrained to binary values.
Training such networks inherently constitutes a combinatorial optimisation problem, which is inefficient to solve classically. 
Sasdelli \textit{et al.}~\cite{sasdelli2021quantum} propose encoding BNN training as a QUBO problem and compare different solvers, including simulated annealing, quantum annealing and classical branch-and-bound methods. 
Recently, Krahn \textit{et al.}~\cite{Krahn2024} developed a novel AQC-powered optimiser to efficiently train general BNNs as well as binary graph networks in a layer-wise manner. 
Beyond binary networks, AQC has also been explored for training small-scale neural networks with real-valued weights. 
Yang \textit{et al.}~\cite{Fengyi2023arXiv} presented an approach to train a multilayer perceptron~(MLP) with a single hidden layer and sigmoid activation using quantum annealers, leveraging the equivalence between MLPs and energy-based models, such as variants of restricted Boltzmann machines with maximal likelihood objective. 
Similarly, Abel \textit{et al.}~\cite{Abel2022} proposed a quantum annealing-based framework for NN training, encoding network parameters as binary strings, approximating activation functions using polynomials and reducing higher-order terms to quadratic forms. 
Their method ensures global optimisation and achieves convergence in a single annealing step, significantly reducing training time compared to classical iterative approaches.
While we are not aware of any methods for training general neural architectures on AQCs, this remains an exciting direction for future research. 

\begin{figure}[t]
\begin{center}
\includegraphics[width=\linewidth]{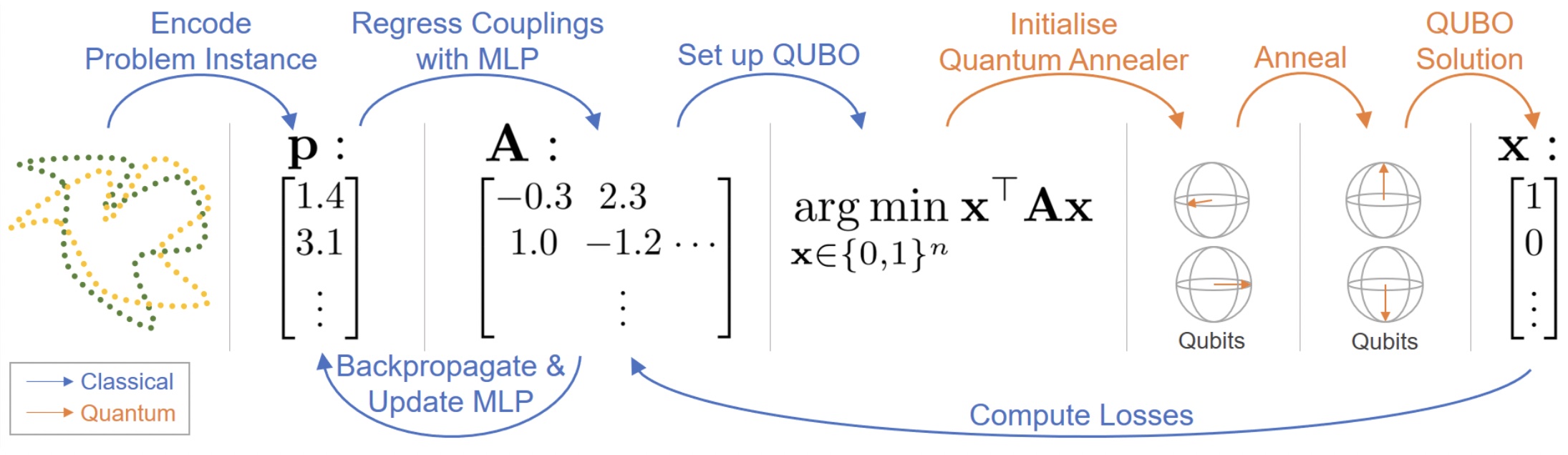}
\end{center}
\caption{QuAnt for learning QUBO coupling matrices~\cite[courtesy of the authors]{seelbach2022quant}. 
A network is trained to predict coupling matrices for given ground-truth solution vectors on sample problems.
} 
\label{fig:qcvQuAnt}
\end{figure}

\subsection{Methods Relying on Gate-based QC}\label{ssec:MethodsGateBased} 
\begin{table*}
\centering
\caption{Overview of gate-based methods for CV problems.
We report problem types, sizes and number of qubits tested if explicitly mentioned in the papers. 
Only the methods~\cite{Chin_2020_ACCV,huang2021experimental,silver2023mosaiq,kolle2024quantum}
were evaluated on real quantum devices.\\
Keys: 
\begin{tabular}[t]{lllll}
	``BIC'': & binary image classification & ``PCAE'': & point clouds auto-encoder\\
	``MCIC'': & multi-class image classification & ``IG'': & image generation \\
	``MCPCC'': & multi-class point cloud classification &``DDM'': & denoising diffusion model \\
	``RF'': & robust fitting  &``BA'': & bundle adjustment \\
	``prep.'': & classical preprocessing & ``NF/INR'': & neural fields/implicit neural representation. & 
\end{tabular}
}
\begin{tabular}{ccccccc}\hline
Method & Problem         & Input type/Datasets & Problem size& \# Qubits & prep. & Fig. \\\hline\hline
Hur \textit{et al.} QMI'22~\cite{hur2022quantum}&BIC& (Fashion) MNIST~\cite{FashionMNIST2017} & $\leq$ length-$32$ feature vectors& 8&\checkmark &\\
QDCNN, IOP'20~\cite{li2020quantum_deep}&MCIC& MNIST, GTSRB~\cite{Houben2013} &$\leq$ 32x32 images & - &&\\
sQCNN-3D, Elsevier'23~\cite{baek20223d}&MCPCC& ModelNet, ShapeNet~\cite{Wu2015}& $\leq$ 32x32x32 voxel grids& 4 &\checkmark&Fig. \ref{fig:sqcnn}\\
HQNN-Parallel, IOP'24~\cite{senokosov2024quantum}&MCIC& (Medical) MNIST, CIFAR~\cite{Krizhevsky2012}& $\leq$ \{64, 28, 32\}$^2$ images& 5 &\checkmark&Fig. \ref{fig:hqnn}\\
ATP, CVPR'25~\cite{afane2025atp}&MCIC& (Fashion) MNIST~\cite{FashionMNIST2017}, CIFAR~\cite{Krizhevsky2012}& $\leq$ \{64, 28, 32\}$^2$ images& - &\checkmark& \\\hline
Chin \textit{et al.}, ACCV'20~\cite{Chin_2020_ACCV}&RF& data points & - & - & &\\
Yang \textit{et al.}, ECCV'24~\cite{yang2024robust}&RF& data points & $\leq$4 points& 19&&\\\hline
3D-QAE, BMVC'23.~\cite{Rathi2023}&PCAE& point clouds & $\leq$16 points &6 &&Fig. \ref{fig:3DQAE}\\\hline
MosaiQ, ICCV'23.~\cite{silver2023mosaiq}&IG& (Fashion) MNIST& $\leq$ length-$10$ noise vectors & 5 &&\\
Huang \textit{et al.}, APS'23~\cite{huang2021experimental} &IG& handwritten digits & $\leq$length-$32$ latent vectors & 6&&\\
Kolle \textit{et al.}, ArXiv'24~\cite{kolle2024quantum}&DDF& (Fashion) MNIST~\cite{FashionMNIST2017}, CIFAR~\cite{Krizhevsky2012}& $\leq$ 32x32 images & 7 &&\\\hline
Piatkowski, ArXiv'22~\cite{piatkowski2022towards}&BA& sets of images & $\leq$32x32 imape patches & - &&\\
\hline
QIREN, ICML'24\cite{zhao2024quantum}&NF/INR&Coordinates & $\leq$64x64 images &  6 &&\\
QVF, ArXiv'25~\cite{Wang2025QVFs}&NF/INR& Coordinates & $\leq$ images, 3D shapes&6& &Fig. \ref{fig:chairs}\\
\hline
\end{tabular} 
\label{tab:papers_gate_based} 
\end{table*}

We next describe QeCV methods using the gate-based quantum computational paradigm, with a high-level overview provided in Tab.~\ref{tab:papers_gate_based}. 
Note that many of the works discussed next do not provide experiments on real QCs (but often using simulators thereof), due to their unavailability or assumptions of more powerful QCs that do not exist as of 2025. 
However, these studies provide strong empirical evidence for the feasibility of gate-based QeCV, making it valuable to first investigate these methods on a simulator before transitioning to real quantum devices. 

\parahead{Image segmentation through quantum kernels} 
Kernel methods are widely used in machine learning to capture the complex patterns in data by constructing a Gram matrix $K(x_i, x_j)$ in a pre-defined feature space, which encodes pairwise relationships and facilitates generalisation to unseen samples.
Similarly, such Gram matrix $K(x_i, x_j)$ can, when executed on gate-based quantum computers, be defined using quantum state overlaps, where each data point $x_i$ is encoded in a quantum state $\ket{\psi(x_i)}$ and similarity is measured through $ |\langle \psi(x_i) | \psi(x_j) \rangle|^2 $. 
Practically, quantum kernel methods, in the computer vision domain, have been experimented in areas such as image classification and segmentation tasks. 
Miroszewsk \textit{et al.}~\cite{miroszewski2023detecting}, under the framework of support vector machine~(SVM), propose to replace classical kernels with quantum kernels, which can manipulate data in an exponentially-large feature space, w.r.t.~number of qubits, to segment cloudy areas for satellite images.

\parahead{Classification} 
Image or 3D classification is a classical task in visual recognition, aiming to assign predefined labels to entire objects, such as images~\cite{chen2021review} or point clouds~\cite{zhang2023deep}. 
It relies on comparing distinguishing object features, which are typically learned by neural networks.

Cong \textit{et al.}~\cite{cong2019quantum}, inspired by the design of classical convolution networks, proposed a translation-equivariant circuit-based quantum alternative named \emph{Quantum Convolution Neural Network}~(QCNN),  which can be used for extracting high-level features. 
The QCNN architecture is visualised in Fig.~\ref{fig:qcnn}. 
The quantum architecture replaces classical convolutional layers with quantum convolutional ones, \ie~quantum transformations that express shared unitary transformations across different regions of the circuits to mimic translational equivariance in extracted quantum features. 
For the dimensionality reduction---akin to classical pooling---quantum information is compressed by discarding ancilla qubits or applying local measurements, introducing non-linear effects. 

\begin{figure}[t]
\centering
\includegraphics[width=\linewidth, trim={0cm 0cm 0cm 0cm}, clip]{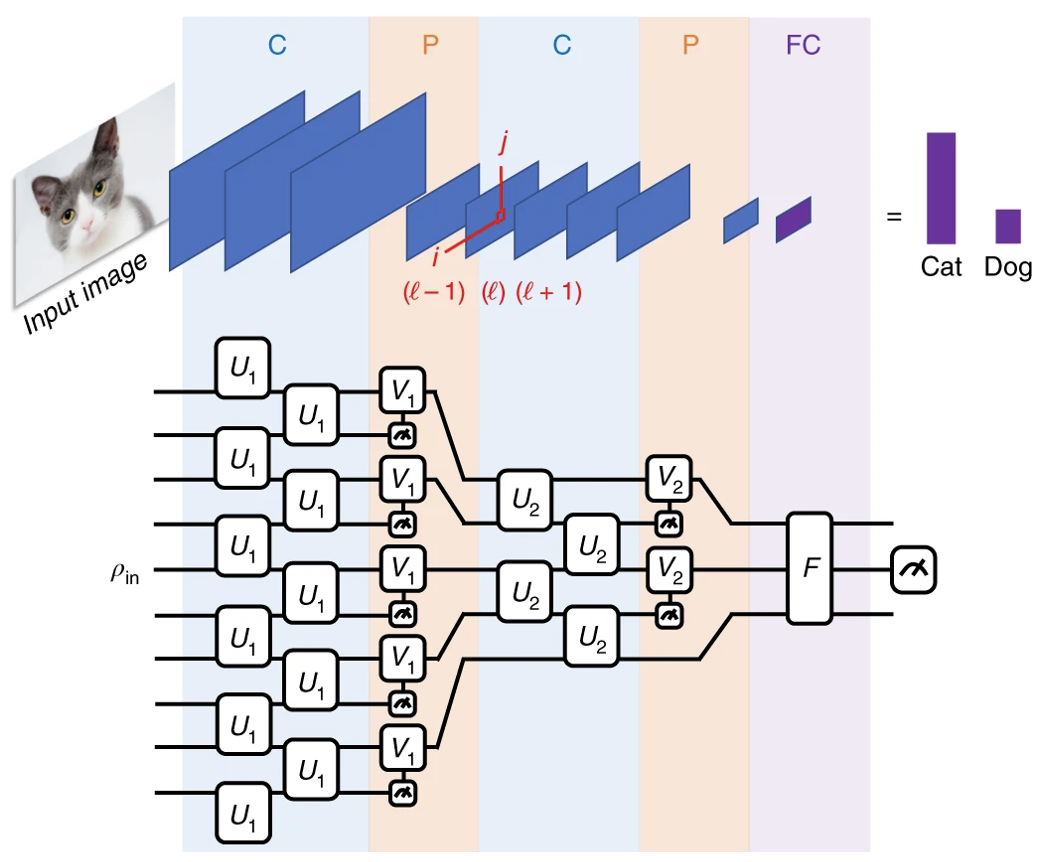}
\caption{A quantum convolutional neural network
~\cite[Springer Nature \textcopyright 2025]{cong2019quantum}.
The circuit uses shared-parameter gates for feature extraction and intermediate measurements for pooling purposes.}
\label{fig:qcnn}
\end{figure}

Hur \textit{et al.}~\cite{hur2022quantum} applied QCNN to MNIST and Fashion MNIST image classification, exploring different data encoding schemes and PQC architectures for convolutional layers. 
They demonstrated that QCNNs can outperform classical models with a comparable number of parameters. 
Li \textit{et al.}~\cite{li2020quantum_deep} proposed a deep QCNN (coined QDCNN) for image recognition, applied to recognise real-world traffic signs on the GTSRB dataset. 
They replaced the convolutional layers of Cong \textit{et al.}~\cite{cong2019quantum} and instead used QRAM to prepare a convolution kernel and perform convolution operations with quantum multipliers. 
A phase estimation subroutine converts amplitude-encoded results to basis-encoded results, allowing to apply, as discussed in basis-encoding in Sec.~\ref{sec:ProblMapMethGateBased}, arbitrary functions to the features. 
Non-linear activation functions, in particular, are approximated using Taylor series expansions with efficient quantum arithmetic.
Dimensionality reduction is achieved by increasing the convolutional stride, omitting traditional pooling layers.

\begin{figure}[t] 
\centering 
\includegraphics[width=\linewidth, trim={0cm 0cm 0cm 0cm}, clip]{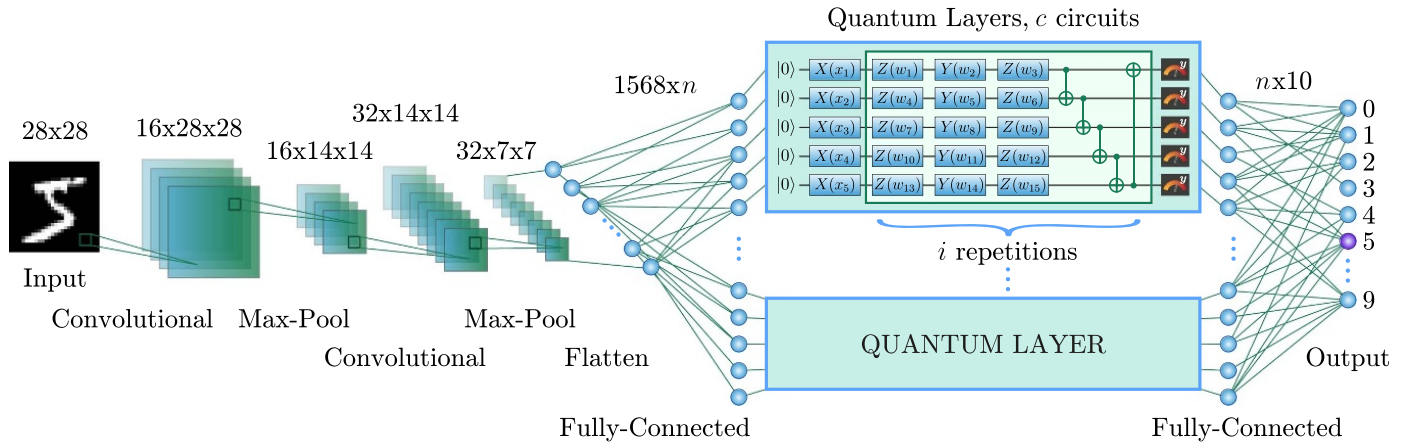} 
\caption{HQNN-Parallel approach for image classification~\cite[reproduced under the CC BY 4.0 license]{senokosov2024quantum}. 
Several PQCs with distinct parameters are added to a classical network to process feature patches in parallel.\vspace{-3mm}} 
\label{fig:hqnn} 
\end{figure} 
Senokosov \textit{et al.}~\cite{senokosov2024quantum} proposed a hybrid quantum-classical model for image classification. 
They integrated quantum circuits into a classical NN to process intermediate feature vectors and further extract features. 
To handle the high dimensionality of the input features on small-scale quantum devices, they used multiple PQCs, each processing only a patch of the input. 
They experimented with MNIST, Medical MNIST and CIFAR datasets. 
Their model is visualised in Fig.~\ref{fig:hqnn}. 
Afane et al.~\cite{afane2025atp} proposed an optimised quantum state preparation technique for image classification on gate-based computers. 
The introduced Adaptive Threshold Pruning (ATP) method optimises a learnable threshold to zero out low-intensity image pixels, thereby reducing the depth of the state preparation subcircuit. 
By lowering the entanglement entropy of the PQCs, this bi-level optimisation improves classification accuracy while enhancing robustness by pruning irrelevant features. 
Crucially, ATP focuses on a data-driven quantum state preparation, allowing it to be modularly integrated with any downstream PQC architecture.

As an extension of quantum-classical classification to 3D shapes, Baek \textit{et al.}~\cite{baek20223d} introduced scalable QCNNs for classifying 3D point clouds; see the architecture in Fig.~\ref{fig:sqcnn}. 
To efficiently handle denser 3D point cloud data, they voxelised the 3D space, selectively extracting features from a controlled number of voxels. 
They also incorporated a fidelity penalisation term to promote feature diversity following quantum encoding, which are subsequently utilised for classification. 

\begin{figure}[t]
\centering
\includegraphics[width=\linewidth, trim={0cm 0cm 0cm 0cm}, clip]{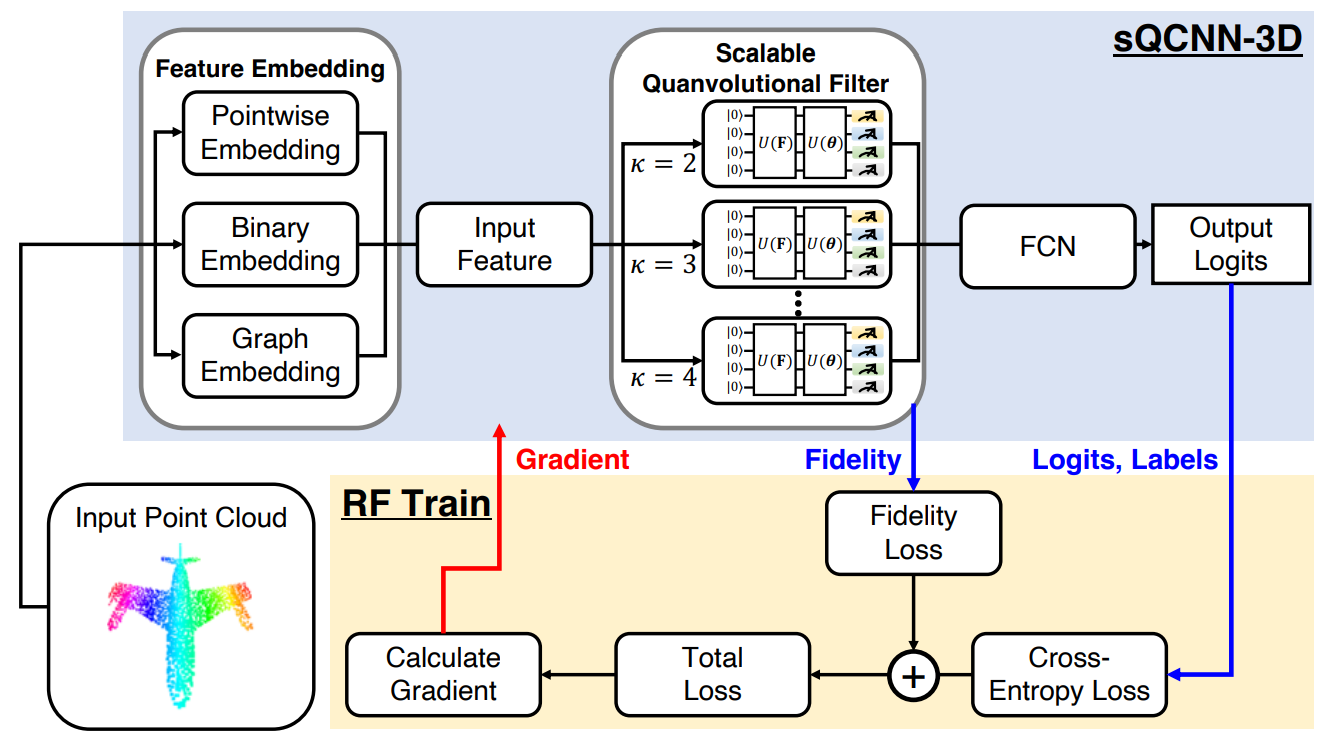}
\caption{SQCNN for point cloud classification~\cite[Elsevier \textcopyright 2025]{baek20223d}
PQC filters, like convolution kernels, extract features from input patches, with parameters shared across them for each PQC.}
\label{fig:sqcnn}
\end{figure}

\parahead{Model fitting} 
Chin \textit{et al.}~\cite{Chin_2020_ACCV} investigated the use of gate-based quantum computing for robust geometric curve fitting in computer vision. 
They proposed a probabilistically convergent classical algorithm for fitting geometric curves and integrated a quantum speedup using the Bernstein-Vazirani algorithm~\cite{BernsteinVazirani1997}. 
Their approach leverages quantum computing to efficiently compute Boolean influence measures that assess data point outlyingness. 
By incorporating quantum techniques, their method enhances the robustness and efficiency of geometric fitting, demonstrating the potential advantages of gate-based quantum computing for model fitting. 
However, this work remained purely theoretical, with an open gap that needed to be filled for practical application. 
This gap was later addressed by Yang \textit{et al.}~\cite{yang2024robust} for small-sized 1D linear regression. 
The authors explicitly derived a quantum oracle capable of computing the feasibility test needed for the Bernstein-Vazirani circuit. 
Their experiments confirmed the theoretical expectations derived by Chin \textit{et al.}~\cite{Chin_2020_ACCV}. 
We also refer to Sec.~\ref{ssec:MethodsAQC} for models fitting methods relying on AQC.

\parahead{Point set alignment} 
Noormandipour and Wang~\cite{NoormandipourWang2022} proposed Quantum Kernel Correlation (QKC), a method for applying a PQC to align two point sets. 
The algorithm is designed so that the output state of a PQC encodes the correct alignment rotation. 
Following Golyanik and Theobalt~\cite{golyanik2020quantum}, QKC restricts the transformation $\mathcal{T}$ to be a rotation and utilises a QAOA-like architecture to predict the rotation angle discretised into $2^n$ bins, where each binary vector corresponds to a unique bin. 
We also refer to Sec.~\ref{ssec:MethodsAQC} for point set alignment methods relying on AQC.

\parahead{Data autoencoding} 
Data autoencoding is a compression task aiming to encode a signal into a lower-dimensional one and then reconstruct it with minimal information loss~\cite{mienye2025deep}.
This process is particularly useful for dimensionality reduction, feature extraction and learning efficient codings of unlabeled data.

Rathi \textit{et al.}\cite{Rathi2023} introduced an autoencoder that employs parametrised quantum circuits in place of classical neural networks to encode and reconstruct sparse 3D point clouds, enabling execution on quantum hardware (see Fig.~\ref{fig:3DQAE}). While their approach has yet to match the performance of classical neural networks, it offers valuable insights into the effectiveness of quantum encoding schemes in parametrised quantum circuits. 

\parahead{Generative networks}
Generative models aim to sample new data from learnt complex data distributions across various modalities.
Modern approaches, such as variational autoencoders (VAEs)~\cite{kingma2013auto}, generative adversarial networks (GANs)~\cite{goodfellow2014generative}, normalizing flows~\cite{rezende2015variational} and diffusion-based models~\cite{rombach2022high,ho2020denoising}, have achieved significant success. 
However, as many learning models, these methods face long-standing challenges of scalability and computational efficiency, leaving room to explore quantum generative models.

In particular, Silver \textit{et al.}~\cite{silver2023mosaiq} and Huang \textit{et al.}~\cite{huang2021experimental} adopt a GAN-based generative framework, replacing the classical generator with a PQC while keeping the discriminator classical. Due to the constraints imposed by current quantum hardware and software, both approaches incorporate pre-processing techniques to reduce the dimensionality of the input data. Specifically, Huang \textit{et al.}~\cite{huang2021experimental} split the training data into image batches for quantum generation, whereas Silver \textit{et al.}~\cite{silver2023mosaiq} further refine this process using re-ordered principal component analysis (PCA) to project data into lower-dimensional representations with approximately uniform variance.
A different approach is proposed by K{\"o}lle \textit{et al.}~\cite{kolle2024quantum}, who introduce a quantum denoising diffusion network that replaces the classical U-Net-based denoising process with a quantum U-Net, where all operations adhere to quantum principles. 
Although their experiments remain limited in scale, a key advantage of their method is the ability to collapse the multi-step denoising diffusion process into a single step, effectively mapping noise directly to structured data. 
This reduction in computational complexity is attributed to the unitary property of quantum operations, where the product of arbitrary unitary matrices remains unitary. 
These developments underscore the potential of quantum generative models to overcome some of the computational bottlenecks faced by their classical counterparts. 

\parahead{Implicit representations} 
Implicit neural representations, or neural fields, use neural networks to encode signals such as images or 3D shapes, by mapping low-dimensional grid coordinates to the corresponding signal values (\textit{\eg~}~RGB pixel values or signed distances). 
Their accuracy improves when the model learns different signal frequencies, inspiring Fourier neural networks and sinusoidal activation functions. 

Zhao \etal~\cite{zhao2024quantum} proposed the quantum implicit representation network (QIREN, Fig.~\ref{fig:qiren}), a data-re-uploading-based quantum model that can learn the Fourier decomposition of signals. 
Data re-uploading~\cite{perez2020data} involves encoding inputs multiple times into the PQC; it is shown to effectively represent Fourier series~\cite{schuld2021effect}. 
QIREN demonstrated that---unlike classical methods where available frequencies grow linearly with problem size---data re-uploading enables exponentially growing frequencies with PQC size under optimal conditions, allowing for fewer parameters and a more precise representation. 

\begin{figure}[t]
\centering
\includegraphics[width=\linewidth, trim={0cm 0cm 0cm 0cm}, clip]{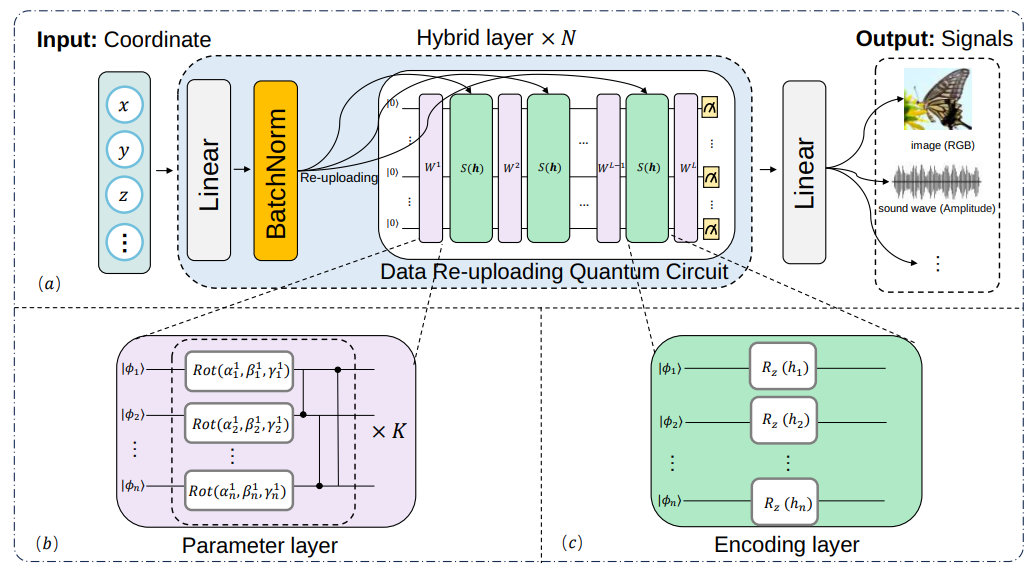}
\caption{Quantum Implicit Representation Network (QIREN)~\cite[courtesy of the authors]{zhao2024quantum} leveraging data re-uploading of the input to the PQC to represent the signal in its Fourier decomposition. 
}
\label{fig:qiren}
\end{figure}

Next, Wang et al.~\cite{Wang2025QVFs} 
proposed Quantum Visual Fields~(QVF), which learn encoding through learnable energy embeddings into quantum states. 
Unlike prior approaches~\cite{zhao2024quantum,baek20223d,senokosov2024quantum} where the quantum circuit is placed in between classical neural network layers as a classical enhancer, QVFs preserve the inherent probabilistic quantum circuit properties and do not use post-processing classical layers. 
The method was shown to outperform QIREN and even widely-used classical MLP baselines in the expressive power (representing high-frequency details with fewer parameters) and scalability across different signal resolutions. 
The QVF framework also successfully generalises to 3D shapes and represents a promising landmark in the development of practical quantum and quantum-inspired machine learning applications in visual computing, see Fig.~\ref{fig:chairs}. 
\parahead{Bundle adjustment} 
Bundle adjustment is a critical optimisation process that refines 3D structures and camera parameters by minimising projection errors between 3D points and their corresponding 2D images~\cite{hartley2003multiple}.

Piatkowski et al.~\cite{piatkowski2022towards} investigate bundle adjustment in satellite imaging. 
They explore both AQC and gate-based quantum computing simultaneously for solving clustering and feature-matching problems within this context.  
For the gate-based quantum approach, they employ quantum kernel methods and the variational quantum eigensolver~(VQE) to enhance feature matching. 
In contrast, under the AQC paradigm, they formulate a QUBO objective, which is then executed on a quantum annealer. 
This dual exploration provides insights into how different quantum computational models can contribute to optimisation challenges in satellite image processing. 

\subsection{High-Level Observations and Summary} 
\label{sec:others_and_discussion}

After having discussed both gate-based and annealing-based quantum computing paradigms, we now analyse their respective characteristics.

Quantum annealing has been applied to energy-based optimisation tasks, with all reviewed methods formulating computer vision problems in the QUBO form to ensure compatibility with the hardware’s coupling constraints, that is, the limited connectivity and admissible problem structure imposed by current quantum annealers. 
Almost all methods---except a few such as point set alignment~\cite{golyanik2020quantum}, Q-Match~\cite{SeelbachBenkner2021}, or image segmentation~\cite{delilbasic2023single}---faced the challenge of efficiently handling constraints in their optimisation, which was largely circumvented using penalty terms. 
Some approaches also encountered higher-order or non-linear objectives over binary variables, addressed through quadratisation with ancilla variables~\cite{cruz2018qubo} or via recursive decomposition~\cite{Meli_2022_CVPR, meli2025qucoop}. 
While some approaches were one-sweep~\cite{golyanik2020quantum,SeelbachBenkner2020,LiGhosh2020,zaech2022adiabatic,Bauckhage2018AdiabaticQC,zaech2024probabilistic}, many adopted iterative refinement of the solution ~\cite{Meli_2022_CVPR,meli2025qucoop,SeelbachBenkner2021,Farina2023,
pandey2025outlier,braunstein2024quantum}. 
Lastly, there have been attempts to apply quantum annealing to train classical neural networks~\cite{sasdelli2021quantum,Krahn2024,Fengyi2023arXiv,Abel2022}, but such hybrid paradigms remain challenged by resource overhead and limited problem size compatibility (see Sec.~\ref{ssec:DWave_Architecture}). 
\begin{figure}[t!]
\centering
\includegraphics[width = 0.96\linewidth]{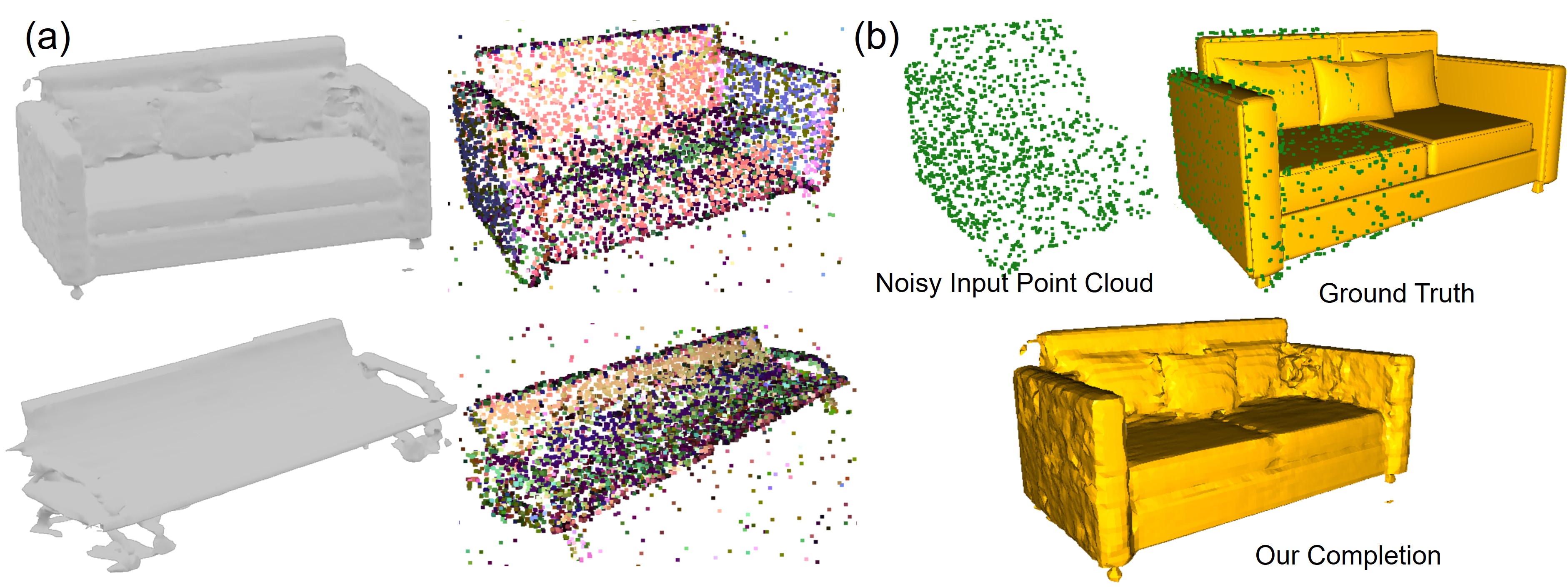} 
\caption{(a) 3D shape reconstruction using QVF~\cite[courtesy of the authors]{Wang2025QVFs}, with marching cubes applied to the inferred fields and latent codes. 
(b) Shape completion from partial inputs with QVF.}
\label{fig:chairs}
\end{figure} 

Gate-based models offer a flexible and more general alternative to quantum annealing, enabling more adaptable circuit designs. 
However, only very few analytic (handcrafted) circuits (\eg~for robust fitting~\cite{Chin_2020_ACCV,yang2024robust}) have been proposed, since their construction requires a strong inductive bias from the underlying problem. 
A growing trend favours PQCs for learning tasks, reflecting a shift toward trainable quantum architectures that can tackle various computer vision problems, including classification~\cite{hur2022quantum,li2020quantum_deep,senokosov2024quantum,baek20223d}, regression~\cite{NoormandipourWang2022,zhao2024quantum,Wang2025QVFs} and generative modelling~\cite{silver2023mosaiq,huang2021experimental,kolle2024quantum}. 
One notable advantage was their integration with kernel methods~\cite{miroszewski2023detecting}, where unitary evolution naturally realises feature space embeddings. 
Parameterised quantum circuits (PQCs) offer promising expressivity, and their tunable architecture makes them well-suited to data-driven quantum learning. However, they typically require more mature—and technically demanding—hardware than annealing-based approaches.
From the hardware perspective, despite steady progress in recent years, two fundamental challenges remain unsolved: (1) error correction and (2) scalability. 
As a result, we still lack a strictly fault-tolerant quantum computer capable of performing large-scale computations with manageable noise levels. 
On the software side, manipulating complex quantum states, especially in gate-based models, remains highly non-trivial. 
Unlike classical neural networks, designing quantum circuits for information processing demands deep interdisciplinary expertise, as computations must strictly adhere to the principles of quantum mechanics, such as unitary evolution. 
This poses significant challenges for researchers entering the field, yet makes QeCV a compelling area of exploration for the next decades. 
\section{Resources and Tools}\label{sec:resources_tools} 
This section summarises the available quantum hardware and quantum computing tools to get started in QeCV. 
Multiple quantum computers are currently accessible via the cloud. 
They differ in the number of qubits, the qubit technology, the qubit connectivity pattern and the low-level qubit 
characteristics. 
Note that the sheer number of qubits in isolation from other characteristics is not a meaningful metric. 
While the so-called ``quantum volume'' represents an attempt to introduce a single metric to express the overall performance in a standardised way~\cite{Cross2019}, we do not use it due to criticism and expected diminished relevance for future quantum systems~\cite{Baldwin2022}. 
Most available machines and those under development are gate-based; D-Wave is among the few companies providing quantum annealers. 
There are also several simulators\footnote{We allow ourselves a terminological simplification and use ``simulators'' collectively for simulators and emulators. The latter can take shortcuts in computing expected outputs of quantum computations.} of quantum hardware that simulate quantum computers with different degrees of fidelity (\textit{\eg~} supporting various levels of noise and decoherence). 
We next provide an overview of the available quantum hardware (Sec.~\ref{ssec:QC_Hardware}), Integrated Development Environments (IDEs) for programming and quantum computer simulators (Sec.~\ref{ssec:QC_simulators}) as of 2025. 
Finally, we summarise the learning materials relevant for QeCV and 
available online (Sec.~\ref{learning_quantum_materials}). 
\subsection{Quantum Computing Hardware}
\label{ssec:QC_Hardware}

Several companies and startups are actively developing quantum computing technologies, including universal fault-tolerant gate-based quantum computers and quantum annealers. 
These systems are built on diverse hardware architectures and can be accessed and programmed through various Software Development Kits (SDKs). 
In this section, we highlight the most advanced quantum computing platforms for which publicly available information exists and summarise some widely used SDKs in Table~\ref{tab:IDEs_SDKs}. 
\parahead{Quantum annealers} 
D-Wave is the first company to commercialise quantum annealing hardware at a significant scale. 
Their latest flagship processor, the Advantage 2, features ${\sim}4400$ qubits with a 20-way qubit connectivity (see Sec.~\ref{ssec:DWave_Architecture} and Fig.~\ref{fig:D-Wave_architectures}). 
This aspect and the qubit connectivity pattern make them useful for moderately large computer vision problems, as highlighted in this survey. 
Currently, other companies and institutions such as Qilimanjaro, Los Alamos National Laboratory~(LANL) and Tokyo Institute of Technology~(TIT) also build their own custom annealers with the scale, however, not comparable to D-Wave~\cite{shedule}\footnote{There also exist so-called digital annealers like Fujitsu that are, however, classical machines for solving optimisation problems that quantum annealers are targeting.}. 
In parallel, alternative quantum computing approaches have emerged, such as Coherent Ising Machines (CIMs)\cite{wang2013coherent}. 
QBoson, for example, has developed a CIM-based system with its latest product, the 550$W$~\cite{QBoson}. 
Unlike quantum annealers relying on superconducting qubits and requiring cooling to near absolute zero for superconductivity, CIMs operate using laser systems and optical qubits enabling computation at room temperature. 
While D-Wave remains the dominant player in quantum annealing, emerging technologies such as CIMs from QBoson present promising alternative paradigms for efficiently solving QUBO and Ising problems.

\parahead{Gate-based quantum hardware} 
Gate-based quantum computers can be realised using different physical platforms, often leveraging distinct types of physical particles. 
Among these, two prominent categories are fermionic gate-based and photonic gate-based quantum computers. 
Fermionic gate-based quantum computers, which typically use superconducting qubits, have seen significant advancements, with IBM being one of the most active players in both hardware development and software ecosystem expansion. 
IBM began providing access to its experimental quantum chips as early as 2019, initially with systems featuring as few as five qubits. 
Since then, the company expanded its offerings and launched Falcon (27 qubits, launched in 2020), Eagle (127 qubits, launched in 2021), Osprey (433 qubits, launched in 2022), Heron (156 qubits, launched 2024) and Condor (1121 qubits, launched in 2024); several of these quantum systems are available via the cloud. 
Google is also actively working on gate-based quantum hardware. 
In 2018, it introduced a 53-qubit quantum computing chip, Sycamore, which was used to experimentally demonstrate that certain problems, such as random circuit sampling, can be performed more efficiently on Sycamore than on the best classical supercomputer, therefore reaching a significant milestone.
Another recent milestone quantum chip developed by Google, Willow, launched in 2024, demonstrates, for the first time, the implementation of quantum error correction on the quantum chip with physical qubits arranged in a surface code layout. 
Notably, Willow is the first experimental chip to demonstrate quantum error correction below the surface threshold, as highlighted in Ref.~\cite{Willow2024}, marking a significant step forward in making quantum computing more reliable.
Differently from many other companies, Google does not provide public direct access to their experimental machines, which are primarily used for internal research and collaboration with other academic and industry institutes.
IBM and Google regularly release and update their long-term roadmaps for future generations of quantum hardware~\cite{IBMroadmap2025, Googleroadmap2025}. 
Other companies that are predominantly working on gate-based quantum hardware each with their own unique type of experiments are Rigetti, IonQ, Quantinuum, AQT, eleQtron, Microsoft and IQM. 
While Rigetti utilises superconducting qubits in their Aspen series, IonQ, Quantinuum, AQT and eleQtron rely on trapped-ion qubits. 
Quantinuum is a merger of Honeywell Quantum Solutions and Cambridge Quantum, working on quantum processors based on trapped-ion technology. 
Although AQT and eleQtron both work with ions they manipulate them in different ways. While eleqtron makes use of microwaves for the qubit control, AQT's devices utilize laser beams~\cite{lippert2022perspectives}.
Next, IQM focuses on superconducting qubit technology and Microsoft recently released the first quantum processor based on topological qubits~\cite{Aghaee2025}. 
Quantum computing with photonic gate-based systems is sometimes referred to as ``continuous variable quantum computing" because photons can theoretically exist in an infinite number of modes. 
Currently, Xanadu is considered the leader in developing large-scale, fault-tolerant quantum computers. 
Through cloud access, users can interact with their Borealis photonic quantum processor, with its architecture designed for Gaussian boson sampling. 
As of the time the survey is written, Borealis can process 216 squeezed-state modes of light (note that in photonic computers, modes can be considered equivalent to qubits). 
Other companies, such as PsiQuantum, QuiX Quantum, LightMatter and ORCA Computing, are also exploring photonic gate-based quantum processors. 
However, they do not currently provide direct access to their quantum chips. 
\subsection{Quantum Computing Frameworks and Simulators} 
\label{ssec:QC_simulators} 

\begin{table}[]
\centering
\caption{Available adiabatic (AQC) and gate-based (GQC) resources for developing QeCV algorithms. 
All SDKs are supported in the Python programming language.
Qubits: Number of (physical) qubits supported.
} 
\begin{tabular}{cccccccc}\hline
SDK & Hardware         & Paradigm  & Qubits \\\hline\hline
Ocean~\cite{dwavesysDWaveOcean} & DWave   & AQC   & {$\sim$}5000    \\
Cirq~\cite{Cirq} & NISQ Circuit   & GQC & {$\sim$}100    \\
Fermilib~\cite{Fermilib} & Simulation    &  GQC   & {$\sim$}10 \\
Qiskit~\cite{qiskit2024} & IBM    &   GQC   & {$\sim$}1000    \\
PennyLane~\cite{PennyLane} & Simulation    &  GQC   & {$\sim$}20  \\
CuQuantum~\cite{10313722} & Simulation    &  GQC   & {$\sim$}40  \\\hline
\end{tabular}
\label{tab:IDEs_SDKs} 
\end{table} 

We next introduce frameworks, \ie~IDEs and SDKs, for programming quantum computers (both gate-based and annealing machines) along with different 
classical quantum hardware simulators. 
Some leading and widely-used platforms are listed in Table \ref{tab:IDEs_SDKs}. 
For each quantum computing paradigm, we discuss frameworks for simulation on classical machines and then present frameworks that can execute the algorithms on real quantum hardware. 
Several Japanese companies introduced quantum annealing simulators and solvers inspired by them. 
Fujitsu offers a digital annealer~\cite{Bosse2020} mimicking quantum annealing using classical hardware, and provides a scalable alternative for solving Ising-like combinatorial optimisation tasks.
Hitachi has also explored quantum-inspired annealing with its CMOS-based Ising machines, leveraging classical simulations to validate their effectiveness~\cite{Yamaoka2015}.
Toshiba has developed a quantum annealing simulator that emulates quantum annealing using classical hardware~\cite{Goto2019}, applying principles from quantum mechanics and integrating hybrid optimisation models to enhance real-world optimisation tasks. 
To date, D-Wave's Ocean~\cite{dwavesysDWaveOcean} is the leading SDK usable for programming of real AQC hardware, \ie~D-Wave annealers. 
Regarding gate-based quantum computer simulators, several companies provide comprehensive frameworks that integrate simulation and development environments, optimised for execution on different hardware platforms. 
IBM offers Qiskit Aer~\cite{QiskitAer2025}, a high-performance simulator with noise modelling capabilities, running on classical CPUs and GPUs. It is fully integrated into the Qiskit software development kit (SDK). 
Google’s Cirq framework~\cite{Cirq} features qsim, a state-of-the-art wavefunction simulator written in C++ with OpenMP support, designed for large-scale simulations on multi-core CPUs and GPUs. 
Microsoft’s Quantum Development Kit (QDK) includes Q\# simulators for full-state and Toffoli simulations on classical CPUs, with cloud-based scalability. 
NVIDIA’s cuQuantum~\cite{10313722} is a GPU-accelerated quantum circuit simulator optimised for high-performance computing clusters; it supports efficient quantum circuit simulations through CUDA-Q/QX libraries. 
Intel provides the Intel Quantum Simulator (IQS), formerly known as qHiPSTER, which is optimised for multi-threaded CPU execution and designed to take advantage of Intel’s HPC architectures. 
Amazon Braket includes a local simulator that supports multiple backends, executed on Amazon Web Services (AWS) cloud infrastructure with CPU and GPU acceleration. 
Atos offers the Quantum Learning Machine (QLM), a dedicated classical supercomputer optimised for quantum simulation, supporting both CPU and FPGA acceleration. 
These platforms serve as all-in-one tools for quantum algorithm design, hardware validation and benchmarking, effectively integrating simulation capabilities with broader quantum computing SDKs and IDEs.
Beyond simulation, these SDKs and development environments enable programming and execution on real gate-based quantum computers. 
IBM provides Quantum Composer, a graphical interface for quantum algorithm design, as well as Qiskit Runtime, optimised for building, executing and optimising quantum workloads. 
These tools are based on Qiskit~\cite{qiskit2024}, an open-source framework supporting multiple gate-based quantum platforms, such as those from IBM, Rigetti, IonQ, IQM and Quantinuum, along with simulators like Aer and MQT DDSIM. 
Rigetti’s Aspen-M-3 quantum processor is accessible through its Quantum Cloud Services (QCS), designed for hybrid quantum-classical algorithms and programmed using Quil/pyQuil, Rigetti’s quantum instruction language. 
Google’s Quantum AI division provides Cirq~\cite{Cirq}, a Python-based framework for programming quantum circuits. Cirq includes built-in simulators for density matrix and wavefunction simulations and supports real quantum hardware. 
Additionally, Azure Quantum offers cloud-based access to a range of quantum hardware providers, including Rigetti, IonQ, Quantinuum and Pasqal, through the Q\# programming language. 
PennyLane~\cite{PennyLane}, developed by Xanadu, is an open-source framework supporting quantum machine learning and variational quantum algorithms. 
In addition to its photonic quantum hardware support (such as the Borealis processor~\cite{Madsen2022QuantumCA}), PennyLane also simulates gate-based quantum operations, making it a versatile tool for quantum software development. 
It integrates with a variety of quantum processors, including those from IBM, Rigetti and IonQ and supports local simulations using classical hardware, making it an excellent choice for both quantum circuit simulation and real quantum hardware execution. 
The GitHub repository~\cite{githubGitHubQosfawesomequantumsoftware} contains a list of references to different kinds of open-source software related to quantum computing. 
\subsection{Learning Materials} \label{learning_quantum_materials}
\begin{table}[]
\centering
\caption{
A selection of adiabatic (AQC) and gate-based (GQC) learning materials relevant to QeCV. 
The difficulty level of the provided material for a reader interested in QeCV ranges from ``$\star$'' (starter) to ``$\star \star$'' (intermediate/expert).} 
\begin{tabular}{ccc}
\hline
Material & Paradigm & Level\\
\hline \hline
Nielsen \& Chuang~\cite{nielsen2002quantum} & AQC/GQC & $\star$/$\star \star$ \\
Das \& Chakrabarti~\cite{das2005quantum} & AQC & $\star$ \\
IBM Quantum Learning~\cite{ibm_learning} & GQC &  $\star$ \\
Pennylane~\cite{pennylane_learning} & GQC &  $\star$ \\
Microsoft Azure Quantum~\cite{microsoft_learning} & GQC &  $\star$ \\
D-Wave Leap~\cite{dwave_learning} & AQC &  $\star$ \\
QCVML Workshop~\cite{QCVML_Workshop_2023} & AQC/GQC & $\star \star$ \\
\hline
\end{tabular}
\label{tab:learning_material}
\end{table}

As of 2025, quantum computing is not a well-known technology for most computer vision researchers and practitioners. 
Its specifics (\textit{\eg~}~due to terminology and specialised notations influenced by physics) and complexity can be a barrier to understanding and engaging in QeCV research. 
Also, our experience from successfully teaching QeCV to computer vision and visual computing students confirms these observations. 
This section provides further selected learning materials apart from this survey that could be helpful in QeCV. 
Some of them are associated with quantum IDEs and SDKs discussed in the last section; an overview is provided in Table \ref{tab:learning_material}. 

\parahead{Textbooks}
Understanding the principles of quantum mechanics is fundamental to grasping how quantum computers operate. 
The textbook by Nielsen and Chuang~\cite{nielsen2002quantum} is a foundational text in quantum information science covering a wide range of topics on gate-based quantum computing (\eg~~quantum algorithms, quantum teleportation, quantum cryptography and quantum error correction). 
Other books and articles by Das and Chakrabarti~\cite{das2005quantum}, McGeoch~\cite{McGeoch2014} and Albash and Lidar~\cite{AlbashLidar2018} are devoted to quantum annealing. 
They provide the fundamentals of quantum annealing and related
optimisation methods. 
Textbooks are excellent resources for those new to the principles of quantum computing, providing foundational knowledge and clear explanations. 
They also serve as excellent refreshers and deep dives for more experienced readers looking to enhance and expand their understanding.
\parahead{Programming} 
To get started in implementing quantum algorithms and interacting with quantum devices, some quantum computer manufacturers and software framework providers make coding examples and learning platforms available to the users; see Sec.~\ref{ssec:QC_simulators}. 
They provide problem modelling examples, notebooks and exercises for practising quantum programming. 
Notable ones include, but are not limited to, IBM Quantum Learning~\cite{ibm_learning}, Microsoft Azure Quantum~\cite{microsoft_learning} for gate-based quantum computing and D-Wave Leap~\cite{dwave_learning} for AQC. 
Pennylane~\cite{pennylane_learning} provides regularly replenished learning materials for quantum differentiable programming, especially for gate-based PQC. 
\parahead{The QCVML Workshop} 
The quantum computing mindset has a high potential to provide new perspectives leading to new insights, especially in the field in which many things are empirical, \textit{\ie~} computer vision. 
As the QeCV community grows, the demand for a dedicated venue at leading computer vision conferences increases. 
Thus, the first and second Quantum Computer Vision and Machine Learning (QCVML) workshops at a primary computer vision conference took place in 2023 and 2024 (Conference on Computer Vision and Pattern Recognition and European Conference on Computer Vision)~\cite{QCVML_Workshop_2023}. 
The list of covered topics includes quantum-inspired approaches, computer vision methods relying on adiabatic quantum computers (\textit{\eg~} quantum robust fitting), quantum neural networks and quantum hardware. 
While the first edition of QCVML focused on quantum annealing in computer vision, the second one was devoted to gate-based quantum hardware and perspectives of using it in computer vision.  
\section{Discussion}\label{sec:discussion} 
We reviewed QeCV approaches leveraging quantum annealers and gate-based quantum computers. 
Despite notable advances in recent years, the field remains in its early stages, with numerous open challenges and promising avenues for future research (Sec.~\ref{ssec:open_challenges}).
These will be explored prior to examining the specifics of publishing a QeCV paper (Sec.~\ref{ssec:publishing_reviewing}) and discussing its potential social implications (Sec.~\ref{ssec:potential_social_implicatins}).

\subsection{Open Challenges and Future Directions} 
\label{ssec:open_challenges} 
Open challenges pervade all aspects of current QeCV research, starting from the tools and available hardware through to data encoding (\textit{\eg~}~of images and 3D shapes) and problem mapping methodologies, PQC architecture designs, integration with existing classical tools and standardised benchmarking (\textit{\eg},~machine learning). 
Thus, quantifying the contribution of the quantum part in a hybrid approach is not always straightforward. 
In addition to the challenges associated with the new computational paradigm and problem formulations compatible with it, 
we next highlight several categories of open challenges (while not claiming exhaustiveness) that we believe will remain important for the field. 
\vspace{0.5mm}
\parahead{Quantum hardware and accessibility} 
There are multiple challenges associated with the accessibility of quantum hardware and its characteristics. 
QeCV is an applied field with high demands on the accessibility of quantum hardware or powerful classical machines capable of simulating it. 
While methods designed for quantum annealers can be tested on real quantum hardware (at least for small problems), methods for gate-based machines, as a rule, have to be tested on simulators; see discussion in Sec.~\ref{ssec:QC_simulators}.
Qubit noise and decoherence as well as error suppression in quantum hardware are some of the biggest challenges in physical quantum computer realisations. 
Another challenge is related to the connectivity patterns between the qubits, \ie~not all required connections between qubits upon the method design exist in hardware (which necessitates minor embedding and other forms of transpiling). 
Next, because a lot of developed quantum hardware is experimental, it is often accessible to a small group of people, usually closely working in or with the hardware development teams. 
On the one hand, access to and evaluation on real hardware is, hence, one of the biggest challenges recognised by the QeCV community. 
On the other hand, quantum method simulation is a valid and widely accepted evaluation methodology in the case of gate-based methods. 
There are reasons for cautious optimism due to the recent progress in quantum error correction~\cite{google_willow}. 
Moreover, the latest estimates regarding the required physical qubits for a fault-tolerant logical qubit has been corrected down to 100:1 (in contrast to 1000:1 or higher before 2025). 

Concerning quantum annealers, D-Wave machines still have substantial errors in the coupling parameters~\cite{error}. 
A relative error of $1\%$ in the biases, couplings, or even readouts is not uncommon. 
Care must be taken when qubit biases and couplings span several orders of magnitude, as such disparities may significantly influence the resulting outcomes.
Aspects of robustness to the resulting errors have not been given much attention so far, except for, \textit{\eg~}~Young \textit{et al.}~\cite{young2013adiabatic} 
proposing an exemplary strategy to combat this error type.

\vspace{0.5mm}
\parahead{Training of deep neural networks on QCs} 
Many problems in computer vision can now be successfully addressed by training some kind of a deep neural network (\eg~transformers~\cite{Dosovitskiy2021}) end-to-end on large enough datasets. 
Rather than being instances of QML, the existing network architectures are computationally classical by design. This raises a compelling question: \emph{Can quantum computers be leveraged to accelerate and make deep network training more {\coo}-efficient, particularly for large-scale models like transformers?}
Despite some progress, such as efforts to train binary networks~\cite{Carrasquilla2023, Krahn2024}, there is no widely accepted framework for training general neural architectures on quantum hardware. A key challenge lies in optimising real-valued network weights within the discrete constraints of a binary quantum annealer~\cite{Li2025arXiv}. However, recent advances in convex formulations of neural networks and co-positive programming~\cite{yurtsever2022q} suggest promising directions---particularly since some constrained real-valued solutions naturally reside on binary vertices. Notably, the recent work of Prakhya~\etal~\cite{prakhya2024convex} introduces a co-positive convex formulation for two-layer, finite-width ReLU networks. With further development, such formulations could pave the way for quantum-compatible neural network training, bridging the gap between classical deep learning and quantum optimisation~\cite{Abbas2024}. 

\vspace{0.5mm}
\parahead{Quantum computers as samplers} 
Quantum annealers (and quantum computers in general) solve problems by repeatedly sampling from a modified posterior distribution, typically selecting either the most frequent or the lowest-energy solution as the final output. This approach bears similarities to maximum a posteriori (MAP) estimation, where the solution is reduced to a single mode, discarding all other sampled information. However, a more principled approach would leverage the full set of quantum samples to approximate the true problem posterior, without resorting to expensive sequential Markov Chain Monte Carlo (MCMC) methods. Since quantum measurements inherently produce likely solutions, these samples could be recycled to perform uncertainty quantification or the identification of multiple plausible solutions (modes). 
Despite its appeal, this approach is non-trivial because quantum annealers optimise a modified posterior rather than the true problem posterior. The transformation or alignment between these distributions is not always straightforward, necessitating a general framework for calibrating quantum samples to match the true posterior. One promising step in this direction is the probabilistic k-means approach of Zaech et al.~\cite{zaech2024probabilistic}, where calibration between the annealed and true posterior is achieved with minimal computational overhead---albeit for the specific case of binary balanced k-means. Extending such methodologies to general posteriors, thereby enabling robust uncertainty quantification, remains an important and promising avenue for future research. 
\vspace{0.5mm}
\parahead{PQC architectures for vision} 
Another open challenge is that most results in PQC research are still theoretical. 
Many assumptions in PQC works are often too restrictive, simplified and unrealistic, often 
making any statements or even speculations for real systems impossible. 
Given that computer vision is an applied research field, many experimental results need to be obtained to confirm theoretical works or show contradictions with them. 
On the other hand, many PQC works demonstrate applications in vision using only a few exemplary problems, such as image classification, often leaving the investigation of a broader range of computer vision problems aside. 
There are, however, many other problems in computer vision that could benefit from quantum computational paradigms and hybrid architectures. 
From the classical setting, it is well known that different problems require different approaches and architectures, inductive biases, losses and training regimes. 
We strongly believe that the situation is similar regarding quantum-enhanced techniques.
Finding [near-]optimal architectures for all relevant, challenging and unsolved computer vision problems could mean the initiation of hundreds and thousands of projects over the next decade and beyond. 

\vspace{0.5mm}
\parahead{PQC trainability and barren plateaus} 
Another open challenge is associated with PQC trainability. 
Of course, small-scale PQCs can be simulated and trained classically, allowing the evaluation of the gradient via computational graphs. 
However, as we expect to reach a point where PQCs become too complex for classical computers, an efficient way for evaluating the gradient for PQCs on quantum hardware would be essential. 
Current algorithms for that, \textit{\eg~}~the parameter-shift rule~\cite{mitarai2018quantum}, require multiple runs and evaluations of a quantum circuit (which scales linearly with the number of parameters).
In classical neural networks, calculating the gradient scales similarly to the forward pass, thanks to the reuse of the intermediate quantities. 
However, the collapse of the state function (due to measurements) in PQCs seems to prevent reusing intermediate gradients when calculating the gradient. 
A method known as simultaneous perturbation stochastic approximation (SPSA)~\cite{spall1992multivariate,spall1998overview,wang2011discrete,gacon2021simultaneous} allows to approximate the gradient with only two runs of the quantum circuit, irrespective of the number of parameters. 
More efficient or accurate gradient computation methods could enable scalable and hardware-native training. 

Similar to training classical gradient-based learning models, proper model initialisation is important to ensure proper gradient flow. 
As most of the time, we can not infer the prior distribution of the parameters, a common choice is to randomly initialise the weights. 
However, this, in the field of PQC, could cause the gradient vanishing problem called ``barren plateaus''.
It is expressed by partial derivatives that become exponentially small with the number of qubits $n$ considered in the PQC: 
\begin{equation} \label{eq:26}
\mathbb{E}_\theta [\partial _\theta \mathcal L(\theta)] = 0, 
\end{equation}
\begin{equation}
\operatorname{Var}_\theta [\partial _\theta \mathcal L(\theta)] \in O \bigg(\frac{1}{\nu^n} \bigg), \;\nu > 1.
\end{equation}
Here, $\theta$ are the tunable parameters, $\mathcal L$ is the loss function and $\nu$ is some number greater than one (a specific value of $\nu$ depends on the quantum state density, the unitary circuit block and the specific measurement operator). 
The root cause of barren plateau is believed to be linked with random matrix theory and unitary design with identified factors such as: 1) unconstrained expressivity~\cite{mcclean2018barren,arrasmith2022equivalence,holmes2022connecting} 2) global measurement~\cite{cerezo2021cost} 3) specific noise such as depolarizing noise~\cite{wang2021noise} and 4) specific entanglement types~\cite{ortiz2021entanglement}. 
Some design architectures, such as quantum convolutional neural network~\cite{cong2019quantum,pesah2021absence}, permutation-equivariant quantum neural network~\cite{schatzki2022theoretical} and initialisation strategies such as Refs.~\cite{zhang2022escaping,grant2019initialization,shi2024avoiding}, have been proposed as useful designs to relieve this problem; more practical policies are yet to be discovered in computer vision. 
Finally, recent research results suggest that some architectures that avoid barren plateaus are also classically simulable~\cite{bermejo2024quantum, cerezo2023does}. 

\vspace{0.5mm}
\parahead{Parameter selection in quantum annealing} 
A common way to find optimal weights for rectifiers in QUBO formulations is grid search. 
One of the practical reasons for its popularity is that it works well in practice. 
It has been observed in several publications that rectification weights can be selected from large ranges. 
Hence, they can often be determined even with large search steps. 
Moreover, several works propose to select rectifier weights in other ways, \ie~ using iterative rectifier policies~\cite{zaech2022adiabatic}, a co-positive reformulation~\cite{yurtsever2022q} or a tree-structured Parzen estimator (TPE)~\cite{pandey2025outlier}. 
QuAnt-type QUBO learning~\cite{seelbach2022quant} avoids explicit rectifiers as it does not require explicit linear constraints. 
Parameter selection in quantum annealing, nevertheless, remains an open challenge as all techniques have advantages and disadvantages, such as the generalisation of the found parameters in the case of the grid search and TPE~\cite{pandey2025outlier} or the convergence speed~\cite{zaech2022adiabatic, yurtsever2022q}. 
\parahead{High-risk nature of QeCV research} 
QeCV research is high-risk due to many open challenges, some of which have already been described in this section. 
The currently available quantum hardware is noisy and error-prone; accessing it and performing systematic and large-scale experiments might be associated with high costs and investments and no short-term payoffs. 
Related to it is another uncertainty concerning the quantum computing roadmap. 
What if fault-tolerant real quantum hardware arrives later than predicted and what will it mean for QeCV in the mid- and long-term? 
Encoding classical data into quantum states remains a challenge, especially images and 3D data. 
The cost of encoding can potentially nullify advantages in the later method stages. 
The theoretical assumptions made in many papers on QML are unrealistic and far from the state of modern quantum hardware. 
Moreover, quantum machine learning architectures allow weaker inductive biases and less flexibility in the architecture design compared to classical machine learning (\ie~the choice of learnable operations to choose from is much more restricted compared to the classical case). 
We believe these challenges and risks have to be taken into account when preparing for submission and reviewing a QeCV paper. 
\subsection{Publishing and Reviewing a QeCV Paper}
\label{ssec:publishing_reviewing} 
Publishing and reviewing papers on QeCV brings additional challenges on top of reaching the publication bar of the respective venue, since QeCV is an emerging field. 
Those include challenges due to the (comparably new for the CV community) computational paradigm, presentation of fundamentals, the expected novelty and evaluation on real quantum hardware vs.~simulations. 
\parahead{Challenges due to new computational paradigm} 
A new paradigm brings several challenges with it. 
We highlight the following ones: 1) Motivating its usage; 2) Accessing real quantum hardware or setting up a simulator; 3) Comprehensive and accessible presentation of the method and results for the CV community; 4) Reasonable interpretation of the experimental results (\eg~avoiding ``optimism bias'' in the absence of strong experimental evidence). 
%
%
Care must be taken to use the terminology as accurately and consistently as possible, despite being new to the CV community. 

\parahead{The fundamentals and the novelty aspect} 
Providing an extended background section about QeCV would help many readers understand relevant fundamentals of quantum computing and, eventually, how a new method works and why certain design choices have been made; the paper will be perceived as self-contained. 
As the field progresses, the fundamentals section will shrink until it is possible to keep it very short (a few sentences) or omit it entirely while assuming that the readers are familiar with the basics. 
\parahead{Reviewing a QeCV paper}
The history of publications at CV venues led to the emergence of expectations from QeCV papers different from those in non-quantum CV (\eg~numerical and runtime improvements compared to previous methods). 
Hence, reviewers should be informed and calibrate their expectations. 
Technical contributions should be analysed and presented from several perspectives, with a particular emphasis on how the new work advances quantum computational science and how it compares to the previous quantum state of the art.  
In the case of QeCV, improvements could be, \eg~in terms of a better complexity class, fewer parameters in a model or no necessity to perform relaxations of the objective functions. 
When the first approach for a given problem relying on quantum hardware is proposed, it is unreasonable to demand/expect the outperformance of classical state of the art (especially considering the current hardware realisations of quantum computers). 
On the other hand, having a quantum approach just for the sake of a new quantum version---with no conceivable theoretical improvements compared to existing classical techniques---could be valuable and ignite new ideas, but also be seen as problematic.

\vspace{0.5mm}
\parahead{Real quantum hardware \textit{vs} simulators} 
Another aspect concerns the experimental evaluation. 
While there are more and more ways to access real quantum hardware, not everyone can still afford it. 
At this stage, we believe that papers should not be penalised for not having experiments on real quantum hardware. 
After all, it is in the community's interest to be inclusive and open to new ideas. 
Results on quantum computer simulators or with objective optimisation by simulated annealing are also convincing (for problems that can be simulated or sampled on classical hardware in a reasonable time). 
\vspace{0.5mm}
\parahead{Development stages of QeCV approaches} 
QeCV papers demonstrate different maturity levels of the proposed methods and presented results. 
Based on them, we propose to distinguish several stages of QeCV method development (the higher, the more advanced): 
\begin{itemize}[leftmargin=*] 
\item \textbf{Stage 1:} A CV problem has been successfully mapped into a form admissible to quantum hardware, giving rise to a QeCV approach; 
\item \textbf{Stage 2:} A QeCV approach has been evaluated on a quantum hardware simulator; 
\item \textbf{Stage 2A:} A QeCV approach evaluated on a quantum hardware simulator outperforms previous QeCV approaches tested on a simulator; 
\item \textbf{Stage 3:} A QeCV approach has been tested on a gate-based NISQ machine or a quantum annealer; 
\item \textbf{Stage 3A:} A QeCV approach tested on a gate-based machine or a quantum annealer, \ie~real quantum hardware, outperforms previous QeCV approaches tested on real hardware; 
\item \textbf{Stage 4:} A QeCV approach tested on a gate-based machine or a quantum annealer shows confident results on the level of classical state of the art; 
\item \textbf{Stage 4A:} A QeCV approach tested on a gate-based machine or a quantum annealer outperforms previous QeCV approaches and shows confident results on the level of classical state of the art; 
\item \textbf{Stage 5:} A QeCV approach tested on a gate-based machine or a quantum annealer clearly shows superior results than the best-known classical state of the art; 
\item \textbf{Stage 5A:} A QeCV approach tested on a gate-based machine or a quantum annealer clearly shows superior results than previous QeCV approaches and the best-known classical state of the art. 
\end{itemize} 
Note that we understand under ``gate-based machine'' in the above list either a gate-based NISQ device or an error-corrected gate-based quantum computer in future. 
Moreover, the achieved stage depends not only on the method but also on the quantum hardware development level. 
This implies that level upgrades will be possible solely due to hardware improvements. 
QeCV and classical approaches to the same problem are always compared in stages 2A and 3A--5A. 
The case when a QeCV approach tested on real quantum hardware outperforms a previous QeCV method tested on a quantum hardware simulator is not explicitly mentioned (for the sake of simplicity, as this case is expected to be rare and irrelevant in the near term) and can be classified into stages 4A or higher. 
The achieved highest stages can differ across the problems with known QeCV approaches. 
Theoretical papers are of Stage 1. 
First papers proposing gate-based QeCV solutions to a new problem usually fall into Stage 2 (Stage 3 for annealers) and the first follow-ups are often of Stages 2/2A (Stage 3A for annealers). 
Several QeCV methods using quantum annealers reach stages 4/4A; reaching the same stages for gate-based methods depends on the progress in gate-based quantum hardware. 
The long-term goal of QeCV is to investigate which CV problems can have QeCV solutions and which stages they can reach. 
\emph{Reaching and maintaining stage 5/5A on any problem is the ultimate goal of QeCV for the next decades.} 
The community can monitor the progress in QeCV and orient itself by the proposed development stages. 
The motivation of compatibility with quantum hardware serves as a driver of scientific discovery. 
\subsection{Potential Social Implications}
\label{ssec:potential_social_implicatins} 
We believe that QC has a high potential to reshape and strongly influence in the long term many areas of science and engineering, including computer vision. 
Readers interested in ethical considerations of quantum computing can refer to  Roberson~\cite{roberson2021building}. 
Computer vision and machine learning are requiring more and more electricity and cause more and more \coo\ emissions~\cite{2019arXiv190602243S, Fu2021}. 
Hence, quantum computing has the potential to reduce those in the long term as quantum-enhanced techniques are increasingly widely adopted. 
We also should not forget that many things can be discovered in QeCV until we have fault-tolerant quantum computers. 
Hence, quantum computational paradigms are a driver of scientific discovery in the broad sense. 
QCs will not replace existing computers, but rather enrich the infrastructure and set of tools for computer vision in addition to CPUs, GPUs and embedded processors among others. 
In future, fault-tolerant gate-based quantum computers could enable efficient modeling of non-negligible quantum-mechanical effects (beyond the possibilities of classical machines) in 1) image formation and 2) solving inverse problems in computer vision. 
These two categories include high-resolution imaging; visual sensing and inverse problems where interferometric effects play a role; and imaging and inverse problems limited by the Heisenberg bound (\eg~in microscopy). 

\section{Conclusion}
\label{sec:conclude} 
This survey introduced and provided a comprehensive overview of an emerging field of QeCV, one of the first applied quantum computing  disciplines across the sciences. 
We reviewed the fundamentals and the operational principles of quantum computers required to understand existing state-of-the-art QeCV methods for readers with a computer vision background and no strong background in physics. 
Among the two largest quantum computing paradigms, i.e., gate-based and quantum annealing, the latter is more technologically advanced and allows experimentation on real quantum hardware. 
In contrast, the gate-based paradigm has been studied for computer vision theoretically or on simulators so far, though the recent introduction of the error correction code in the Willow quantum processor is a highly promising milestone towards scalable universal quantum computing for computer vision.

Next, we reviewed resources and tools for QeCV research, aspects that need to be taken into account when writing or reviewing a QeCV paper, and discussed current challenges, future directions, and potential social implications. 
We also introduced the development stages of QeCV approaches. 
To conclude---while the researchers and hardware producers will be occupied with the development of fault-tolerant and reliable quantum hardware during the next years---we see QeCV as an exciting field with many open challenges that has a high potential to provide solutions of a new kind to some of the hardest and unsolved problems, and develop next-generation techniques in computer vision. 

\vspace{5pt}
\parahead{Acknowledgements} 
NKM, MM and VG acknowledge the support of the Deutsche Forschungsgemeinschaft (DFG, German Research Foundation), project number 534951134. 
TB acknowledges support from the Engineering and Physical Sciences Research Council [grant EP/X011364/1]. TB was supported by a UKRI Future Leaders Fellowship [grant number MR/Y018818/1]. 
The authors thank Sharareh Sayyad for feedback on Sec.~\ref{sec:operational_principles} and the Appendices.

\clearpage\twocolumn[\vspace*{\fill}]
\bibliographystyle{IEEEtran}
\bibliography{egbib}

\clearpage\twocolumn[\vspace*{\fill}]
\begin{appendices}

\section{Quantum speedup---a representative example}\label{sec:gatespeedup}

This section will introduce Deutsch's algorithm, a simple quantum algorithm that demonstrates speedup over the classical method.
Deutsch's algorithm solves the following problem: let $f:\{0,1\} \mapsto \{0,1\}$ be a binary function. We say $f$ is \emph{balanced} if $f(0)\ne f(1)$, else $f$ is \emph{constant}, \ie~$f(0) = f(1)$. Table~\ref{tab:binfuncs} shows the possible $f$'s. Given an \emph{unknown} $f$, determine if $f$ is constant or balanced.

\begin{table}[h]\centering
\subfloat[Const.]{
\begin{tabular}{|c|c|}
\hline
  $x$   &  $f(x)$ \\
  \hline
  0   & 0 \\
  1   & 0 \\
  \hline
\end{tabular}}
\hspace{0.5em}
\subfloat[Bal.]{
\begin{tabular}{|c|c|}
\hline
  $x$   &  $f(x)$ \\
  \hline
  0   & 0 \\
  1   & 1 \\
  \hline
\end{tabular}}
\hspace{0.5em}
\subfloat[Bal.]{
\begin{tabular}{|c|c|}
\hline
  $x$   &  $f(x)$ \\
  \hline
  0   & 1 \\
  1   & 0 \\
  \hline
\end{tabular}\label{tab:thebal}}
\hspace{0.5em}
\subfloat[Const.]{
\begin{tabular}{|c|c|}
\hline
  $x$   &  $f(x)$ \\
  \hline
  0   & 1 \\
  1   & 1 \\
  \hline
\end{tabular}}
\caption{Constant and balanced binary functions.}
\label{tab:binfuncs}
\end{table}

Explanations are for what ``unknown'' means: in the context of Deutsch's Algorithm, we are able to evaluate $f$, but we do not have access to its definition. 
This makes an already simplistic problem contrived, but such an artificial problem is essential for an accessible introduction. 
Each truth table in Table~\ref{tab:binfuncs}, in this case, defines a logical gate that evaluates the function $f$.

How would the problem be solved classically? The solution must involve evaluating $f$ twice, \ie~evaluating $f(0)$ and $f(1)$ and comparing the outputs. Can we do better with a quantum machine? 

Deutsch's Algorithm constructs the 2-qubit quantum circuit shown in Fig.~\ref{fig:deutsch}, where $U_f$ is the quantum gate corresponding to the (unknown) function $f$. 
As alluded to in Sec.~\ref{sec:composegates}, any logical gate can be made reversible, hence a quantum gate. This is achieved by introducing a dummy input $y$, a redundant output that simply copies the input $x$, and a new output $y \oplus f(x)$, \ie~the XOR of $y$ and $f(x)$; see Fig.~\ref{fig:quantf}. 
Intuitively, $U_f$ outputs $f(x)$ in the second output qubit if input $y$ is set to $0$.

As a concrete example, say we are given the $f$ defined in Table~\ref{tab:thebal}. The corresponding truth table is
\begin{table}[h]\centering
\begin{tabular}{|c|c|c|c|}
\hline
\multicolumn{2}{|c|}{Inputs} & \multicolumn{2}{c|}{Outputs} \\
\hline
   $x$  & $y$ & $x$ & $y \oplus f(x)$ \\
\hline
   0  & 0 & 0 & $1$ \\
   0  & 1 & 0 & $0$ \\
   1  & 0 & 1 & $0$ \\
   1  & 1 & 1 & $1$ \\
   \hline
\end{tabular}
\end{table}
which implies that $U_f$ is the unitary matrix
\begin{equation}
	\label{eq:unitaryf}
	U_f =
    \begin{matrix}
		& \textit{00} & \textit{01} & \textit{10} & \textit{11} \\ \cline{2-5}
		\textit{00}  & 0  & 1  & 0  & 0  \\
		\textit{01}  & 1  & 0  & 0  & 0  \\
		\textit{10}  & 0  & 0  & 1  & 0  \\
		\textit{11}  & 0  & 0  & 0  & 1
    \end{matrix}
\end{equation}
where the row (resp.~column) labels indicate the input (resp.~output) combinations and the 1's in the matrix indicate the inputs and outputs that are corresponding.

\begin{figure}[t]\centering
\begin{quantikz}
\ket{0} & \gate[2]{H^{\otimes 2}}\slice{$\ket{\phi_1}$} & \gate[2]{U_f} \slice{$\ket{\phi_2}$} & \gate{H}\slice{$\ket{\phi_3}$} & \meter{} \\
\ket{1} &                         &               &  \gate{I_2}        &
\end{quantikz}
\caption{Quantum circuit for Deutsch's algorithm.} 
\label{fig:deutsch} 
\end{figure}
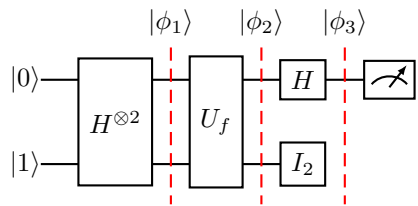

\begin{figure}[t]\centering
\subfloat[Quantum circuit]{
\begin{quantikz}
\ket{x}&\gate[2]{U_f}& \qw & \rstick{$ \ket{x} $} \\
\ket{y}&&\qw & \rstick{$ \ket{y \oplus f(x)} $}
\end{quantikz}
}
\hspace{1em}
\subfloat[Truth table. $f(x)' := \overline{f(x)}$]{ 
\begin{tabular}{|c|c|c|c|}
\hline
\multicolumn{2}{|c|}{Inputs} & \multicolumn{2}{c|}{Outputs} \\
\hline
   $x$  & $y$ & $x$ & $y \oplus f(x)$ \\
\hline
   0  & 0 & 0 & $f(0)$ \\
   0  & 1 & 0 & $f(0)'$ \\
   1  & 0 & 1 & $f(1)$ \\
   1  & 1 & 1 & $f(1)'$ \\
   \hline
\end{tabular}
\label{fig:uftruthtable}
}
\caption{Quantum implementation of $f$.}
\label{fig:quantf}
\end{figure}

The quantum circuit in Fig.~\ref{fig:deutsch} is initialised to the basic state $\ket{01}$. The successive application of the gates yields
\begin{align}
    \ket{\phi_3} = (H \otimes I_2 )U_f H^{\otimes 2}\ket{01},
\end{align}
followed by a measurement of the first qubit. To gain more insights, we need to step through the algorithm.

From Fig.~\ref{fig:deutsch}, we have
\begin{align}
    \ket{\phi_1} = \frac{1}{2} \left[ \begin{matrix} 1 \\ -1 \\ 1 \\ -1 \end{matrix} \right] = \frac{\ket{00} - \ket{01} + \ket{10} - \ket{11}}{2}.
\end{align}
Applying $U_f$ on $\ket{\phi_1}$ following the truth table in Fig.~\ref{fig:uftruthtable} gives
\begin{align}\label{eq:phi2tensored}
\ket{\phi_2} = \frac{\ket{0f(0)} - \ket{0 f(0)'} + \ket{1 f(1)} - \ket{1 f(1)'} }{2}.
\end{align}
Note that all four basic states in $\ket{\phi_1}$ are \emph{processed simultaneously via a single invocation} of $U_f$. As a concrete example, if $f$ is as defined in Table~\ref{tab:thebal}, we have
\begin{align}\label{eq:phi2eg}
    \ket{\phi_2} = \frac{\ket{01} - \ket{00} + \ket{10} - \ket{11} }{2}.
\end{align}
We could also obtain the same by multiplying Eq.~\eqref{eq:unitaryf} with Eq.~\eqref{eq:phi2eg} 
\begin{align}
    U_f
    \left[ \begin{matrix} \frac{1}{2} \\ -\frac{1}{2} \\ \frac{1}{2} \\ -\frac{1}{2} \end{matrix} \right] = 
    \left[ \begin{matrix} - \frac{1}{2} \\ \frac{1}{2} \\ \frac{1}{2} \\ -\frac{1}{2} \end{matrix} \right],
\end{align}
with the result simply being the vector form ~\eqref{eq:phi2eg}.

In the general case of an unknown $f$, Eq.~\eqref{eq:phi2tensored} can be rewritten in the separable form
\begin{align}
\ket{\phi_2} &= \left[ \frac{(-1)^{f(0)}\ket{0} + (-1)^{f(1)}\ket{1}}{\sqrt{2}} \right] \left[ \frac{\ket{0}-\ket{1}}{\sqrt{2}} \right] \\
&= \begin{cases}
    \pm \left[ \frac{\ket{0}+\ket{1}}{\sqrt{2}} \right] \left[ \frac{\ket{0}-\ket{1}}{\sqrt{2}} \right], \;\;\;\; \text{if $f$ is constant;} \\
    \pm \left[ \frac{\ket{0} - \ket{1}}{\sqrt{2}} \right] \left[ \frac{\ket{0}-\ket{1}}{\sqrt{2}} \right], \;\;\;\; \text{if $f$ is balanced.}
\end{cases}
\end{align}
Recognising that $H\otimes I_2$ applies the Hadamard gate on the 1st qubit while leaving the 2nd qubit unchanged, we have
\begin{align}
    \ket{\phi_3} = \begin{cases}
    \pm \ket{0} \left[ \frac{\ket{0}-\ket{1}}{\sqrt{2}} \right], \;\;\;\; \text{if $f$ is constant;} \\
    \pm \ket{1} \left[ \frac{\ket{0}-\ket{1}}{\sqrt{2}} \right], \;\;\;\; \text{if $f$ is balanced.}\label{eq:phi2sep}
\end{cases}
\end{align}
The result is that the first qubit will be in a state from the computational basis. 

The final step in Deutsch's algorithm is to measure the first qubit; in fact, using the process described in Eq.~\eqref{eq:measurement}. Since the first qubit is always going to be in a basis state---either $\ket{0}$ or $\ket{1}$ depending on whether $f$ is constant or balanced---the measurement outcome will immediately reveal which type of function the unknown $f$ is. 
As alluded to above, only a single evaluation of $f$ is required in Deutsch's Algorithm, compared to two in the classical method. While this does not seem to be a significant result at first glance, it is sufficient to introduce the fundamentally different approach and potential benefits of quantum computing. In any case, much greater gaps between quantum and classical methods (\eg~quadratic, exponential) have been achieved on more complex problems, such as Grover's Search Algorithm~\cite{grover1996fast} and Shor's 
Factoring Algorithm~\cite{Shor1997}.

Note that the computational gain of Deutsch's Algorithm is specified in terms of \emph{query complexity}, \ie~the number of function evaluations relative to that required in the classical method. Query complexity is a common way to assess quantum speedup, including for more prominent algorithms~\cite{ambainis2018understanding}.

\section{Proofs} 
We next provide proofs to several claims from the main text, \ie the construction of the diagonal problem Hamiltonian $H_P$ in Eq.~\eqref{eq:problem_hamiltonian}, the adiabatic theorem, Theorem 1 and the no-crossing eigenvalue theorem, Theorem 2.

\subsection{Eigenvalues of $H_P$}
\label{sec:eigenvalues_of_H_P}
To construct the time-dependent Hamiltonian of Schrödinger's Equation for AQC, we introduce in Eqs.~\eqref{eq:problem_hamiltonian} and \eqref{eq:initial_hamiltonian} the problem and initial Hamiltonians $H_P$ and $H_I$. 
For the problem Hamiltonian, we have (essentially) stated the following lemma:

\begin{lemma}{$H_P$ is diagonal}{}
The matrix
\begin{eqnarray*}
\label{eq:defHP}
    H_P &=& \sum_{i=1}^n \sum_{j=1}^n J_{i,j} \sigma_i^z \sigma_j^z + \sum_i b_i \sigma_i^z, \\
    \label{eq:defSimgai}
    \sigma_i^z &=& \underbrace{I \otimes I \hdots \otimes  I}_{(i-1)\text{-many times}} ~\otimes \sigma_z \otimes \underbrace{ I \otimes  \hdots \otimes I}_{(n-i)\text{-many times}}, \\
    \label{eq:defSimga}
    \sigma_z &=& \begin{bmatrix} 1 & 0 \\ 0 & -1 \end{bmatrix}, \ \ I = \begin{bmatrix} 1 & 0 \\ 0 & 1 \end{bmatrix}.
\end{eqnarray*}
is a diagonal $2^n \times 2^n$ matrix whose diagonal values enumerate all costs $s^T J s + s^T b$, \ie~for $2^n$ many $s \in  \{-1,1\}^n$. 
\end{lemma}
\begin{proof}
Note that $\sigma_i^z$ is a diagonal matrix, which immediately makes the entire $H_p$ diagonal. For any $a\in \{-1,1\}$, define 
\begin{equation}
    e(a) = \begin{cases} \begin{bmatrix} 1 \\ 0 \end{bmatrix} & \text{if } a = 1, \\
\begin{bmatrix} 0 \\ 1 \end{bmatrix}&\text{if } a = -1
\end{cases}
\end{equation}
as an eigenvector of $\sigma^z$ with eigenvalue $a$, and for any $s\in \{-1,1\}^n$ consider $e_s := e(s_1) \otimes e(s_2) \otimes \hdots \otimes e(s_n)$. Note that the multiplication of $e_s$ with $\sigma^z_i$ yields
\begin{align*}
    \sigma_i^z e_s &= e(s_1) \otimes \hdots e(s_{i-1}) \otimes \sigma_z e(s_i) \otimes e(s_{i+1}) \otimes \hdots \otimes e(s_n) \\
    &=\begin{cases}
    e_s & \text{if } s_i=1, \\
    -e_s & \text{if } s_i = -1,
    \end{cases}
\end{align*}  
due to the tensor product being multi-linear. As such, $e_s$ is an eigenvector of $\sigma_i^z$ with eigenvalue $s_i$, which results in 
\begin{align*}
    H_P e_s &=\sum_{i=1}^n \sum_{j=1}^n J_{i,j}  \sigma_i^z \sigma_j^z e_s) + \sum_i b_i \sigma_i^z e_s \\
    &= \left(\sum_{i=1}^n \sum_{j=1}^n J_{i,j} s_i s_j + \sum_i b_i s_i  \right) e_s,
\end{align*}
as claimed.
\end{proof}

\subsection{Quantum Adiabatic Theorem}
\label{sec:adiabatic_quantum_theorem} 
We next show the proof of the adiabatic theorem, Theorem 1, which supports quantum annealing as a heuristic for solving QUBO problems. A proof like this can also be found in \cite{das2005quantum} or in references cited in \cite{Hauke_2020}, while the main text from \cite{Hauke_2020} also provides further useful context.
We then elaborate on the \emph{``large spectral gap''} and the \emph{``sufficiently slow transition''} conditions needed to sample accurate solutions.

\begin{proof}
Note that all $H(t)$ are Hermitian matrices, such that for all $t$ there exists an orthonormal basis $\ket{k(t)}$ of eigenvectors to eigenvalues $\lambda_k(t)$, i.e.
\begin{equation}
\label{eq:eigenvectors2}
    H(t)\ket{k(t)} = \lambda_k(t) \ket{k(t)}.
\end{equation}
Thus, at each time we can represent a quantum state $\ket{\psi(t)}$ that is evolving under \eqref{eq:schroedinger_equation} as 
\begin{equation}
    \ket{\psi(t)} = \sum_k c_k(t)\ket{k(t)}
\end{equation}
for suitable coefficients $c_k(t)$. Inserting this representation into the Schrödinger equation yields on the left-hand side to 
\begin{align}
    i  \hslash \frac{d}{dt}\ket{\psi(t)} &=  i  \hslash \sum_k \frac{d}{dt} \left(c_k(t)\ket{k(t)} \right) \\
    &=  i  \hslash \sum_k c_k'(t) \ket{k(t)} + c_k(t) \ket{k'(t)},
\end{align}
and on the right-hand side to 
\begin{align}
    H(t) \ket{\psi(t)} = \sum_k\lambda_k(t) c_k(t)\ket{k(t)}.
\end{align}
Taking the product of the entire equation with one fixed $\ket{k(t)}$ and using the orthonormality
\begin{equation}
\langle l(t), k(t) \rangle = \begin{cases} 1 & \text{if } l=k \\
0 & \text{ otherwise.}
\end{cases}
\end{equation}
yields
\begin{align}
\label{eq:schoedinger_in_eigenbasis}
  &i  \hslash \left(  c_k'(t)  + c_k(t) \langle k(t),k'(t)\rangle + c_k(t)\sum_{l\neq k} \langle k(t), l'(t) \rangle \right) \nonumber \\
  &= \lambda_k(t) c_k(t).
\end{align}
Our next step is to gain an understanding of how $ \langle k(t), l'(t) \rangle $ behaves. To do so, we differentiate \eqref{eq:eigenvectors2} (in the notation of $\ket{l}(t)$) with respect to $t$ to find
$$ H'(t) \ket{l(t)} + H(t) \ket{l'(t)} = \lambda_l'(t) \ket{l(t)} + \lambda_l(t) \ket{l'(t)},$$
and, after the product with $\bra{k(t)}$, obtain
\begin{align*}
    \bra{k(t)} H'(t) \ket{l(t)}  + \bra{k(t)}H(t) \ket{l'(t)} &= \lambda_l(t) \Braket{k(t) |l'(t)}\\
    \Updownarrow \\
    \bra{k(t)} H'(t) \ket{l(t)}  + \lambda_k(t)\Braket{k(t) | l'(t)} &= \lambda_l(t) \Braket{k(t)|l'(t)}\\ 
    \Updownarrow  \\
    \langle k(t), l'(t) \rangle &= \frac{ \Braket{k(t)|H'(t) |l(t)} }{\lambda_l(t) - \lambda_k(t)}.
\end{align*}
Now if 
\begin{align}
    \label{eq:spectralGap} \|H'(t)\| \ll \min_{l\neq k} \lambda_l(t) - \lambda_k(t) 
\end{align}
holds for all times $t$, then the term $\sum_{l\neq k} \langle k(t), l'(t) \rangle  $ in \eqref{eq:schoedinger_in_eigenbasis} becomes negligible and we obtain 
\begin{equation}
\label{eq:componentwise_2}
    c_k'(t) \approx -\frac{i}{\hslash} (\lambda_k(t)-\langle k(t), k'(t) \rangle) c_k(t)
\end{equation}
 This implies that
\begin{equation}
    c_k(t) \approx c_k(0)\ \text{exp}\left(-\frac{i}{\hslash}\int_0^t  \lambda_k(s) - i \hslash\langle k(s), k'(s) \rangle~ds\right)
\end{equation}
and shows that if $c_k(0)=0$ has held for all but one index $k$, it will approximately hold for all $t \leq T$. The quality of this approximation depends on two things:
\begin{enumerate}
    \item First on the so-called \textit{spectral gap} 
$$ \min_t \min_{l\neq k} \lambda_l(t) - \lambda_k(t) $$ 
being large for $\lambda_k(t)$ being the smallest eigenvalue of $H(t)$.
\item Second on $H'(t)$ being small. Note that for our example of $H(t) = (1-t/T) H_I + t/T H_P $ the derivative, $H'(t)$, is proportional to $1/T$, \ie~the total \textit{annealing time}, indicating that longer annealing times will improve the quality of the solution, provided that the physical realization allows to do so in a stable manner.   
\end{enumerate}
\end{proof}

\subsection{No Crossing of Eigenvalues}
\label{sec:no_crossing_eigenvalues}
Before we come to the actual proof, let us recall the particular part of the Perron-Frobenius Theorem.
\begin{theorem}{Perron-Frobenius}{}
Let $H \in \mathbb{R}^{m \times m}$ be non-negative, \ie~all $H_{i,j}\geq 0$ and irreducible. Then the largest eigenvalue in magnitude is positive and simple. 
\end{theorem}
Our proof strategy is to show that $-H(t)$ in the evolution of the Hamiltonians is non-negative and irreducible, such that the simplicity of the largest eigenvalue ensures a non-zero spectral gap, \ie~a non-zero difference between the smallest and second smallest eigenvalues of $H(t)$. 

Recall that a matrix $H \in \mathbb{R}^{m \times m}$ is irreducible if and only if the directed graph we obtain by looking at vertices $i$, $i \in \{1, \hdots ,m\}$ and connecting vertex $i$ with vertex $j$ if $H_{i,j}\neq 0$, is \textit{strongly connected}. The latter means that there is a path from any vertex to any other vertex. We first prove the following lemma.
\begin{lemma}{$H_I$ is irreducible}{}
The matrix $H_I$ defined as 
\begin{eqnarray*}
\label{eq:initial_hamiltonian_2}
    H_I &=& -\sum_{i}  \kappa \sigma_i^x  \\
    \sigma_i^x &=& \underbrace{I \otimes I \hdots I}_{(i-1)\text{-many times}}  \otimes \sigma_x \otimes \underbrace{ I \otimes I \hdots \otimes I}_{(n-i)\text{-many times}}, \\
    \sigma_x &=& \begin{bmatrix} 0 & 1 \\ 1 & 0 \end{bmatrix}, \ \ I = \begin{bmatrix} 1 & 0 \\ 0 & 1 \end{bmatrix}.
\end{eqnarray*}
is irreducible. 
\end{lemma}
\begin{proof}
As before, for any $a\in \{-1,1\}$, define 
\begin{equation}
e(a) = \begin{cases} \begin{bmatrix} 1 \\ 0 \end{bmatrix} & \text{if } a = 1, \\
\begin{bmatrix} 0 \\ 1 \end{bmatrix}&\text{if } a = -1
\end{cases}
\end{equation}
as an eigenvector of $\sigma^z$ with eigenvalue $a$,
and for any $s\in \{-1,1\}^n$ consider $e_s := e(s_1) \otimes e(s_2) \otimes \hdots \otimes e(s_n)$. We find that 
\begin{align}
     \sigma_i^x e_s &=   \hdots \otimes e(s_{i-1}) \otimes \sigma_x e(s_i) \otimes e(s_{i+1}) \otimes \hdots \\
     &=   \hdots \otimes e(s_{i-1}) \otimes e(-s_i) \otimes e(s_{i+1}) \otimes \hdots  .
\end{align}
This means, interpreting the $2^n$ many indices of $H_I$ as edges associated with every $s \in \{-1,1 \}^n$, the non-zero entries of $H_I$ connect every vertex $s$ to any other vertex $t$ that arises from $s$ by flipping one sign. Since every $s_1$ can be generated from every other $s_2$ by a sequence of sign flips, the graph induced by $H_I$ is strongly connected, which yields the assertion. 
\end{proof}

We can now prove the intended theorem, Theorem~\ref{thm:nocrossing}.
\begin{proof}
First, note that the diagonal entries of $H_I$ are zero. Second, note that $H_P$ is a diagonal matrix. Third, note that for any $c\in \mathbb{R}$
\begin{align}
    &\argmin_{s\in \{-1,+1\}^n} s^TJs + s^Tb \\
    &= \argmin_{s\in \{-1,+1\}^n} s^TJs + s^Tb  - c \cdot n\\
    &= \argmin_{s\in \{-1,+1\}^n} s^TJs + s^Tb  - c (s^T I s)\\
    &= \argmin_{s\in \{-1,+1\}^n} s^T(J-cI)s + s^Tb, 
\end{align}
which means that one can always ensure all entries of $H_P$ are negative. Thus, for any continuous monotone $f:[0,T]\rightarrow [0,1]$ with $f(0)=0$ and $f(T)=1$ we have that
$$ H(t) = (1-f(t))H_I + f(t)H_P $$
is non-positive and irreducible for $t<T$ due to the Perron-Frobenius theorem (and the previous Lemma), which means that the smallest eigenvalue is simple. 
Finally, the assumption of the ground state being non-degenerate yields the same for $t=T$. Thus, the minimal spectral gap in $[0,T]$ (which is attained as a continuous function over a compact interval) is greater than zero.  
\end{proof}

\subsection{Equivalence of gate-based and AQC}
\label{app:equivalence_gb_aqc}
We sketch the proof supporting the equivalence of gate-based and adiabatic quantum computing, Theorem~\ref{thm:equiv_aqc_gqc}.
Again, as mentioned in the main text, we only sketch the proof and refer the reader to~\cite{van2001powerful,aharonov2008adiabatic} for a full discussion.

\begin{proof}
\textbf{From adiabatic to gate-based model.} The unitary evolution of a quantum system is governed by the time-dependent Schrödinger equation, whose solution can be expressed as the time-evolution operator $U(\tau)$:
\begin{equation}
    U(\tau) = \textit{T}\,\text{exp} \left[-i \int_{0}^{\tau} H(t) dt\right],
\end{equation}
where $U(\tau)$ describes the system evolution from time $0$ to $\tau$ and $\textit{T}$ is the time-ordering operator.
In numerical simulations, such a continuous evolution is often discretized through Trotterization, which approximates the evolution as a sequence of time-independent Hamiltonian steps:
\begin{equation}
    U(\tau) = \textit{T}\,\text{exp} [-i \int_{0}^{\tau} H(t) dt] \approx \prod_{k=1}^{p} \text{exp} [-i H(k \Delta t) \Delta t].
\end{equation}
Here, $p$ should be chosen to be sufficiently large such that $\Delta t = \tau / p$ is small for this approximation to hold. 
In the context of quantum annealing, the system's Hamiltonian evolves smoothly from an initial Hamiltonian $ H_I $ to a final problem Hamiltonian $ H_P $ (see Eq.~\eqref{eq:hamiltonian_transition}). To apply Trotterization in this setting, we use the first-order Trotter formula:
\begin{equation}
    e^{i (A+B) \Delta t} = e^{i A \Delta t} e^{i B \Delta t} + \mathcal{O}{((\Delta t)^2)},
\end{equation}
which allows us to approximate the time-evolution operator in a form suitable for implementation on gate-based quantum hardware:
\begin{align}
\label{eq:aqc_to_circuit}
    U(\tau) = \prod_{k=1}^{p} &\text{exp} [(-i (1-f(k \Delta t)) H_I \Delta t] \nonumber \\
    &\cdot \text{exp} [(-i f(k \Delta t) H_P \Delta t],  
\end{align}
which can be executed within a circuit-based quantum computing framework.

\vspace{0.5mm}
\parahead{From gate-based to adiabatic model}
The proof consists in designing an AQC-Hamiltonian $H$ that drives the system, in polynomial time, into the output state of a given quantum circuit. 
It was provided by Aharonov \textit{et al.}~\cite{aharonov2008adiabatic} and goes back to the circuit-to-Hamiltonian construction of Kitaev~\cite{kitaev2002classical}. 
The idea is to decompose the circuit unitary into a product $U = U_p\cdots U_1$, each $U_k$ operating on one or two qubits. 
Assume $\ket{\gamma_k}$ is the state of the system after gate $U_k$: The goal is then to construct the final Hamiltonian $H_P$ with $\ket{\gamma_p}$ as its ground state. 
However, not knowing $\ket{\gamma_p}$, it seems impossible to explicitly specify $H_P$. 
Kitaev~\cite[Section 14.4]{kitaev2002classical} showed that by adding a so-called clock register to the system, it is possible to construct---without knowing the states $\ket{\gamma_k}$---a Hamiltonian $H$ whose ground state is
\begin{equation}
\label{eq:psi_gqc_to_aqc}
    \ket{\psi} = \frac{1}{\sqrt{p+1}} \sum_{k=0}^p \ket{\gamma_k} \otimes \ket{k}^c.
\end{equation}
The clock register $\ket{k}^c$ validates the correctness of the quantum state propagation throughout the circuit. 
This Hamiltonian $H$ is constructed as a sum
\begin{equation}
\label{eq:H_gqc_to_aqc}
    H = \frac{1}{2} \sum_{k=1}^p H_k + H_{\text{in}} + H_{\text{out}},
\end{equation}
where Hamiltonians $H_{\text{in}} $ and $ H_{\text{out}}$ (not dependent on any $U_k$) ensure the correct initialisation and evolution of the clock register. 
Only Hamiltonians 
\begin{align}
    H_k = & -U_k \otimes \ket{k}\bra{k-1}^c - U_k^{\dagger} \otimes \ket{k-1}\bra{k}^c \nonumber \\
    &+ I \otimes (\ket{k}\bra{k}^c + \ket{k-1}\bra{k-1}^c)
\end{align} 
involve gates $U_k$ and their inverses, denoted by $U_k^{\dagger}$. 
They ensure that any wrong state of the clock register of the system incurs a higher energy level while any correct state is tensor-ed with the desired state of the propagation at step $k$. 
At the end of the evolution, one measures the clock register in Eq.~\eqref{eq:psi_gqc_to_aqc} with probability $1/(p+1)$ in the state $\ket{p}^c$, meaning that the working register is in the desired ground state $\ket{\gamma_p}$. 
Aharonov et~al.~\cite[section 4]{aharonov2008adiabatic} showed that the spectral gap of the resulting AQC-Hamiltonian is lower-bounded by $1/p^3$, justifying that the running time of the adiabatic computation is polynomial with the circuit depth. 
\end{proof}

\section{Monte Carlo Sampling} 
\label{app:montecarlosampling}
Another question one could ask is how quantum or simulated annealing compare to classical algorithms that are inspired by quantum annealing. 
This question is explored, \eg~in Crosson et al.~\cite{crosson2016simulated}. 
They look at an algorithm where low-energy states are estimated using classical Markov chain Monte Carlo methods. 
The authors consider not only one ---as done in simulated annealing, but $L$ state configurations $(x_1, \ldots, x_L)$, $x_i \in \{0, 1\}^n$, of the quantum system at a time. 
The distribution of those multi-states at time $s=t/T$ is given by
\begin{equation}
    \pi(x_1, \ldots, x_L) = 
    \frac{1}{Z} e^{-\frac{\beta s}{L} \sum_{i=1}^L f(x_i)}\prod_{j=1}^n \phi(\bar x_j),
\end{equation}
where $Z$ is a normalisation constant, $\beta$ is a temperature parameter, $f(x_i)=x_i^\top Q x_i$ is the energy of the system in state $x_i$ and $\phi(\bar x_j)$ is a function counting the number of consecutive bits that disagree along the so-called worldline $\bar x_j = (x_{j,1}, \ldots, x_{j,L})$. 
The multi-state configuration allows to simulate quantum effects like tunnelling, superposition and entanglement. 
The quantum distribution $\Pi(x)$, which represents the actual quantum system's behaviour, is approximated by summing over all the multi-state probabilities: 
\begin{equation}
    \Pi(x) = \sum_{x_2, \ldots, x_L} \pi(x, x_2, \ldots, x_L).
\end{equation} 
The simulated evolution then uses a Markov chain Monte Carlo sample from $\pi$ at various values of the adiabatic parameter $s$. 
In a simple form, a configuration $x^\prime$ is chosen in the neighbourhood of the current $x$ by randomly flipping a bit in a selected worldline and the new configuration is accepted with probability 
\begin{equation}
    P(x,x^\prime) =\frac{1}{2nL}\min\left\{1, \frac{\pi(x^\prime)}{\pi(x)}\right\}.
\end{equation}

The authors~\cite{crosson2016simulated} show that this simulated quantum annealing demonstrates an exponential advantage over simulated annealing on a problem class with a high spike on the energy landscape, where quantum annealing also exponentially outperforms simulated annealing. 
Therefore, for these problems, it was ruled out that the advantages of quantum tunnelling in simulated quantum annealing and quantum annealing can also be achieved by general-purpose classical optimisation algorithms and thus are not exclusively quantum advantages.

Simulated quantum annealing can be done particularly well for Hamiltonians that are stoquastic, i.e.~the off-diagonal entries are real and non-positive (in the standard basis).
The underlying Hamiltonians for the D-Wave machines are also stoquastic. 
Nevertheless, there already exist experiments where a scaling advantage towards simulated quantum annealing was reported~\cite{king2021scaling}. 
\end{appendices}


\end{document}